%% file: neurips_2023.tex
\documentclass{article}

\PassOptionsToPackage{numbers, compress}{natbib}

\usepackage[preprint]{neurips_2023}

\usepackage[utf8]{inputenc} 
\usepackage[T1]{fontenc}    
\usepackage{hyperref}       
\usepackage{url}            
\usepackage{booktabs}       
\usepackage{amsfonts}       
\usepackage{nicefrac}       
\usepackage{microtype}      
\usepackage{xcolor}         
\usepackage{graphicx}
\input{math_commands.tex}

\usepackage{multirow}
\usepackage{algorithm}
\usepackage{algorithmic}
\usepackage{adjustbox}
\usepackage{bbm}
\usepackage{bm}
\usepackage{mathrsfs}
\usepackage{amsmath} 
\usepackage{float}
\usepackage{diagbox}
\usepackage{subcaption}
\usepackage{amssymb}
\newtheorem{thm}{Theorem}[section]
\newtheorem{defn}[thm]{Definition}
\newtheorem{lemma}[thm]{Lemma}
\newtheorem{assume}{Assumption}
\newtheorem{remark}[thm]{Remark}
\newtheorem{corollary}[thm]{Corollary}
\newtheorem{prop}[thm]{Proposition}

\newenvironment{proof}{Proof:}{\hfill$\square$}
\usepackage{threeparttable}
\usepackage{makecell}
\usepackage{enumerate}

\usepackage{colortbl}
\definecolor{bgcolor}{rgb}{0.8,1,1}
\definecolor{bgcolor2}{rgb}{0.8,1,0.8}

\usepackage{tikz}
\usetikzlibrary{positioning, shapes.geometric}

\newcommand{\norm}[1]{\left\|#1\right\|}
\newcommand{\rom}[1]{(\romannumeral #1)}

\def\geo{\mathrm{Geom}}
\newcommand{\dotprod}[1]{\left\langle #1\right\rangle}

\newcommand{\indicatorr}{\mathbbm{1}}
\newcommand{\svrs}{$\mathrm{SVRS}^{\mathrm{1ep}}$}
\DeclareMathOperator{\prox}{prox}

\newcommand{\myred}[1]{{\color{red} #1}}

\def\hardr{{r}}
\def\hardf{f^{\mathrm{h}}}
\def\harddim{{m}}
\def\hardc{{c}}
\def\hardzeta{{\zeta}}
\def\subspace{{\gF}}
\def\ifo{{h}^\mathrm{I}}
\def\pifo{{h}^\mathrm{P}}

\def\spn{\mathrm{span}}

\title{Stochastic Distributed Optimization under Average Second-order Similarity: Algorithms and Analysis}

\author{Dachao Lin \thanks{Equal Contribution.} \thanks{Academy for Advanced Interdisciplinary Studies; Peking University; \texttt{lindachao@pku.edu.cn};}
\And
Yuze Han \footnotemark[1] \thanks{School of Mathematical Sciences; Peking University; \texttt{hanyuze97@pku.edu.cn};}
\And
Haishan Ye
\thanks{Corresponding Author; School of Management; Xi'an Jiaotong University; SGIT AI Lab, State Grid Corporation of China; \texttt{yehaishan@xjtu.edu.cn};}	
\And
Zhihua Zhang 
\thanks{School of Mathematical Sciences; Peking University; \texttt{zhzhang@math.pku.edu.cn}.}
}

\begin{document}
\maketitle

\begin{abstract}
We study finite-sum distributed optimization problems involving a master node and $n-1$ local nodes under the popular $\delta$-similarity and $\mu$-strong convexity conditions. 
We propose two new algorithms, SVRS and AccSVRS, motivated by previous works. 
The non-accelerated SVRS method combines the techniques of gradient sliding and variance reduction and achieves a better communication complexity of $\tilde{\gO}(n {+} \sqrt{n}\delta/\mu)$ compared to existing non-accelerated algorithms.
Applying the framework proposed in Katyusha X \cite{allen2018katyusha}, we also develop a directly accelerated version named AccSVRS with the $\tilde{\gO}(n {+} n^{3/4}\sqrt{\delta/\mu})$ communication complexity.
In contrast to existing results, our complexity bounds are entirely smoothness-free and exhibit superiority in ill-conditioned cases.
Furthermore, we establish a nearly matched lower bound to verify the tightness of our AccSVRS method.
\end{abstract}

\input{introduction}

\input{upperbound}

\input{lower}

\input{experiments}
\section{Conclusion}\label{sec:con}
In this paper, we have introduced two new algorithms, SVRS and its directly accelerated version AccSVRS, and established improved communication complexity bounds for distributed optimization under the similarity assumption.
Our rates are entirely smoothness-free and only require strong convexity of the objective, average similarity, and proximal friendliness of components.
Moreover, our methods also have nearly optimal gradient complexity (leaving out the log term) when applied to smooth components in specific cases.
It would be interesting to remove additional log terms to achieve both optimal
communication and local gradient calls as \cite{kovalev2022optimal}, as well as investigating the complexity under other similarity assumptions (such as SS instead of AveSS) in future research.

\begin{ack}
Lin, Han, and Zhang have been supported by the National Key Research and Development Project of China (No. 2022YFA1004002) and the National Natural Science Foundation of China (No. 12271011).
Ye has been supported by the National Natural Science Foundation of China (No. 12101491). 

\end{ack}
\bibliography{reference}
\bibliographystyle{plainnat}

\newpage

\newpage 
\input{appendix}

\end{document}

%% file: math_commands.tex
\usepackage{amsmath,amsfonts,bm}

\def\1{\bm{1}}

\def\eps{{\epsilon}}

\def\vzero{{\bm{0}}}

\def\va{{\bm{a}}}
\def\vb{{\bm{b}}}

\def\ve{{\bm{e}}}

\def\vg{{\bm{g}}}

\def\vu{{\bm{u}}}
\def\vv{{\bm{v}}}
\def\vw{{\bm{w}}}
\def\vx{{\bm{x}}}
\def\vy{{\bm{y}}}
\def\vz{{\bm{z}}}

\def\mA{{\bm{A}}}
\def\mB{{\bm{B}}}

\def\mD{{\bm{D}}}

\def\mI{{\bm{I}}}

\def\mN{{\bm{N}}}

\def\mZ{{\bm{Z}}}

\DeclareMathAlphabet{\mathsfit}{\encodingdefault}{\sfdefault}{m}{sl}
\SetMathAlphabet{\mathsfit}{bold}{\encodingdefault}{\sfdefault}{bx}{n}

\def\gA{{\mathcal{A}}}

\def\gD{{\mathcal{D}}}

\def\gF{{\mathcal{F}}}
\def\gG{{\mathcal{G}}}

\def\gI{{\mathcal{I}}}

\def\gL{{\mathcal{L}}}

\def\gO{{\mathcal{O}}}

\def\gS{{\mathcal{S}}}

\def\sN{{\mathbb{N}}}

\def\sP{{\mathbb{P}}}

\def\sR{{\mathbb{R}}}

\newcommand{\E}{\mathbb{E}}

\DeclareMathOperator*{\argmin}{arg\,min}

%% file: introduction.tex
\section{Introduction}
We have witnessed the development of distributed optimization in recent years. 
Distributed optimization aims to cooperatively solve a learning task over a predefined social network by exchanging information exclusively with immediate neighbors.
This class of problems has found extensive applications in various fields, including machine learning, healthcare, network information processing, telecommunications, manufacturing, natural language processing tasks, and multi-agent control \cite{sun2022distributed,kairouz2021advances,nguyen2022federated,liu2021federated,zhang2022multi,banabilah2022federated}. 
In this paper, we focus on 
the following classical finite-sum optimization problem in a centralized setting:
\begin{equation}\label{eq:obj}
    \min_{\vx \in \sR^d} f(\vx):= \frac{1}{n}\sum_{i=1}^{n} f_i(\vx),
\end{equation}
where each $f_i$ is differentiable and corresponds to a client or node, and the target objective is their average function $f$. 
Without loss of generality, we assume $f_1$ is the master node and the others are local nodes.
In each round, every local node can communicate with the master node
certain information, such as the local parameter $\vx$, local gradient $\nabla f_i(\vx)$, and some global information gathered at the master node.
Such a scheme can also be viewed as 
decentralized optimization over a star network \citep{tian2022acceleration}.

Following the wisdom of statistical similarity residing in the data at different nodes, many previous works study scenarios where
the individual functions exhibit relationships or, more specifically, certain homogeneity shared among the local $f_i$'s and $f$.
The most common one is under the $\delta$-second-order similarity assumption \cite{khaled2022faster,kovalev2022optimal}, that is,
\[ 
\norm{\nabla^2 f_i(\vx) - \nabla^2 f(\vx)} \leq \delta, \forall \vx \in \sR^d, i \in \{1,\dots,n\}. 
\]
Such an assumption also has different names in the literature, such as $\delta$-related assumption, bounded Hessian dissimilarity, or function similarity \cite{arjevani2015communication,karimireddy2020scaffold, shamir2014communication,sun2022distributed,zhang2015disco}.
The rigorous definitions are deferred to
Section \ref{sec:pre}.
Moreover, the second-order similarity assumption can hold with a relatively small $\delta$ compared to the smoothness coefficient of $f_i$'s in many practical settings, such as statistical learning.
More insights on this can be found in the discussion presented in \cite[Section 2]{sun2022distributed}. 
The similarity assumption indicates that the data across different clients share common information on the second-order derivative, 
potentially leading to a reduction in communication among clients.
Meanwhile, the cost of communication is often much higher than that of local computation in distributed optimization settings \cite{bekkerman2011scaling,lian2017can,kairouz2021advances}.
Hence, researchers are motivated to develop efficient algorithms characterized by low communication complexity, which is the primary objective of this paper as well.

Furthermore, we need to emphasize that 
prior research \cite{han2021lower,woodworth2016tight,zhou2019lower,defazio2016simple,li2021anita,allen2017katyusha} has shown tightly matched lower and upper bounds on computation complexity for the finite-sum objective in Eq.~\eqref{eq:obj}. 
These works focus on gradient complexity under (average) smoothness \cite{zhou2019lower} instead of communication complexity under similarity.
Indeed, we will also discuss and compare the gradient complexity as shown in \cite{kovalev2022optimal}, to explore the trade-off between communication and gradient complexity.

Although the development of distributed optimization with similarity has lasted for years, the optimal complexity under full participation 
was only recently achieved 
by \citet{kovalev2022optimal}. They employed gradient-sliding \cite{lan2016gradient} and obtained the optimal communication complexity $\tilde{\gO}(n\sqrt{\delta/\mu})$ for $\mu$-strongly convex $f$ and $\delta$-related $f_i$'s in Eq.~\eqref{eq:obj}.
However, the full participation model requires the calculation of the whole gradient $\nabla f(\cdot)$, 
which incurs a communication cost of $n{-}1$
in each round. 
In contrast, partial participation could reduce the communication burden and yield improved complexity.
Hence, \citet{khaled2022faster} introduced client sampling, a technique that selects one client for updating in each round.
They developed a non-accelerated algorithm SVRP, which achieves the communication complexity of
$\tilde{\gO}(n{+}\delta^2/\mu^2)$.
Additionally, they proposed a Catalyzed version of SVRP with the complexity
$\tilde{\gO}(n{+}n^{3/4}\sqrt{\delta/\mu})$, which is better than the rates obtained in the full participation setting.

We believe there are several potential avenues for improvement inspired by \cite{khaled2022faster}. 
1) \citet{khaled2022faster} introduced the requirement that each individual function is strongly convex (see \cite[Assumption 2]{khaled2022faster}). 
However, this constraint is absent in prior works.
Notably, in the context of full participation,
even non-convexity is deemed acceptable\footnote{Readers can check that the proof of \cite{kovalev2022optimal} only requires $f_1(\cdot)+\frac{1}{2\theta}\norm{\cdot}^2$ is strongly convex, which can be guaranteed by $\delta$-second-order similarity since $f$ is $\mu$-strongly convex and $\theta=1/(2\delta)$ therein.}.
A prominent example is the shift-and-invert approach to solving PCA \cite{sorensen2002numerical,garber2016faster}, where each component is smooth and non-convex, but the average function remains convex.
Thus we doubt the necessity of requiring strong convexity for individual components.
2) In hindsight, it seems that the directly accelerated SVRP could only achieve a bound of $\tilde{\gO}(n+\sqrt{n}\cdot\delta/\mu)$ based on the current analysis, which is far from being satisfactory compared to
its Catalyzed version. 
Consequently, there might be room for the development of a more effective algorithm for direct acceleration.
3) It is essential to note that the Catalyst framework introduces an additional log term in the overall complexity,
along with the challenge of parameter tuning.
This aspect is discussed in detail in \cite[Section 1.2]{allen2018katyusha}.
Therefore, we intend to address the aforementioned concerns, particularly on designing directly accelerated methods under the second-order similarity assumption.
\subsection{Main Contributions}
In this paper, we address the above concerns under the average similarity condition. 
Our contributions are presented
in detail below and we provide a comparison with previous works in Table \ref{table:res-com}:
\begin{itemize}
    \item First, we combine gradient sliding and client sampling techniques to develop an improved non-accelerated algorithm named SVRS (Algorithms \ref{algo:l-svrs}). SVRS achieves a communication complexity of $\tilde{\gO}(n+\sqrt{n}\cdot \delta/\mu)$, surpassing SVRP in ill-conditioned cases. 
    Notably, this rate does not need component strong convexity and applies to the function value gap instead of the parameter distance.
    \item Second, building on SVRS, we employ a classical interpolation framework motivated by Katyusha X \cite{allen2018katyusha} to introduce the directly accelerated SVRS (AccSVRS, Algorithm \ref{algo:acc-svrs}).
    AccSVRS achieves the same communication bound of $\tilde{\gO}(n+n^{3/4}\sqrt{\delta/\mu})$ as Catalyzed SVRP.
    Specifically, our bound is entirely smoothness-free and slightly outperforms Catalyzed SVRP, featuring a log improvement and not requiring component strong convexity.
    \item Third, by considering the proximal incremental first-order oracle in the centralized distributed framework, we establish a lower bound, which nearly matches the upper bound of AccSVRS in ill-conditioned cases.
\end{itemize}

\begin{table}[t]
\renewcommand{\arraystretch}{1.6}
\centering
\caption{Comparison of communication under similarity for the strongly convex objective.}
\small
\begin{threeparttable}
    \begin{tabular}{|c|c|c|c|}
        \hline
        &\textbf{Method/Reference} & \textbf{Communication complexity} & \textbf{Assumptions} \\
        \hline
        \multirow{2}{*}{\makecell{No \\ Sampling}} & AccExtragradient \cite{kovalev2022optimal} & $\gO \left( n\sqrt{\frac{\delta}{\mu}} \log \frac{1}{\varepsilon} \right)$  & SS only for $f_1$ \\
        \cline{2-4}
        &Lower bound \cite{arjevani2015communication} & $\Omega \left( n\sqrt{\frac{\delta}{\mu}} \log \frac{1}{\varepsilon} \right)$ & SS for $f_i$'s \\
        \hline
        \multirow{5}{*}{\makecell{Client \\ Sampling}} & SVRP \cite{khaled2022faster} & $\gO \left(\left(n+\frac{\delta^2}{\mu^2}\right) \log \frac{1}{\varepsilon} \right)$\tnote{{\color{blue}(1)}} & SC for $f_i$'s, AveSS \\
        \cline{2-4}
        & Catalyzed SVRP \cite{khaled2022faster} & $\gO \left(\left(n+n^{3/4}\sqrt{\frac{\delta}{\mu}}\right) \log \frac{1}{\varepsilon} \;\myred{ \log\frac{L}{\mu}}\right)$\tnote{{\color{blue}(2)}} \quad \qquad & SC for $f_i$'s, AveSS \\
        \cline{2-4}
        & SVRS (Thm \ref{thm:svrs-rate}) & $\gO \left(\left(n+\sqrt{n} \cdot\frac{\delta}{\mu}\right) \log \frac{1}{\varepsilon} \right)$ & AveSS \\
        \cline{2-4}
        & AccSVRS (Thm \ref{thm:accsvrs-rate}) & $\gO \left(\left(n+n^{3/4} \sqrt{\frac{\delta}{\mu}}\right) \log \frac{1}{\varepsilon} \right)$ & AveSS \\
        \cline{2-4}
        & Lower bound (Thm \ref{thm:lower}) & $\Omega \left(n+n^{3/4} \sqrt{\frac{\delta}{\mu}}\log \frac{1}{\varepsilon} \right)$\tnote{{\color{blue}(3)}} & AveSS \\
        \hline
    \end{tabular}   
    \begin{tablenotes}
    {\scriptsize
        \item [] \tnote{{\color{blue}(1)}} The rate only applies to $\E\norm{\vx_k-\vx_*}^2$, otherwise it would introduce $L$ in the log term; \tnote{{\color{blue}(2)}} The term $\log(L/\mu)$ comes from the Catalyst framework. See Appendix \ref{app:Cata-svrp} for the detail. \tnote{{\color{blue}(2, 3)}} Here we only list the rates of the common ill-conditioned case: $\mu = \gO(\delta/\sqrt{n})$. See Appendices for the remaining case.
        {\em Notation:} $\delta$=similarity parameter (both for SS and AveSS), $L$=smoothness constant of $f$, $\mu$=strong convexity constant of $f$(or $f_i$'s), $\varepsilon$=error of the solution for $\E f(\vx_k){-}f(\vx_*)$. Here $L\geq \delta \geq \mu \gg \epsilon > 0$.
        {\em Abbreviation:}	SC=strong convexity, SS=second-order similarity, AveSS=average SS. 
    }
    \end{tablenotes}
\end{threeparttable}\label{table:res-com}
\end{table}

\subsection{Related Work}
\paragraph{Gradient sliding/Oracle Complexity Separation.}
For optimization problems with a separated structure or multiple building blocks, such as Eq.~\eqref{eq:obj},  there are scenarios where computing the gradients/values of some parts (or the whole) is more expensive than the others (or a partial one).    
In response to this challenge, techniques such as the gradient-sliding method \cite{lan2016gradient} and the concept of oracle complexity separation \cite{ivanova2022oracle} have emerged.
These methods advocate for the infrequent use of more expensive oracles compared to their less resource-intensive counterparts.
This strategy has found applications in
zero-order \cite{beznosikov2020derivative,dvinskikh2020accelerated,ivanova2022oracle,stepanov2021one},
first-order
\cite{lan2016gradient,lan2022accelerated,lan2016conditional,kamzolov2020optimal} and 
high-order methods \cite{kamzolov2020optimal,grapiglia2022adaptive,ahookhosh2021high}, as well as in addressing
saddle point problems
\cite{alkousa2020accelerated,beznosikov2021distributed}. 
Our algorithms can be viewed as a variance-reduced version of gradient sliding tailored to leverage the similarity assumption.
\paragraph{Distributed optimization under similarity.} 
Distributed optimization has a long history with a plethora of existing 
works and surveys. 
To streamline our discussion,
we only list the most relevant references, particularly under the similarity and strong convexity assumptions.
In the full participation setting, 
which involves deterministic methods,
the first algorithm credits to DANE \cite{shamir2014communication}, though its analysis is limited to quadratic objectives. 
Subsequently, AIDE \cite{reddi2016aide}, DANE-LS and DANE-HB \cite{yuan2020convergence} improved the rates for quadratic objective;
Disco \cite{zhang2015disco} SPAG \cite{hendrikx2020statistically}, ACN \cite{agafonov2021accelerated} and DiRegINA \cite{daneshmand2021newton} improved the rates for self-concordant objectives.
As for general strongly convex objectives, \citet{sun2022distributed} introduced the SONATA algorithm, and
\citet{tian2022acceleration} proposed accelerated SONATA.
However, their complexity bounds include additional log factors. 
These factors have recently been removed by Accelerated Extragradient \cite{kovalev2022optimal}, whose complexity bound perfectly matches the lower bound in \cite{arjevani2015communication}.
We highly recommend the comparison of rates in \cite[Table 1]{kovalev2022optimal} for a comprehensive overview.
Once the discussion of deterministic methods is concluded,
\citet{khaled2022faster}  shifted their focus to  
stochastic methods using client sampling. 
They proposed SVRP and its Catalyzed version, both of which exhibited superior rates compared to deterministic methods.


%% file: upperbound.tex
\section{Preliminaries}\label{sec:pre}
\paragraph{Notation.} 
We denote vectors by lowercase bold letters (e.g., $ \vw, \vx$), and matrices by capital bold letters (e.g., $ \mA, \mB$).
We let $\|\cdot\|$ be the $\ell_2$-norm for vectors, or induced $\ell_2$-norm for a given matrix: $\norm{\mA} = \sup_{\vu\neq 0} \norm{\mA\vu}/\norm{\vu}$. 
We abbreviate $[n]=\{1, \ldots, n\}$ and $\mI_d\in \sR^{d \times d}$ is the identity matrix.
We use $\vzero$ for the all-zero vector/matrix, whose size will be specified by a subscript, if necessary, and otherwise is clear from the context.
We denote $\mathrm{Unif}(\gS)$ as the uniform distribution over set $\gS$. 
We say $T \sim \geo(p)$ for $p \in (0,1]$ if $\sP(T=k) = (1-p)^{k-1} p, \forall k \in \{1,2,\dots\}$, i.e., $T$ obeys a geometric distribution. We adopt $\E_k$ as the expectation for all randomness appeared in step $k$, and $\indicatorr_{A}$ as the indicator function on event $A$, i.e., $\indicatorr_{A}=1$ if event $A$ holds, and $0$ otherwise.
We use $\gO(\cdot), \Omega(\cdot), \Theta(\cdot)$ and $\tilde{\gO}(\cdot)$ notation to hide universal constants and log-factors.
We define the Bregman divergence induced by a differentiable (convex) function $h\colon \sR^d \to \sR$ as $D_h(\vx,\vy):= h(\vx)-h(\vy)-\dotprod{\nabla h(\vy),\vx-\vy}$.

\paragraph{Definitions.} 
We present the following common definitions used in this paper.
\begin{defn}\label{defn:sc}
    A differentiable function $g\colon\sR^d\to\sR$ is $\mu$-strongly convex (SC) if
    \begin{equation}\label{eq:sc}
        g(\vy) \geq g(\vx) + \dotprod{\nabla g(\vx), \vy-\vx} +\frac{\mu}{2} \norm{\vy-\vx}^2, \forall \vx,\vy \in \sR^d.
    \end{equation}
    Particularly, if $\mu=0$, we say that $g$ is convex.
\end{defn}
\begin{defn}\label{defn:sm}
    A differentiable function $g\colon\sR^d\to\sR$ is $L$-smooth if
    \begin{equation}\label{eq:sm}
        g(\vy) \leq g(\vx) + \dotprod{\nabla g(\vx), \vy-\vx} + \frac{L}{2} \norm{\vy-\vx}^2, \forall \vx,\vy \in \sR^d.
    \end{equation}
\end{defn}

There are many basic inequalities involving strong convexity and smoothness, see \cite[Appendix A.1]{d2021acceleration} for an introduction.
Next, we present the definition of second-order similarity in distributed optimization.
\begin{defn}\label{defn:ave-ss}
    The differentiable functions $f_i$'s satisfy $\delta$-average second-order similarity (AveSS) if the following inequality holds for $f_i$'s and $f = \frac{1}{n}\sum_{i=1}^n f_i$: 
    \begin{equation}\label{eq:avess}
        \textit{(AveSS)} \quad \frac{1}{n} \sum_{i=1}^n \norm{\left[\nabla [f_i-f](\vx)- \nabla [f_i-f](\vy)\right]}^2 \leq \delta^2 \norm{\vx-\vy}^2, \forall \vx, \vy \in \sR^d.
    \end{equation}
\end{defn}
\begin{defn}\label{defn:ss}
    The differentiable functions $f_i$'s satisfy $\delta$-component second-order similarity (SS) if the following inequality holds for $f_i$'s and $f = \frac{1}{n}\sum_{i=1}^n f_i$:
    \begin{equation}\label{eq:ss}
        \textit{(SS)} \quad \norm{\left[\nabla [f_i-f](\vx)- \nabla [f_i-f](\vy)\right]}^2 \leq \delta^2 \norm{\vx-\vy}^2, \forall \vx, \vy \in \sR^d, i \in [n].
    \end{equation}
\end{defn}

Definitions \ref{defn:ave-ss} and \ref{defn:ss} first appear in \cite{khaled2022faster}, which is an analogy to (average) smoothness in prior literature \cite{zhou2019lower}.
Particularly, $f_i$'s satisfy $\delta$-AveSS implies that $(f{-}f_i)$'s satisfy $\delta$-average smoothness, while $f_i$'s satisfy $\delta$-SS implies that $(f{-}f_i)$'s satisfy $\delta$-smoothness.
Additionally, many researchers \cite{karimireddy2020scaffold,arjevani2015communication,shamir2014communication,zhang2015disco,sun2022distributed,kovalev2022optimal} use the equivalent one defined by Hessian similarity (HS) if assuming that $f_i$'s are twice differentiable. Thus we also list them below and leave the derivation in Appendix \ref{app:dev-hs}.
\begin{equation}\label{eq:hes-ss}
    \text{(AveHS)} \ \norm{\frac{1}{n}\sum_{i=1}^n\left[\nabla^2 f_i(\vx)-\nabla^2 f(\vx)\right]^2}\leq \delta^2;   \text{(HS)} \ \norm{\nabla^2 f_i(\vx) - \nabla^2 f(\vx)} \leq \delta, \forall i \in [n].
\end{equation}
Since our algorithm is a first-order method, we adopt the gradient description of similarity (Definitions \ref{defn:ave-ss} and \ref{defn:ss}) without assuming twice differentiability for brevity.

As mentioned in \cite{arjevani2015communication,sun2022distributed}, if $f_i$'s satisfy $\delta$-AveSS (or SS), and $f$ is $\mu$-strongly convex and $L$-smooth, then generally $L \gg \delta \gg \mu > 0$ for large datasets in practice. 
Therefore, researchers aim to develop algorithms that
achieve communication complexity solely related to $\delta, \mu$ (or log terms of $L$).
This is also our objective.
To finish this section, 
we will clarify several straightforward yet essential propositions,
and the proofs are deferred to Appendix \ref{app:aux}.
\begin{prop}\label{prop:ss-sc}
    We have the following properties among SS, AveSS, and SC:
    1) $\delta$-SS  implies $\delta$-AveSS, but $\delta$-AveSS only implies $\sqrt{n}\delta$-SS. 
    2) If $f_i$'s satisfy $\delta$-SS and $f$ is $\mu$-strongly convex, then for all $i \in [n], f_i(\cdot)+\frac{\delta-\mu}{2}\norm{\cdot}^2$ is convex, i.e., $f_i$ is $(\delta-\mu)$-almost convex \cite{carmon2018accelerated}.
\end{prop}
\section{Algorithm and Theory}\label{sec:algo-theory}
In this section, we introduce our main algorithms,
which are developed
to solve the distributed optimization problem in Eq.~\eqref{eq:obj} under Assumption \ref{ass:1} below: 
\begin{assume}\label{ass:1}
    We assume that $f_i$'s satisfy $\delta$-AveSS, and $f$ is $\mu$-strongly convex with $\delta \geq \mu > 0$.
\end{assume}

Assumption \ref{ass:1} does not need each $f_i$ to be $\mu$-strongly convex. 
In fact, it is acceptable that $f_i$'s are non-convex, since by Proposition \ref{prop:ss-sc}, $f_i$'s are $(\sqrt{n}\delta{-}\mu)$-almost convex \cite{carmon2018accelerated}.
In the following, we first propose our new algorithm SVRS, which combines the techniques of gradient sliding and variance reduction, resulting in improved rates. 
Then we establish the directly accelerated method motivated by \cite{allen2018katyusha}. 

\subsection{No Acceleration Version: SVRS}\label{sec:svrs}

We first show the one-epoch Stochastic Variance-Reduced Sliding (\svrs) method in Algorithm~\ref{algo:l-svrs}. 
Before delving into the theoretical analysis, we present some key insights into our method.
These insights aim to enhance comprehension and facilitate connections with other algorithms.

\paragraph{Variance Reduction.} Our algorithm can be viewed as adding variance reduction from \cite{kovalev2022optimal}.
Besides the acceleration step, the main difference lies in the proximal step, where \citet{kovalev2022optimal} solved:
\[ \vx_{t+1} \approx \argmin_{\vx \in \sR^d} B_{\theta}^t(\vx):=\dotprod{\nabla f(\vx_t)-\nabla f_{1}(\vx_t), \vx-\vx_t} + \frac{1}{2\theta}\norm{\vx-\vx_t}^2+f_1(\vx). \]
To save the heavy communication burden of calculating $\nabla f(\vx_t)$, we apply client sampling by selecting a random 
$\nabla f_{i_t}(\vx_t)$ in the $t$-th step.
However, this substitution introduces significant noise.
To mitigate this, we incorporate a correction term
$\vg_t=\nabla f_{i_t}(\vw_0)-\nabla f(\vw_0)$ from previous wisdom 
\cite{johnson2013accelerating} to reduce the variance.

\paragraph{Gradient sliding.} Our algorithm can be viewed as adding gradient sliding from SVRP \cite{khaled2022faster}.
The main difference also lies in the proximal point problem, where \citet{khaled2022faster} solved:
\[ \vx_{t+1} \approx \argmin_{\vx \in \sR^d} C_{\theta}^t(\vx):=\dotprod{-\vg_t, \vx-\vx_t} + \frac{1}{2\theta}\norm{\vx-\vx_t}^2+f_{i_t}(\vx). \]
Here we adopt a fixed proximal function $f_1$ instead of $f_{i_t}$, which can be viewed as approximating
$ f_{i_t}(\vx) {\approx} f_{1}(\vx) {+} [f_{i_t} {-} f_1](\vx_t) {+} \dotprod{\nabla [f_{i_t} {-} f_1](\vx_t), \vx {-} \vx_t} {+} \frac{1}{2\theta'} \norm{\vx {-} \vx_t}^2$ with a properly chosen $\theta' > 0$.
Such a modification is motivated by \cite{kovalev2022optimal}, where they reformulated the objective as $f(\vx)=[f(\vx)-f_1(\vx)]+f_1(\vx)$.
Thus they could employ gradient sliding to skip heavy computations of $\nabla [f-f_1](\vx)$ by utilizing the easy computations of $\nabla f_1(\vx)$ more times.
Fixing the proximal function $f_1$ leads to the same metric space owned by $f_1$ in each step, which could benefit the analysis and alleviate the requirements on $f_i$'s compared to SVRP. 
Indeed, in our setting $f_1$ can be replaced by any other \textbf{fixed} client $f_b, b \in [n]$. In this case, the master node would be $f_b$ instead of $f_1$.


\paragraph{Bregman-SVRG.}
Our algorithm can be viewed as the classical Bregman-SVRG \cite{dragomir2021fast} with the reference function $f_1(\cdot)+\frac{1}{2\theta}\norm{\cdot}^2$ after introducing the Bregman divergence:
\begin{equation*}
\begin{aligned}
    \vx_{t+1} & \approx \argmin_{\vx \in \sR^d} A_{\theta}^t(\vx) \stackrel{\eqref{eq:alg_prox}}{=} \argmin_{\vx \in \sR^d} \dotprod{\nabla f_{i_t}(\vx_t)-\nabla [f_{i_t}-f](\vw_0), \vx-\vx_t} + D_{f_1(\cdot)+\frac{1}{2\theta}\norm{\cdot}^2}(\vx,\vx_t).
\end{aligned}
\end{equation*}
We need to emphasize that the proof of Bregman-SVRG requires additional structural assumptions \cite[Assumpotion 3]{dragomir2021fast}, which is not directly applicable in our setting.
Hence, the rigorous proof of Bregman-SVRG under our similarity assumption is still meaningful as far as we are concerned.  

\begin{algorithm}[t]
    \caption{\svrs$(f,\vw_0,\theta,p)$}
    \begin{algorithmic}[1]\label{algo:l-svrs}
        \STATE \textbf{Input:} $\vw_0 \in\sR^d$, $p \in (0, 1), \theta>0$
        \STATE Initialize $\vx_0 = \vw_0$, compute $\nabla f(\vw_0)$, and set $T \sim \geo\left(p\right)$
        \FOR{$ t = 0, 1, 2, \dots, T-1$}
        \STATE Sample $i_t \sim \text{Unif}([n])$ and compute $\vg_t = \nabla f_{i_t}(\vw_0)-\nabla f(\vw_0)$
        \STATE Approximately solve the local proximal point problem:
        \begin{align}\label{eq:alg_prox}          
        \vx_{t+1} \approx \argmin_{\vx \in \sR^d} A_{\theta}^t(\vx):=\dotprod{\nabla f_{i_t}(\vx_t)-\nabla f_{1}(\vx_t) - \vg_t , \vx-\vx_t} + \frac{1}{2\theta}\norm{\vx-\vx_t}^2+f_1(\vx) 
        \end{align}
        \ENDFOR
        \STATE \textbf{Output:} $\vx_{T}$
    \end{algorithmic}
\end{algorithm}

\subsubsection{Communication Complexity under Distributed Settings}\label{sec:com-com}
When applied to the distributed system, the communication complexity of \svrs can be described as follows:
At the beginning of each epoch, the master (corresponding to $f_1$) sends $\vw_0$ to all clients. Each client computes $\nabla f_i(\vw_0)$ from its local data and sends it back to the master. 
The master then builds $\nabla f(\vw_0)$ after collecting all $\nabla f_i(\vw_0)$'s.
The communication complexity is $2(n-1)$ in this case.
Next, the algorithm enters into the loop iterations. 
In each iteration, the master only sends current $\vx_t$ to the chosen client $i_t$. The $i_t$-th client computes $\nabla f_{i_t}(\vx_t)$ and sends it to the master (the first client). 
Then the master solves (inexactly) the local problem (Line 5 in Algorithm \ref{algo:l-svrs}) to get an inexact solution $\vx_{t+1}$.
The communication complexity is $2$ in this case.
Thus, the total communication complexity of \svrs \ is $2(n-1)+2T$. Note that $\E T = 1/p$ and generally $p=1/n$. 
We obtain that one epoch communication complexity is $4n-2$ in expectation.

We would like to emphasize that our setup differs from that in \cite{levy2023slowcal,mishchenko2022proxskip}, where the authors assume the nodes can perform calculations and transmit vectors in parallel.
We recognize the significance of both setups.
However, there are situations where communication is more expensive than computation.
For instance, in a business network or communication network the communication between any two nodes can result in charges and the risk of information leakage.
To mitigate these costs, we should reduce the frequency of communication.
Thus, we focus on the nonparallel setting.

\subsubsection{Convergence Analysis of SVRS}
Based on the one-epoch method \svrs, we could introduce our non-accelerated algorithm SVRS, 
which starts from $\vw_0 \in \sR^d$ and repeatedly
performs the update\footnote{See Algorithm~\ref{algo:svrs} in Appendix~\ref{app:proof_upper} for the details.} 
\[ \vw_{k+1} = \mathrm{SVRS}^{\mathrm{1ep}}(f, \vw_k, \theta, p), \ \forall k \geq 0. \] 
Now we derive the convergence rate of SVRS\footnote{Similar results for the popular loopless version \cite{kovalev2020don} can also be derived, see Appendix \ref{app:loopless-svrs} for the detail.}.
The main technique we apply is replacing the Euclidean distance with the Bergman divergence.
Denote the reference function
\begin{equation}\label{eq:func-h}
    h(\vx):=  f_1(\vx)+\frac{1}{2\theta}\norm{\vx}^2-f(\vx).
\end{equation}
By Assumption \ref{ass:1} and 1) in Proposition \ref{prop:ss-sc}, we see that $f_i$'s are $\sqrt{n}\delta$-SS. i.e., $[f_1-f](\cdot)$ is $(\sqrt{n}\delta)$-smooth.
Thus, $h(\cdot)$ is $(\frac{1}{\theta}{-}\sqrt{n}\delta)$-strongly convex and $(\frac{1}{\theta}{+}\sqrt{n}\delta)$-smooth if $\theta < \frac{1}{\sqrt{n}\delta}$, that is,
\begin{equation}\label{eq:div-con}
    0 \leq \frac{1-\sqrt{n}\theta\delta}{2\theta} \norm{\vx-\vy}^2 \stackrel{\eqref{eq:sc}}{\leq} D_{h}(\vx,\vy) \stackrel{\eqref{eq:sm}}{\leq} \frac{1+\sqrt{n}\theta\delta}{2\theta} \norm{\vx-\vy}^2.
\end{equation}
Hence, if $\sqrt{n}\theta\delta = \Theta(1)$, $h(\cdot)$ is nearly a rescaled Euclidean norm since its condition number related to $\norm{\cdot}$ is $\frac{1+\sqrt{n}\theta\delta}{1-\sqrt{n}\theta\delta} = \Theta(1)$.
Next, we employ the properties of the Bregman divergence $D_h(\cdot,\cdot)$ 
to build the one-epoch progress of \svrs as shown below: 
\begin{lemma}\label{lemma:one-step-loop}
    Suppose Assumption \ref{ass:1} holds. Let $\vw^+ =$ \svrs $(f,\vw_0,\theta,p)$ with $\theta = 1/(4\sqrt{n}\delta)$, and the approximated solution $\vx_{t+1}$ satisfies 
    \begin{equation}\label{eq:cond-app1}
        \norm{\nabla A_{\theta}^t(\vx_{t+1})}^2 \leq \frac{\mu}{20\theta} \norm{\vx_t-\argmin_{\vx \in\sR^d} A_{\theta}^{t}(\vx)}^2, \forall t \geq 0.
    \end{equation}
    Then for all $\vx \in \sR^d$ that is independent of the indices $i_1,i_2,\dots,i_T$ in \svrs$(f,\vw_0,\theta,p)$, we have
    \begin{equation}\label{eq:svrs-one}
        \E f(\vw^+) - f(\vx) \leq \E~ p \langle \vx - \vw_0, \nabla h(\vw^+) - \nabla h(\vw_0) \rangle - \big(p-\frac{2}{9n}\big) D_h(\vw_0, \vw^+) - \frac{2\mu\theta}{5}D_{h}(\vx, \vw^+).
    \end{equation}
\end{lemma}
\begin{remark}
    We note that some papers \cite{beznosikov2022compression, beznosikov2023similarity} assume the smoothness and convexity of component functions, and adopt local updates for solving the proximal step. However, we replace these assumptions with a proximal approximately solvable assumption
    \eqref{eq:cond-app1}, which could even cover some nonsmooth and non-convex but proximal trackable component functions. We regard our assumption as more essential since the local updates can be viewed as partially solving this proximal step.
\end{remark}
The proof of Lemma \ref{lemma:one-step-loop} is left in Appendix \ref{app:one-step-loop}. From Lemma \ref{lemma:one-step-loop}, we find a well-behaved proximal operator is sufficient to
ensure favorable progress.
Finally, we establish the 
convergence rate and 
communication complexity of the SVRS method, and the proof is deferred to Appendix \ref{app:svrs-rate}. 
\begin{thm}\label{thm:svrs-rate}
    Suppose Assumption \ref{ass:1} holds. If in \svrs (Algorithm \ref{algo:l-svrs}), the hyperparameters are set as $\theta = 1/(4 \sqrt{n} \delta), p=1/n$, and the approximate solution $\vx_{t+1}$ in each proximal step satisfies Eq.~\eqref{eq:cond-app1}. 
    Then for any error $\varepsilon>0$, when 
    \[ k \geq K_1 := \max\left\{2, \frac{5\delta}{\mu\sqrt{n}}\right\}\log\frac{3\left(1+\frac{\delta}{\mu\sqrt{n}}\right)[f(\vw_0)-f(\vx_*)]}{\varepsilon}, \]
    i.e., after $\tilde{\gO}(n+\sqrt{n} \delta / \mu)$ communications in expectation, we obtain that $\E f(\vw_k)-f(\vx_*) \leq \varepsilon$. 
\end{thm}

\begin{remark}
    Our results enjoy the following advantages over SVRP \cite{khaled2022faster}:
    The convergence of SVRP (\cite[Theorem 2]{khaled2022faster}) only applied to $\E \norm{\vw_k-\vx_*}^2$, which can also be derived by our results from strong convexity: $f(\vw_k)-f(\vx_*) \geq \frac{\mu}{2}\norm{\vw_k-\vx_*}^2$. 
    However, the reverse is not applicable since we do not assume the smoothness of $f$, or indeed the smoothness coefficient is very large.
    Moreover, for ill-conditioned problems (e.g., $\delta / \mu \gg \sqrt{n}$), our step size $1/(4\sqrt{n} \delta)$ is much larger than $\mu/(2\delta^2)$ used in SVRP, and the convergence rate is also faster than SVRP:
    $\tilde{\gO}\left(n+\sqrt{n}\delta/\mu\right)$ vs.~$\tilde{\gO}\left(n+\delta^2/\mu^2\right)$.
    Finally, we do not need the strong convexity assumption of component functions.
\end{remark}

\subsection{Acceleration Version: AccSVRS}\label{sec:accsvrs}

\begin{algorithm}[t]
    \caption{Accelerated SVRS (AccSVRS)}\label{algo:acc-svrs}
    \begin{algorithmic}[1]
        \STATE \textbf{Input:} $\vz_0 = \vy_0 \in\sR^d, p, \tau \in (0, 1), \alpha, \theta>0, K \in \{1,2,\dots\}$
        \FOR{$ k = 0,1,2,\dots, K-1$}
        \STATE $\vx_{k+1} = \tau \vz_k + (1 - \tau) \vy_k$
        \STATE $\vy_{k+1} =$ \svrs$(f, \vx_{k+1}, \theta, p)$
        \STATE $\bm{\gG}_{k+1} = p\left(\nabla [f_1-f_{j_k}](\vx_{k+1})-\nabla [f_1-f_{j_k}](\vy_{k+1})+\frac{1}{\theta}\left(\vx_{k+1}-\vy_{k+1}\right)\right), j_k \sim \text{Unif}([n])$
        \STATE $\vz_{k+1} = \argmin_{\vz\in\sR^d} \frac{1}{2\alpha}\norm{\vz - \vz_k}^2 + \dotprod{\bm{\gG}_{k+1}, \vz} + \frac{3\mu}{20}\norm{\vz-\vy_{k+1}}^2 = \frac{\vz_k+0.3\mu\alpha \vy_{k+1}-\alpha \bm{\gG}_{k+1}}{1+0.3\mu\alpha}$
        \ENDFOR
        \STATE \textbf{Output:} $\vy_{K}$
    \end{algorithmic}
\end{algorithm}

Now we apply the classical interpolation technique motivated by Katyusha X \cite{allen2018katyusha} to establish accelerated SVRS (AccSVRS, Algorithm \ref{algo:acc-svrs}). 
The main difference between AccSVRS and Katyusha X is due to the different choices of distance spaces.
Specifically, we adopt $D_h(\cdot, \cdot)$ instead of the Euclidean distance used in Katyusha X. Thus, the gradient mapping step (corresponding to Step 2 in \cite[(4.1)]{allen2018katyusha}) should be built on the reference function $h(\cdot)$ defined in Eq.~\eqref{eq:func-h}, i.e., $\nabla h(\vx_{k+1})-\nabla h(\vy_{k+1})$ instead of $(\vx_{k+1}-\vy_{k+1})/\theta$. 
Moreover, noting that $\nabla h(\cdot)$ could involve the heavy gradient computing part $\nabla f(\cdot)$, we further employ its stochastic version (Step 5 in Algorithm \ref{algo:acc-svrs}) to reduce the overall communication complexity.

Next, we delve into the convergence analysis. We first give the core lemma for AccSVRS, which is also motivated by the framework of Katyusha X \cite{allen2018katyusha}. The proof is deferred to Appendix \ref{app:one-loop}.
\begin{lemma}\label{lemma:one-loop}
    Suppose Assumption \ref{ass:1} holds, and $\theta = 1/(4\sqrt{n} \delta), p=1/n, \alpha \leq n\theta/(2\tau)$ in Algorithm~\ref{algo:acc-svrs}, where \svrs$(f, \vx_{k+1}, \theta, p)$ satisfies Eq.~\eqref{eq:cond-app1} in each iteration. 
    Then for all $\vx \in \sR^d$ that is independent of the random indices $i_1^{(k)},i_2^{(k)},\dots,i_T^{(k)}$ in \svrs$(f,\vx_{k+1}, \theta, p)$, we have that
    \begin{equation}\label{eq:acc-1}
        \E_k \frac{\alpha}{\tau} \left[f(\vy_{k+1}) - f(\vx)\right] \leq \E_k (1-\tau) \cdot \frac{\alpha}{\tau} \left[f(\vy_{k})-f(\vx)\right] + \frac{\norm{\vx - \vz_k}^2}{2} -\frac{1+0.3\mu\alpha}{2}\norm{\vx-\vz_{k+1}}^2.
    \end{equation}
\end{lemma}
Finally, we present the convergence rate and communication complexity of AccSVRS based on Lemma \ref{lemma:one-loop}, and the proof is left in Appendix \ref{app:accsvrs-rate}.
\begin{thm}\label{thm:accsvrs-rate}
    Suppose Assumption \ref{ass:1} holds. Consider AccSVRS with the following hyperparameters 
    \[ \theta=\frac{1}{4\sqrt{n}\delta}, p=\frac{1}{n}, \tau=\frac{1}{4}\min\left\{1, \frac{n^{1/4}}{2}\sqrt{\frac{\mu}{\delta}}\right\}, \alpha=\frac{\sqrt{n}}{8\delta\tau}, \]
    and Eq.~\eqref{eq:cond-app1} is satisfied in each iteration of \svrs$(f, \vx_{k+1}, \theta, p)$.	
    Then for any $\varepsilon > 0$, when 
    \[ k \geq K_2 := \max\left\{4, 8 n^{-1/4} \sqrt{\delta/\mu} \right\}\log\frac{2[f(\vy_0)-f(\vx_*)]}{\varepsilon}, \]
    i.e., after $\tilde{\gO}\left(n+n^{3/4}\sqrt{\delta/\mu}\right)$ communications in expectation, we obtain that 
    $ \E f(\vy_k)-f(\vx_*) \leq \varepsilon$.
\end{thm}

\begin{remark}
    Although 
    roughly the same as the communication complexity obtained by Catalyzed SVRP in \cite[Theorem 3]{khaled2022faster}, 
    our results have the following advantages.
    
    \textbf{Fewer assumptions}. Except for the strong convexity of $f$ and AveSS of $f_i$'s, we do not need to assume component strong convexity appearing in \cite[Assumption 2]{khaled2022faster}.
    
    \textbf{Inexact proximal step.} \citet[Theorem 3]{khaled2022faster} require exact evaluations of the proximal operator, though they mention that this is only for the convenience of analysis.
    Our framework allows approximated solutions in each proximal step, and the approximation criterion 
    \eqref{eq:cond-app1} is error-independent, i.e., irrelevant to the final error $\varepsilon$. 
    Since the local proximal function is strongly convex, we could solve the problem in a few steps if additionally assuming the smoothness of $f_1$.
    
    \textbf{Smoothness-free bound.} As shown in \cite[Appendix G.1]{khaled2022faster} or Appendix \ref{app:Cata-svrp}, even if an exact proximal step is allowed,  a dependence on the smoothness coefficient would be introduced in the total communication iterations of Catalyzed SVRP, though only in a log scale. Our directly accelerated method has no dependence on the smoothness coefficient.
\end{remark}

\subsection{Gradient Complexity under Smooth Assumption}\label{sec:grad-com}
Due to the importance of total computation in the machine learning and optimization community, we consider a more common setup by \textbf{additionally assuming} that $f_1$ is $L$-smooth with $L \geq \delta \geq \mu > 0$,
which together with Assumption \ref{ass:1} 
facilitates the quantification of
Eq.~\eqref{eq:cond-app1}. 
Then we can compute the total gradient complexity for AccSVRS as shown below.
By Proposition \ref{prop:ss-sc} and our assumptions, $A_{\theta}^t(\vx)$ is $(\frac{1}{\theta}{-}\sqrt{n}\delta)$-strongly convex and $(\frac{1}{\theta}{+}L)$-smooth.
Using accelerated methods starting from $\vx_t$, we can guarantee that Eq.~\eqref{eq:cond-app1} holds after 
$ T_{\mathrm{app}} = \tilde{\gO}\left(\sqrt{\frac{1+\theta L}{1-\sqrt{n}\theta\delta}}\right) = \tilde{\gO}\left(1+n^{-1/4} \sqrt{L/\delta}\right)$
iterations with the choice of $\theta$ in Theorem \ref{thm:accsvrs-rate}.
Hence, the total gradient calls in expectation are
\[ \gO(n T_{\mathrm{app}} \cdot K_{2}) = \tilde{\gO} \left(n + n^{3/4}\left(\sqrt{\delta/\mu} + \sqrt{L/\delta}\right) + \sqrt{nL/\mu}\right). \]
Since $\delta \in [\mu, L]$, we recover the optimal gradient complexity $\tilde{\gO}(n+n^{3/4}\sqrt{L/\mu})$ for the average smooth setting \cite[Table 1]{zhou2019lower} if neglecting log factors.
Particularly, when $\delta=\Theta(\sqrt{\mu L})$, we even obtain the nearly optimal gradient complexity $\tilde{\gO}(n + \sqrt{nL/\mu})$ for the component smooth setting \cite{han2021lower,hannah2018breaking,woodworth2016tight}. 
We leave the details in Appendix \ref{app:grad-com}.
Although the gradient complexity is not the primary focus of our work, 
we have demonstrated that the gradient complexity bound of AccSVRS is nearly optimal for certain values of $\delta$ in specific cases.

%% file: lower.tex
\renewcommand{\arraystretch}{1}
\section{Lower Bound}\label{sec:lower}
In this section, we establish the lower bound of the communication complexity, which nearly matches the upper bound of AccSVRS.

\subsection{Definition of Algorithms}\label{sec:lower:defn}
In this subsection, we specify the class of algorithms to which our lower bound can apply.
We first introduce 
the Proximal Incremental First-order
Oracle (PIFO) \citep{woodworth2016tight, han2021lower}, which is defined as 
    $\pifo_{f_i} (\vx, \gamma)
    = [ f_i(\vx), \nabla f_i(\vx), \prox_{f_i}^\gamma (\vx) ]$ with $\gamma > 0$. 
Here the proximal operator is 
defined as 
    $\prox_{f_i}^\gamma (\vx) := \argmin_{\vu} \{ f_i (\vu) + \frac{1}{2\gamma} \norm{\vx - \vu}^2 \}.
    $
In addition to the local zero-order and first-order information of $f_i$ at $\vx$,
the PIFO $\pifo_{f_i} (\vx, \gamma)$ 
also provides some global information through the proximal operator\footnote{
If we let $\gamma \rightarrow \infty,$ $\prox_{f_i}^\gamma (\vx)$ converges to the exact minimizer of $f_i$, irrelevant to the choice of $\vx$.}.
Then we assume the algorithm has access to the PIFO and the definition of algorithms is presented as follows.

\begin{defn}
\label{defn:lower_alg_informal}
    Consider a randomized algorithm $\gA$ to solve problem~\eqref{eq:obj}.
    Suppose the number of communication rounds is $T$.
    At the initialization stage, the master node~$1$ communicates with all the others.
    In round~$t~(0 \le t \le T-1)$, the algorithm samples a node~$i_t \sim \mathrm{Unif}([n])$, and  node~$1$ communicates with node~$i_t$.
    Then the algorithm samples a Bernoulli random variable $a_t$ with constant expectation $c_0/n$.
    If $a_t = 1$, 
    node~$1$ communicates with all the others.
    Define the information set $\gI_{t+1}$ as the set of
    all the possible points
    $\gA$ can obtain after round~$t$.
    The algorithm updates $\gI_{t+1}$ based on the linear-span operation and PIFO, and finally outputs a certain point in $\gI_T$.
\end{defn}
At the initialization stage, the communication cost is $2(n-1)$.
In each communication round,
the Bernoulli random variable $a_t$ determines whether the master node communicates with all the others, i.e., whether to calculate the full gradient.
Since $\E a_t = c_0 / n$,
the expected communication cost of each round is of the order $\Theta(1)$.
Thus the total communication cost is of the order $\Theta(n + T)$ and we can use 
$T$ to measure the communication complexity.
Moreover, one can check Algorithm~\ref{algo:acc-svrs} satisfies Definition~\ref{defn:lower_alg_informal}.
The formal definition and detailed analysis are deferred to Appendix~\ref{app:lower:defn}.

\newcommand{\sddots}{\raisebox{2pt}{$\scalebox{.75}{$\ddots$}$}}
\subsection{The Construction and Results}\label{sec:lower:construct}
In this section, we construct a hard instance of problem \eqref{eq:obj} and then use it to establish the lower bound.
Due to space limitations, we only present several key properties. 
The complete framework of construction is deferred to Appendix~\ref{app:lower:construct}.

Inspired by \citep{han2021lower}, we consider the class of matrices
$\mB(m, \hardzeta ) = \left[
\begin{smallmatrix}
     1 & -1 &        &        & \\
       & \sddots & \sddots & \\
       &        & 1      & {-}1 \\
       &        &        & \hardzeta
\end{smallmatrix}\right] \in \sR^{m \times m}$.
This class of matrices is widely used to establish lower bounds for minimax optimization problems \citep{zhang2021complexity, ouyang2021lower, zhang2022lower},
and $ \mA(m, \hardzeta) := \mB(m, \hardzeta)^\top \mB(m, \hardzeta)$ is the well-known tridiagonal matrix in the analysis of lower bounds for convex optimization \citep{nesterov2018lectures, lan2018optimal, carmon2021lower}.
Denote the $l$-th row of $\mB(\harddim, \hardzeta)$ as $ \vb_l(\harddim, \hardzeta)^\top$.
We partition the row vectors of $\mB(\harddim, \hardzeta)$
according to the index sets 
$\gL_i = \{ l: 1 \le l \le \harddim,\, l \equiv i-1 \,(\mathrm{mod} \ (n-1) ) \}$ for $2 \le i \le n$
and $\gL_1 = \varnothing$\footnote{Such a way of partitioning is also inspired by \citep{han2021lower} and similar to that in \citep{kovalev2022optimalvi}.
However, our setting is different from theirs.
}.
These sets are mutually exclusive and their union is $[\harddim]$.
Then we consider the following problem
\begin{align}\label{eq:obj_hard_before_scale}
    \min_{\vx \in \sR^\harddim}  \hardr (\vx; \harddim, \hardzeta, \hardc) {=} \frac{1}{n} \sum_{i=1}^n \left[\hardr_i (\vx; \harddim, \hardzeta, \hardc) {:=}
    \begin{cases}
        \frac{\hardc}{2} \norm{\vx}^2 {-} n \dotprod{\ve_1, \vx}
        &\!\!\!\! \text{for } i = 1, \\
        \frac{\hardc}{2} \norm{\vx}^2 {+} \frac{n}{2} \sum\limits_{l \in \gL_i} \norm{\vb_{l}(\harddim, \hardzeta)^{\top} \vx}^2
        &\!\!\!\! \text{for } i \neq 1.
    \end{cases} \right]
\end{align}


Here $\ve_i \in \sR^\harddim$ denotes the unit vector with the $i$-th element equal to $1$ and others equal to $0$.
Then one can check 
$\hardr(\vx; \harddim, \hardzeta, \hardc)
= \frac{1}{2}\, \vx^\top \mA(\harddim, \hardzeta)\, \vx + \frac{c}{2} \norm{\vx}^2 - \dotprod{\ve_1, \vx}$.
Clearly, $\hardr$ is $\hardc$-strongly convex.
We can also determine the AveSS parameter as follows.
The proof is deferred to Appendix~\ref{app:lower:pf_avess}.
\begin{prop}\label{prop:hard_aveSS_except1}
    Suppose that $0 < \hardzeta \le \sqrt{2}$, $n \ge 3$ and $\harddim \ge 3$. Then 
    $\hardr_i$'s satisfy
    $\sqrt{ 8n + 4 }$-AveSS. 
\end{prop}

Define the subspaces $\{ \subspace_k \}_{k=0}^\harddim$ as $\subspace_0 = \{ \vzero \}$ and $\subspace_k = \spn\{ \ve_1, \ve_{2}, \dots, \ve_{k}\} $ for $1 \le k \le \harddim$.
The next lemma is fundamental to our analysis.
The proof is deferred to Appendix~\ref{app:lower:pf_expand_infor}.
\begin{lemma}\label{lem:hard_expand_infor}
Suppose the algorithm $\gA$ satisfies Definition~\ref{defn:lower_alg_informal} and apply it to solve  problem~\eqref{eq:obj_hard_before_scale} with $n\ge 3$ and $m \ge 4$.
We have (i) $\gI_0 = \subspace_1$. (ii) Suppose $\gI_t \subseteq \subspace_k$ ($1 \le k \le \harddim - 3$). If $i_t$ satisfies $k \in \gL_{i_t}$ or $a_t = 1$, then
$\gI_{t+1} \subseteq \subspace_{k+3}$; otherwise, 
$\gI_{t+1} \subseteq \subspace_k$.
\end{lemma}

Lemma~\ref{lem:hard_expand_infor} guarantees that in each round, only when a specific component is sampled or the full gradient is calculated, can we expand the information set by at most three dimensions.
For problem~\eqref{eq:obj_hard_before_scale}, we could never obtain an approximate solution unless we expand the information set to the whole space (see Proposition~\ref{prop:hard_scale_property} in Appendix~\ref{app:lower:construct}), while Lemma~\ref{lem:hard_expand_infor} implies that the process of expanding is very slow.
Then we can establish the following lower bound.
\begin{thm}\label{thm:lower}
For any $n \ge 3$, $ \delta, \mu > 0$, 
algorithm $\gA$
satisfying Definition~\ref{defn:lower_alg_informal} and
sufficiently small $\eps > 0$,
there exists a rescaled version of problem~\eqref{eq:obj_hard_before_scale} such that (i) Assumption~\ref{ass:1} holds; 
(ii) In order to find an $\eps$-suboptimal solution $\hat{\vx}$ such that $\E r(\hat\vx) - \min_\vx r(\vx) < \eps$ by $\gA$,
the communication complexity in expectation is $\tilde{\Omega} ( n + n^{3/4} \sqrt{\delta / \mu} )$.
\end{thm}
This lower bound nearly matches the upper bound in Theorem~\ref{thm:accsvrs-rate} up to log factors, implying Algorithm~\ref{algo:acc-svrs} is nearly optimal in terms of communication complexity.
The detailed statement and proof are deferred to Appendices~\ref{app:lower:construct} and \ref{app:lower:pf_thm_lower}.

%% file: experiments.tex
\section{Experiments}\label{sec:exp}

\begin{figure}[t]
    \centering
    \hspace{-5pt}
    \begin{subfigure}[b]{0.33\textwidth}
        \includegraphics[width=\linewidth]{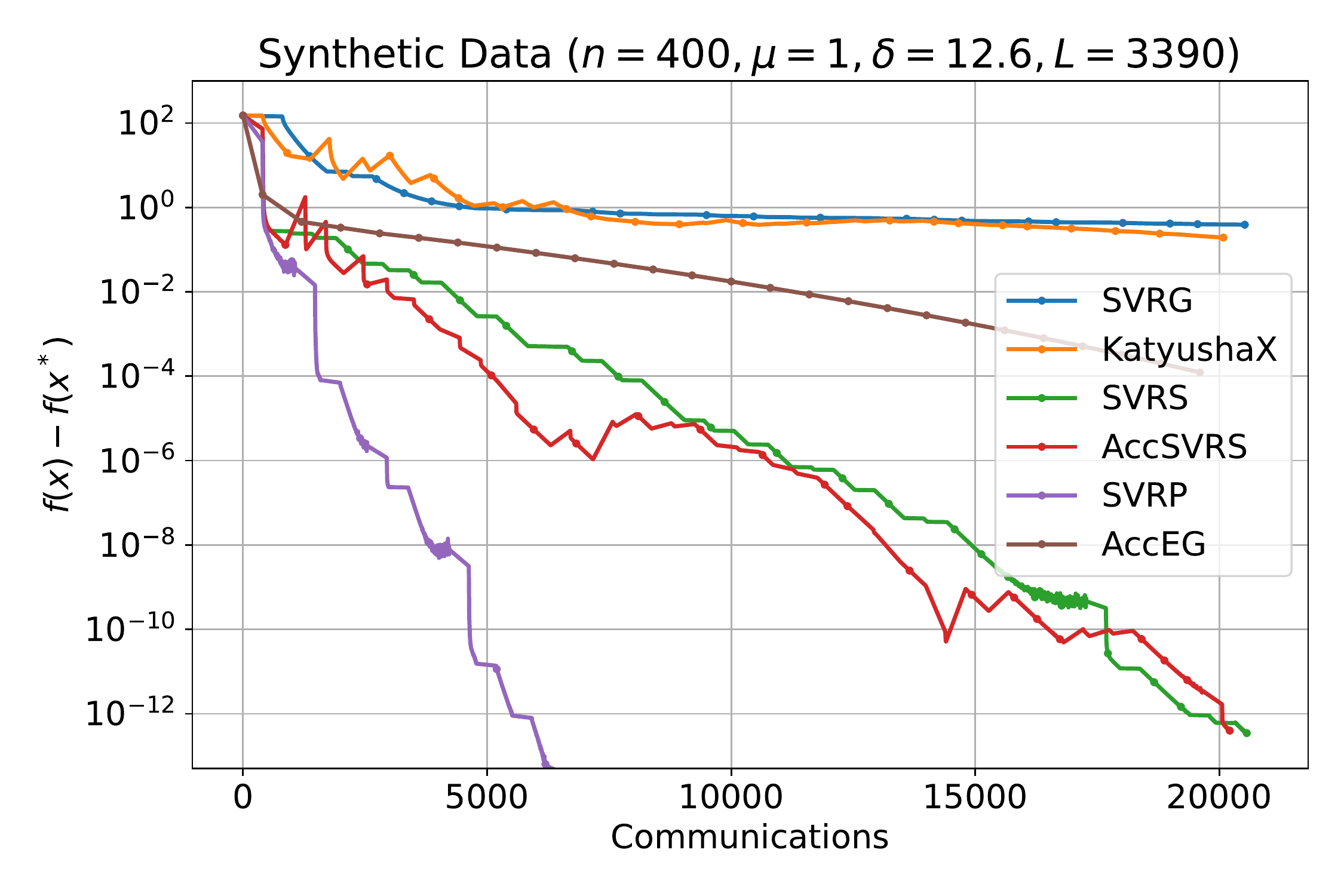}
    \end{subfigure}
    \hspace{-5pt}
    \begin{subfigure}[b]{0.33\textwidth}
        \includegraphics[width=\linewidth]{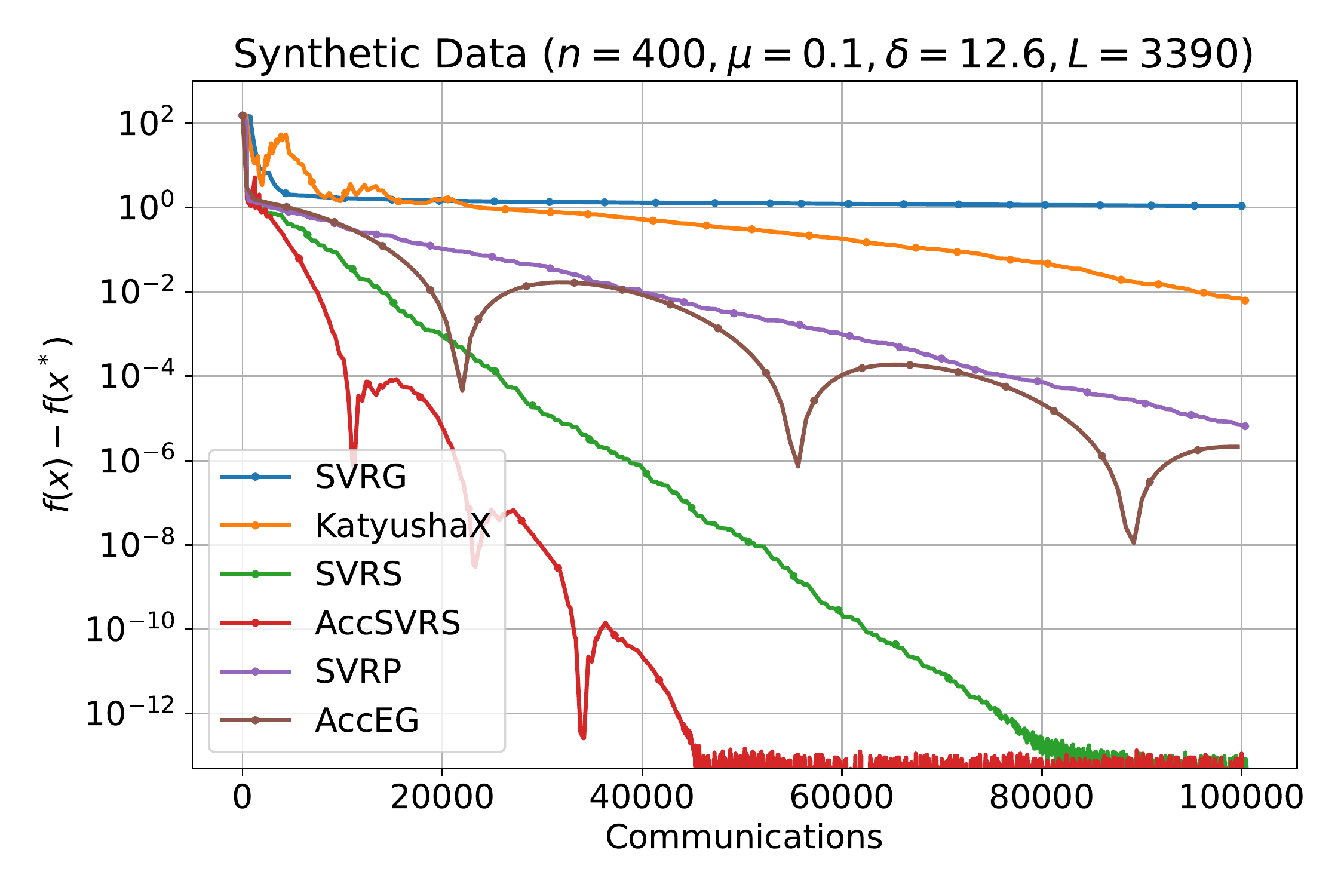}
    \end{subfigure}	
    \hspace{-5pt}
    \begin{subfigure}[b]{0.33\textwidth}
        \includegraphics[width=\linewidth]{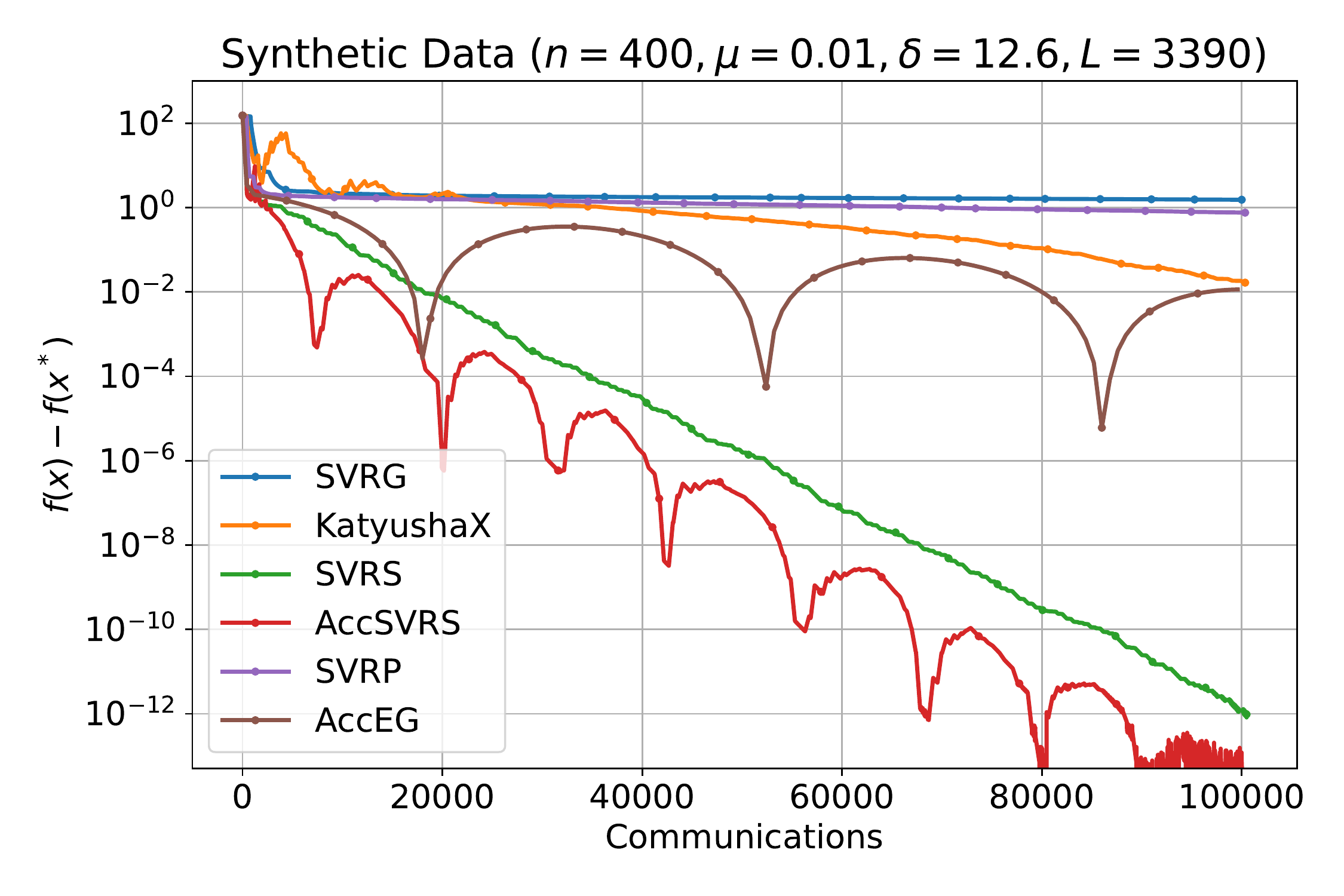}
    \end{subfigure}
    \hspace{-5pt}
    \caption{Numerical experiments on synthetic data. The corresponding coefficients are shown in the title of each graph. We plot the function gap on a log scale versus the number of communication steps, where one exchange of vectors counts as a communication step.}
    \label{fig:exp}
\end{figure}

To demonstrate the advantages of our algorithms, we conduct the same numerical experiments as those in \cite{kovalev2022optimal,khaled2022faster}. 
We focus on the linear ridge regression problem with $\ell_2$ regularization, where the average loss $f$ has the formulation:
$ f(\vx) = \frac{1}{n}\sum_{i=1}^n \big[f_i(\vx):= \frac{1}{m} \sum_{j=1}^{m}\left(\vz_{i,j}^\top\vx-y_{i,j}\right)^2+\frac{\mu}{2}\norm{\vx}^2\big]$.
Here $\vz_{i,j} \in \sR^d$ and $y_{i,j} \in \sR, \forall i \in [n], j \in [m]$ serve as the feature and label respectively, and $m$ can be viewed as data size in each local client. 
We consider a synthetic dataset generated by adding a small random noise matrix to the center matrix,
ensuring a small $\delta$.
To capture the differences in convergence rates between our methods and SVRP caused by different magnitudes of $\mu$, we vary $\mu=10^{-i}, i\in \{0,1,2\}$. 
We compare our methods (SVRS and AccSVRS) against SVRG, KatyushaX, SVRP (Catalyzed SVRP is somehow hard to tune so we omit it), and Accelerated Extragradient (AccEG) using their theoretical step sizes, except that we scale the interpolation parameter $\tau$ in KatyushaX and AccSVRS for producing practical performance (see Appendix \ref{app:exp} for detail). 
From Figure \ref{fig:exp}, we can observe that for a large $\mu$, SVRP outperforms
existing algorithms due to its high-order dependence on $\mu$.
However, when the problem becomes ill-conditioned with a small $\mu$, AccSVRS exhibits significant improvements compared to other algorithms.

%% file: appendix.tex
\appendix
\section{Auxiliary Results}\label{app:aux}
\begin{prop}[Three-point identity {\cite[Lemma3.1]{chen1993convergence}}]
    Given a differentiable function $h\colon \sR^d \to \sR$, we have the following equality:
    \begin{equation}\label{eq:div-eq}
        \dotprod{\vx-\vy,\nabla h(\vy)-\nabla h(\vz)} = D_h(\vx,\vz)-D_h(\vx,\vy)-D_h(\vy,\vz), \forall \vx,\vy,\vz \in \sR^d.
    \end{equation}
\end{prop}

\begin{prop}\label{app:ll-l}
    Denote $\forall i \in \sN, X_i = \begin{cases}
    1 & \textrm{with probability} \quad p \\
    0 & \textrm{with probability} \quad 1-p 
    \end{cases}$, and $X_1,X_2,\dots$ are independent and identically distributed random variables. Then $Y := \inf_{i} \{i:X_i=1\}\sim\geo(p)$.
\end{prop}
\begin{proof}
    We direct verify the probability distribution:
    \[ \sP(Y=k) = \prod_{i=1}^{k-1} \sP(X_i=0)\sP(X_{k}=1) = (1-p)^{k-1}p, \quad k \in \{1,2,\dots\}.\]
    Hence, we see that $Y \sim \geo(p)$.
\end{proof}

\begin{prop}[Proposition \ref{prop:ss-sc} in the main text]
    We have the following properties among SS, AveSS, and SC:
    1) The $\delta$-SS can deduce $\delta$-AveSS, but $\delta$-AveSS can only deduce $\sqrt{n}\delta$-SS. 
    2) If $f_i$'s satisfy $\delta$-SS and $f$ is $\mu$-strongly convex, then for all $i \in [n], f_i(\cdot)+\frac{\delta-\mu}{2}\norm{\cdot}^2$ is convex, i.e., $f_i$ is $(\delta-\mu)$-almost convex \cite{carmon2018accelerated}.
\end{prop}
\begin{proof}
    1) The first part ``$\delta$-SS $\Rightarrow \delta$-AveSS'' is trivial. The second part is because for all $i \in [n]$,
    \[ \norm{\left[\nabla [f_i-f](\vx)- \nabla [f_i-f](\vy)\right]}^2 \leq \sum_{j=1}^n \norm{\left[\nabla [f_j-f](\vx)- \nabla [f_j-f](\vy)\right]}^2 \stackrel{\eqref{eq:avess}}{\leq} n\delta^2\norm{\vx-\vy}^2. \]
    Thus Eq.~\eqref{eq:ss} holds with parameter $\sqrt{n}\delta$.
	
    2) Since $f_i$'s satisfy $\delta$-SS, we get $\forall i \in [n], f-f_i$ is $\delta$-smooth, thus $\frac{\delta}{2}\norm{\vx}^2-[f(\vx)-f_i(\vx)]$ is convex (e.g., \cite[Theorem A.1]{d2021acceleration}). Moreover, we also have $f(\vx)-\frac{\mu}{2}\norm{\vx}^2$ is a convex function since $f$ is $\mu$-strongly convex (e.g., \cite[Theorem A.2]{d2021acceleration}). Therefore, we obtain that
    \[ f_i(\vx)+\frac{\delta-\mu}{2}\norm{\vx}^2 = \left( \frac{\delta}{2}\norm{\vx}^2-[f(\vx)-f_i(\vx)]\right)+\left(f(\vx)-\frac{\mu}{2}\norm{\vx}^2\right) \]
    is also convex. The proof is finished.
\end{proof}

\begin{lemma}[{\citet[Fact 2.3]{allen2018katyusha}}]\label{lemma:fact}
    Given sequence $D_0, D_1, \dots$ of reals, if $N \sim \text{Geom}(p)$, then
    \begin{equation}\label{eq:fact-geom}
        \E_N[D_{N-1} - D_{N}] = p \E[ D_0 - D_{N}], \E_N[D_{N-1}] = (1-p)\E[D_N] + p D_0
    \end{equation}
\end{lemma}

\begin{lemma}[{\citet[Lemma 2.4]{allen2018katyusha}}]
    If $g(\cdot)$ is proper convex and $\sigma$-strongly convex and 
    $\vz_{k+1} = \argmin_{\vz\in\sR^d} \frac{1}{2\alpha}\norm{\vz-\vz_k}^2 + \dotprod{\bm{\xi}, \vz} + g(\vz)$, then for every $\vx \in \sR^d$, we have
    \begin{equation}\label{eq:aux}
        \dotprod{\bm{\xi}, \vz_k-\vx} + g(\vz_{k+1})-g(\vx) \leq \frac{\alpha}{2}\norm{\bm{\xi}}^2+\frac{\norm{\vx-\vz_k}^2}{2\alpha}-\frac{(1+\sigma\alpha)\norm{\vx-\vz_{k+1}}^2}{2\alpha}.
    \end{equation}
\end{lemma}

\begin{lemma}[{\citet[Lemma~2.10]{han2021lower}}]\label{lem:geo}
    Let $\{Y_i\}_{i=1}^m$ be independent random variables such that $Y_i \sim \geo(p_i)$ with $p_i > 0$.
    Then for $m \ge 2$, we have
    \begin{align*}
        \sP \left( \sum_{i=1}^m Y_i > \frac{m^2}{4(\sum_{i=1}^m p_i)} \right) \ge \frac{1}{9}.
    \end{align*}
\end{lemma}

\section{Hessian Similarity}\label{app:dev-hs}
In this section, we show that AveHS (HS) defined in Eq.~\eqref{eq:hes-ss} is equivalent to AveSS (SS).
\begin{prop}
    For twice differentiability $f_i$'s and $f$, AveSS $\Leftrightarrow$ AveHS, SS $\Leftrightarrow$ HS.
\end{prop}
\begin{proof}
    Indeed, we only need to prove the following results for twice differentiability $g$:
    \begin{equation}\label{eq:tmp-ss-hs}
        \frac{1}{n}\sum_{i=1}^n \norm{\nabla g_i(\vx)-\nabla g_i(\vy)}^2 \leq \delta^2 \norm{\vy-\vx}^2 \Leftrightarrow \norm{\frac{1}{n} \sum_{i=1}^n\left(\nabla^2 g_i(\vx)\right)^2} \leq \delta^2, \forall \vx, \vy \in\sR^d.
    \end{equation}
    ``$\Rightarrow$'': Taking $\vy=\vx+t\vv, t \in \sR\backslash \{0\}, \vv \in\sR^d, \norm{\vv}=1$ and letting $t \to 0$, we get  
    \begin{eqnarray*}
        \delta^2 &=& \lim_{t\to 0}\frac{\delta^2 \norm{\vx-\vy}^2}{t^2} \geq \lim_{t\to 0} \frac{1}{n} \sum_{i=1}^n \norm{\frac{\left[\nabla g_i(\vx)- \nabla g_i(\vx + t\vv)\right]}{t}}^2 \\
        &=& \frac{1}{n} \sum_{i=1}^n \norm{\nabla^2 g_i(\vx) \vv}^2 = \vv^\top \left[\frac{1}{n} \sum_{i=1}^n\left(\nabla^2 g_i(\vx)\right)^2\right] \vv.
    \end{eqnarray*}
    The final equality uses the fact that $\nabla^2 g_i(\vx)$ is a symmetric matrix.
    Now by the arbitrary of $\vv \in \sR^d$ with $\norm{\vv}=1$, we get $ \norm{\frac{1}{n} \sum_{i=1}^n\left(\nabla^2 g_i(\vx)\right)^2}\leq \delta^2$.
	
    ``$\Leftarrow$'': We use the integral formulation:
    \begin{eqnarray*}
        &&\frac{1}{n}\sum_{i=1}^n \norm{\nabla g_i(\vx)-\nabla g_i(\vy)}^2 = 
        \frac{1}{n}\sum_{i=1}^n \norm{\int_0^1 \nabla^2 g_i(\vx+t(\vy-\vx))\left(\vy-\vx\right)dt }^2 \\
        &=& \left(\vy-\vx\right) \left[\frac{1}{n}\sum_{i=1}^n \int_0^1 \nabla^2 g_i(\vx+s(\vy-\vx))\nabla^2 g_i(\vx+t(\vy-\vx))dsdt\right]\left(\vy-\vx\right) \\
        &\stackrel{\rom{1}}{\leq}& \left(\vy-\vx\right) \left[\frac{1}{n}\sum_{i=1}^n \int_0^1 \left(\nabla^2 g_i(\vx+t(\vy-\vx)) \right)^2 dt\right]\left(\vy-\vx\right) \\ 
        &\leq& \int_0^1 \norm{\frac{1}{n} \sum_{i=1}^n\left(\nabla^2 g_i(\vx+t(\vy-\vx))\right)^2} \cdot \norm{\vy-\vx}^2 dt \leq \delta^2 \norm{\vy-\vx}^2,
    \end{eqnarray*}
    where $\rom{1}$ uses the inequality $\mA_t^2+\mA_s^2 \succeq \mA_s\mA_t+\mA_t\mA_s$ for symmetric matrices $\mA_s, \forall s \in [0, 1]$ since $(\mA_t-\mA_s)^2 \succeq 0$, and the final inequality uses the assumption.
	
    Hence, Eq.~\eqref{eq:tmp-ss-hs} is proved. Now choosing $g_i=f_i-f, \forall i \in [n]$, we obtain ``AveSS $\Leftrightarrow$ AveHS''. Additionally, letting $n=1$ and noting that $\norm{\left(\nabla^2 g_i(\vx)\right)^2} =\norm{\nabla^2 g_i(\vx)}^2$, we obtain ``SS $\Leftrightarrow$ HS''. The proof is finished.
\end{proof}

\section{Concrete Complexity of Catalyst SVRP}\label{app:Cata-svrp}
Inherited from the computation of \cite[Appendix G.1]{khaled2022faster}, we see that the total iterations of Catalyst SVRP is
\begin{eqnarray*}
    \E T_{\mathrm{iter}}^{\mathrm{total}} &=& 
    8\sqrt{\frac{\mu+\gamma}{\mu}} \max\left\{\frac{\delta^2}{(\gamma+\mu)^2},n\right\} \log \left(\frac{f(\vx_0)-f(\vx_*)}{\varepsilon}\cdot \frac{32(\mu+\gamma)}{\mu}\right) \log \iota, \\
    \iota &:=& A\left(\frac{2}{1-\rho}+\frac{2592\gamma}{\mu(1-\rho)^2(\sqrt{q}-\rho)^2}\right),
\end{eqnarray*}
where $\rho = \sqrt{q}/2 = \frac{\sqrt{\mu/(\mu+\gamma)}}{2} \in (0, \frac{1}{2})$, $A = \frac{L+\gamma}{\mu+\gamma}\left(1+\frac{(\gamma+\mu)^2 n}{\delta^2}\right)$.
Letting $\gamma = \max\left\{\frac{\delta}{\sqrt{n}}-\mu, 0\right\}$, we recover the complexity:
\begin{eqnarray*}
    \E T_{\mathrm{iter}}^{\mathrm{total}} &=&  
    8\max\left\{n,n^{3/4}\sqrt{\frac{\delta}{\mu}}\right\} \log \left(\max\left\{32, \frac{32\delta}{\mu\sqrt{n}}\right\} \cdot \frac{f(\vx_0)-f(\vx_*)}{\varepsilon} \right) \log\iota, \\
    \iota &=& A\left(\frac{2}{1-\rho}+\frac{2592\gamma}{\mu(1-\rho)^2(\sqrt{q}-\rho)^2}\right) = \Theta\left(A\left(1+\frac{\gamma(\mu+\gamma)}{\mu^2}\right) \right) = \Theta\left(\frac{A(\mu+\gamma)^2}{\mu^2}\right).
\end{eqnarray*}
When $\delta/\mu \leq \sqrt{n}$, leading to $\gamma=0$, then we get
\[ \frac{A(\mu+\gamma)^2}{\mu^2} = \frac{L}{\mu}\left(1+\frac{\mu^2 n}{\delta^2}\right) = \Theta\left(\frac{L\mu n}{\delta^2}\right).  \]
Thus, $\E T_{\mathrm{iter}}^{\mathrm{total}}  = \gO\left(\left(n+n^{3/4}\sqrt{\frac{\delta}{\mu}}\right)\log\frac{f(\vx_0)-f(\vx_*)}{\varepsilon}\myred{\log\frac{L\mu n}{\delta^2}}\right)$.

When $\delta/\mu \geq \sqrt{n}$, i.e., $\max\left\{n,n^{3/4}\sqrt{\frac{\delta}{\mu}}\right\}=n^{3/4}\sqrt{\frac{\delta}{\mu}}$, we get 
$\gamma=\frac{\delta}{\sqrt{n}}-\mu \leq L-\mu$ (note that $L\geq \delta \geq \mu, n\geq 1$ by assumption), leading to
\[ \frac{2L}{\mu} \leq \frac{A(\mu+\gamma)^2}{\mu^2} = \frac{2(L+\gamma)(\mu+\gamma)}{\mu^2} \leq \frac{4L^2}{\mu^2}. \] 
Thus, $\E T_{\mathrm{iter}}^{\mathrm{total}}  = \gO\left(\left(n+n^{3/4}\sqrt{\frac{\delta}{\mu}}\right)\log\frac{f(\vx_0)-f(\vx_*)}{\varepsilon}\myred{\log\frac{L}{\mu}}\right)$ (for small enough error $\varepsilon$).

\section{Proofs for Section \ref{sec:algo-theory}}\label{app:proof_upper}
The complete procedure of SVRS is presented in Algorithm~\ref{algo:svrs}.
\begin{algorithm}[t]
    \caption{Stochastic Variance-Reduced Sliding (SVRS)}\label{algo:svrs}
    \begin{algorithmic}[1]
        \STATE \textbf{Input:} $\vw_0  \in\sR^d, p \in (0, 1),  \theta>0, K \in \{1,2,\dots\}$
        \FOR{$ k = 0,1,2,\dots, K-1$}
        \STATE $\vw_{k+1} =$ \svrs$(f, \vw_{k}, \theta, p)$
        \ENDFOR
        \STATE \textbf{Output:} $\vw_{K}$
    \end{algorithmic}
\end{algorithm}
Before giving the omit proofs,
we need the following one-step lemma. 
\begin{lemma}\label{lemma:one-step}
    Suppose Assumption \ref{ass:1} holds. If the step size $\theta \leq 1/(2\sqrt{n}\delta)$ in \svrs (Algorithm \ref{algo:l-svrs}), then the following inequality holds for all $\vx \in \sR^d$ that is independent to the index $i_k$:
    \begin{align}
        \E_t [f(\vx_{t+1}){-}f(\vx)] \leq \ & \E_t D_{h}(\vx,\vx_t) {-} \left(1 {+} \frac{\mu\theta/2}{1+\sqrt{n}\theta\delta}\right) D_{h}(\vx,\vx_{t+1}) {+} \frac{2\theta^2\delta^2}{(1-\sqrt{n}\theta\delta)^2} D_h(\vw_0, \vx_t) \nonumber \\
        &+\frac{2+\mu\theta}{2\mu} \left[\norm{\nabla A_{\theta}^t(\vx_{t+1})}^2-\frac{\mu}{20\theta}\norm{\vx_t-\argmin_{\vx \in\sR^d} A_{\theta}^{t}(\vx)}^2 \right]. \label{eq:lemma-2}
    \end{align}
\end{lemma}

\begin{proof}
    First, note that
    \begin{equation}\label{eq:grad-A1}
    \begin{aligned}
        \nabla A_{\theta}^t(\vx) &= \frac{\vx-\vx_t}{\theta} + \nabla f_1(\vx) - \nabla f_1(\vx_t) + \nabla f_{i_t}(\vx_t)-[\nabla f_{i_t}(\vw_0)-\nabla f(\vw_0)] \\
        &= \nabla f(\vx) + \nabla h(\vx) - \nabla h(\vx_t) + \nabla (f_{i_t}-f)(\vx_t)- \nabla (f_{i_t}-f)(\vw_0).
    \end{aligned}
    \end{equation}
    Now we begin from the strong convexity of function $f$ in Assumption \ref{ass:1},
    \begin{eqnarray}
        && \E_t [f(\vx_{t+1})-f(\vx)] \stackrel{\eqref{eq:sc}}{\leq} \E_t \dotprod{\vx-\vx_{t+1}, -\nabla f(\vx_{t+1})} - \frac{\mu}{2}\norm{\vx_{t+1}-\vx}^2 \nonumber \\
        &\stackrel{\eqref{eq:grad-A1}}{=}& \E_t \dotprod{\vx-\vx_{t+1}, \nabla h(\vx_{t+1})-\nabla h(\vx_{t})} +\dotprod{\vx-\vx_{t+1}, \nabla (f_{i_t}-f)(\vx_{t})-\nabla (f_{i_t}-f)(\vw_0)} \nonumber \\
        &&- \dotprod{\vx-\vx_{t+1}, \nabla A_{\theta}^t(\vx_{t+1})} - \frac{\mu}{2}\norm{\vx-\vx_{t+1}}^2  \nonumber\\
        &\stackrel{\rom{1}}{=}& \E_t D_{h}(\vx,\vx_t){-}D_{h}(\vx,\vx_{t+1}){-}D_{h}(\vx_{t+1}, \vx_t) +\dotprod{\myred{\vx_t}{-}\vx_{t+1}, \nabla (f_{i_t}{-}f)(\vx_{t}){-}\nabla (f_{i_t}{-}f)(\vw_0)} \nonumber \\
        &&-\dotprod{\vx-\vx_{t+1}, \nabla A_{\theta}^t(\vx_{t+1})} - \frac{\mu}{2}\norm{\vx_{t+1}-\vx}^2 \nonumber \\
        &\leq& \E_t D_{h}(\vx,\vx_t)-D_{h}(\vx, \vx_{t+1})-D_{h}(\vx_{t+1},\vx_t) \nonumber \\
        &&+ \frac{1-\sqrt{n}\theta\delta}{4\theta} \norm{\vx_{t+1}-\vx_t}^2 + \frac{\theta}{1-\sqrt{n}\theta\delta}\norm{\nabla (f_{i_t}-f)(\vx_{t})-\nabla (f_{i_t}-f)(\vw_0)}^2 \nonumber \\
        &&+ \left[\frac{\mu}{4}\norm{\vx_{t+1}-\vx}^2 + \frac{1}{\mu}\norm{\nabla A_{\theta}^t(\vx_{t+1})}^2\right] - \frac{\mu}{2}\norm{\vx_{t+1}-\vx}^2 \label{eq:tmp-1}
    \end{eqnarray}
    where $\rom{1}$ uses Eq.~\eqref{eq:div-eq} and $\E_{i_t}\dotprod{\vx_t-\vx, \nabla (f_{i_t}-f)(\vx_t)-\nabla (f_{i_t}-f)(\vw_0)}=0$ since $\vx_t-\vx$ is independent to $i_t$ and the final inequality uses $\dotprod{\va,\vb} \leq t^2\norm{\va}^2+\frac{\norm{\vb}^2}{4t^2}$ twice.
    Next, we continue using Eq.~\eqref{eq:div-con} to convert $\norm{\cdot}$ with $D_h(\cdot,\cdot)$ by assumption $\theta \leq 1/(2\sqrt{n}\delta)$:
    \begin{eqnarray*}
        && \E_t [f(\vx_{t+1})-f(\vx)] \\
        &\stackrel{\eqref{eq:tmp-1}\eqref{eq:div-con}}{\leq}& \E_t D_{h}(\vx,\vx_t)-\left(1+\frac{\mu\theta/2}{1+\sqrt{n}\theta\delta}\right) D_{h}(\vx,\vx_{t+1}) - \frac{1-\sqrt{n}\theta\delta}{4\theta} \norm{\vx_{t+1}-\vx_t}^2 \\
        &&+ \frac{\theta}{1-\sqrt{n}\theta\delta}\norm{\nabla (f_{i_t}-f)(\vx_{t})-\nabla (f_{i_t}-f)(\vw_0)}^2+ \frac{1}{\mu}\norm{\nabla A_{\theta}^t(\vx_{t+1})}^2 \\
        &\stackrel{\eqref{eq:avess}}{\leq}& \E_t D_{h}(\vx,\vx_t)-\left(1+\frac{\mu\theta/2}{1+\sqrt{n}\theta\delta}\right) D_{h}(\vx,\vx_{t+1}) + \frac{\theta\delta^2}{1-\sqrt{n}\theta\delta}\norm{\vx_t-\vw_0}^2 \\
        &&- \frac{1-\sqrt{n}\theta\delta}{4\theta} \norm{\vx_{t+1}-\vx_t}^2 + \frac{1}{\mu}\norm{\nabla A_{\theta}^t(\vx_{t+1})}^2 \\
        &\stackrel{\eqref{eq:div-con}}{\leq}& \E_t D_{h}(\vx,\vx_t)-\left(1+\frac{\mu\theta/2}{1+\sqrt{n}\theta\delta}\right) D_{h}(\vx,\vx_{t+1}) + \frac{2\theta^2\delta^2}{(1-\sqrt{n}\theta\delta)^2} D_h(\vw_0, \vx_t) \\
        &&- \frac{1-\sqrt{n}\theta\delta}{4\theta} \norm{\vx_{t+1}-\vx_t}^2 + \frac{1}{\mu}\norm{\nabla A_{\theta}^t(x_{t+1})}^2.
    \end{eqnarray*}
    Finally, we show the error analysis if an approximate solution, i.e., $\norm{\nabla A_{\theta}^t(\vx_{t+1})} \neq 0$ is allowed. Using Proposition \ref{prop:ss-sc}, we see that $f_1(\vx)+\frac{\sqrt{n}\delta-\mu}{2}\norm{\vx}^2$ is a convex function, leading to $A^t_{\theta}(\vx)$ is $\left(\frac{1}{\theta}-\sqrt{n}\delta+\mu\right)$-strongly convex function.	
    Let $\hat{\vx}_{k+1} \in \argmin_{\vx\in\sR^d} A_{\theta}^t(\vx)$, i.e., $\nabla A_{\theta}^t(\hat{\vx}_{k+1})=0$.
    Since $\theta \leq 1/(2\sqrt{n}\delta)$, we can further bound the last two terms: 
    \begin{eqnarray*}
        &&-\frac{1-\sqrt{n}\theta\delta}{4\theta} \norm{\vx_{t+1}-\vx_t}^2 + \frac{1}{\mu} \norm{\nabla A_{\theta}^t(\vx_{t+1})}^2 \leq \frac{1}{\mu}\norm{\nabla A_{\theta}^t(\vx_{t+1})}^2 -\frac{1}{8\theta} \norm{\vx_{t+1}-\vx_t}^2 \\
        &\leq& \frac{1}{\mu}\norm{\nabla A_{\theta}^t(\vx_{t+1})}^2 + \frac{1}{8\theta} \norm{\vx_{t+1}-\hat{\vx}_{t+1}}^2-\frac{1}{16\theta}\norm{\hat{\vx}_{t+1}-\vx_t}^2 \\
        &\leq& \frac{1}{\mu}\norm{\nabla A_{\theta}^t(\vx_{t+1})}^2 + \frac{\theta}{8(1-(\sqrt{n}\delta-\mu)\theta)^2} \norm{\nabla A_{\theta}^t(\vx_{t+1})-\nabla A_{\theta}^t(\hat{\vx}_{t+1})}^2-\frac{1}{16\theta}\norm{\hat{\vx}_{t+1}-\vx_t}^2 \\
        &\leq& \frac{1}{\mu}\norm{\nabla A_{\theta}^t(\vx_{t+1})}^2 +\frac{\theta}{2} \norm{\nabla A_{\theta}^t(\vx_{t+1})}^2-\frac{1}{16\theta}\norm{\vx_t-\hat{\vx}_{t+1}}^2 \\
        &=& \frac{2+\mu\theta}{2\mu} \left[ \norm{\nabla A_{\theta}^t(\vx_{t+1})}^2-\frac{\mu}{8\theta(2+\mu\theta)}\norm{\vx_t-\argmin_{\vx \in \sR^d} A_{\theta}^{t}(\vx)}^2 \right] \\
        &\leq& \frac{2+\mu\theta}{2\mu} \left[ \norm{\nabla A_{\theta}^t(\vx_{t+1})}^2-\frac{\mu}{20\theta}\norm{\vx_t-\argmin_{\vx \in \sR^d} A_{\theta}^{t}(\vx)}^2\right].
    \end{eqnarray*}
    Therefore, Eq.~\eqref{eq:lemma-2} is proved.
\end{proof}	

\subsection{Proof of Lemma \ref{lemma:one-step-loop}}\label{app:one-step-loop}
\begin{proof}
    Since $\theta = 1/(4\sqrt{n}\delta)$ satisfies the condition required in Lemma \ref{lemma:one-step}, we get
    \[ \E_t [f(\vx_{t+1})-f(\vx)] \stackrel{\eqref{eq:lemma-2}}{\leq} \E_t D_{h}(\vx,\vx_t)-\left(1+\frac{2\mu\theta}{5}\right) D_{h}(\vx,\vx_{t+1}) + \frac{2}{9n} D_h(\vw_0, \vx_t). \]
    Taking $t = T-1$ with $T\sim \geo(p)$ and noting that $\vw^+ = \vx_{T}, \vw_0=\vx_0$, by Lemma \ref{lemma:fact}, we get 
    \begin{eqnarray}
        &&\E [f(\vw^+)-f(\vx)] = \E [f(\vx_{T})-f(\vx)] \nonumber \\
        &\leq& \E D_{h}(\vx,\vx_{T-1})- D_{h}(\vx,\vx_{T}) - \frac{2\mu\theta}{5} D_{h}(\vx,\vx_{T}) + \frac{2}{9n} D_h(\vw_0, \vx_{T-1}) \nonumber \\
        &\stackrel{\eqref{eq:fact-geom}}{=}& \E~ p D_{h}(\vx, \vx_0)- pD_{h}(\vx, \vx_{T})- \frac{2\mu\theta}{5} D_{h}(\vx, \vx_{T}) \nonumber \\
        &&+ \frac{2}{9n} \left[(1-p)D_h(\vw_0, \vx_{T}) + p D_h(\vw_0, \vx_0) \right] \nonumber \\
        &\leq& \E~ p D_{h}(\vx, \vw_0)- pD_{h}(\vx,\vw^+) - \frac{2\mu\theta}{5} D_{h}(\vx, \vw^+) + \frac{2}{9n}D_h(\vw_0, \vw^+) \label{eq:svrs-one-b} \\
        &\stackrel{\eqref{eq:div-eq}}{=}& \E~ p\dotprod{\vx-\vw_0, \nabla h(\vw^+)-\nabla h(\vw_0)}-\frac{9pn-2}{9n} D_h(\vw_0, \vw^+) - \frac{2\mu\theta}{5} D_{h}(\vx, \vw^+). \nonumber 
    \end{eqnarray}
    Thus, Eq.~\eqref{eq:svrs-one} is proved.
\end{proof}

\subsection{Proof of Theorem \ref{thm:svrs-rate}}\label{app:svrs-rate}
\begin{proof}
    Choosing $\vx=\vx_*$ and $\vx=\vw_k$ in Eq.~\eqref{eq:svrs-one-b}, which are all independent to indices $i_1, i_2\dots,i_T$ in \svrs$(f,\vw_k,\theta,p)$, then we get
    \[ \E_k [f(\vw_{k+1})-f(\vx_*)] \leq \E_k p D_{h}(\vx_*, \vw_k)- pD_{h}(\vx_*, \vw_{k+1}) - \frac{2\mu\theta}{5} D_{h}(\vx_*, \vw_{k+1}) + \frac{2}{9n}D_h(\vw_k, \vw_{k+1}). \]
    \[ \E_k [f(\vw_{k+1})-f(\vw_k)] \leq \E_k p D_{h}(\vw_k, \vw_k)- pD_{h}(\vw_k, \vw_{k+1}) - \frac{2\mu\theta}{5} D_{h}(\vw_k, \vw_{k+1}) + \frac{2}{9n}D_h(\vw_k, \vw_{k+1}). \]
    Adding both inequalities together, we could obtain 
    \begin{align*}
        \E \left[2f(\vw_{k+1})-f(\vw_k)-f(\vx_*)\right] \leq & \E p D_{h}(\vx_*, \vw_k)-\left(p+\frac{2\mu\theta}{5}\right)D_{h}(\vx_*,\vw_{k+1}) \\
        &- \left(p+\frac{2\mu\theta}{5}-\frac{4}{9n}\right) D_{h}(\vw_k,\vw_{k+1}).
    \end{align*}
    Noting that $p=1/n$, thus $p+\frac{2\mu\theta}{5}-\frac{4}{9n}>0$. Based on Eq.~\eqref{eq:div-con}, after rearranging the terms, we get
    \[\E [f(\vw_{k+1})-f(\vx_*)] + \frac{1}{2} \left(p+\frac{2\mu\theta}{5}\right)D_{h}(\vx_*,\vw_{k+1}) \leq \E~ \frac{1}{2}[f(\vw_k)-f(\vx_*)]+ \frac{p}{2}D_{h}(\vx_*, \vw_k). \]
    Now we denote the potential function as
    \begin{equation}\label{eq:poten-1}
        \Phi_k = \E [f(\vw_{k})-f(\vx_*)] + \frac{1}{2} \left(p+\frac{2\mu\theta}{5}\right)D_{h}(\vx_*,\vw_{k}).
    \end{equation}
    By $\theta = 1/(4\sqrt{n}\delta)$, we obtain
    \[ \E\Phi_{k+1} \leq \max\left\{1-\frac{1}{2}, \left(1+\frac{2\mu\theta}{5p}\right)^{-1} \right\} \E \Phi_k = \max\left\{1-\frac{1}{2}, \left(1+\frac{2\mu\sqrt{n}}{5\delta}\right)^{-1} \right\} \E \Phi_k. \]
    When $\frac{2\mu\sqrt{n}}{5\delta} \geq 1$, we get $\E\Phi_{k+1} \leq \frac{1}{2} \E \Phi_k$.
    Otherwise, $\frac{2\mu\sqrt{n}}{5\delta} < 1$, by inequality $\frac{1}{1+x}\leq 1-\frac{x}{2}, \forall 0 \leq x \leq 1$, we get $\left(1+\frac{2\mu\sqrt{n}}{5\delta}\right)^{-1} \leq 1-\frac{\mu\sqrt{n}}{5\delta}$. Hence,
    $\E\Phi_{k+1} \leq \left(1-\frac{\mu\sqrt{n}}{5\delta}\right)\E\Phi_{k}$.
    Therefore, we obtain 
    $ \E \Phi_{k+1} \leq \max\left\{1-\frac{1}{2}, 1-\frac{\mu\sqrt{n}}{5\delta} \right\}\E\Phi_k$. 
    Moreover, the initial term
    \begin{eqnarray*}
        \Phi_0 &\stackrel{\eqref{eq:poten-1}}{=}& f(\vw_0)-f(\vx_*) + \frac{1}{2} \left(p+\frac{2\mu\theta}{5}\right)D_{h}(\vx_*,\vw_0) \\
        &\stackrel{\eqref{eq:div-con}}{\leq}& f(\vw_0) - f(\vx_*) + \frac{1}{2} \left(p+\frac{2\mu\theta}{5}\right)\frac{1+\sqrt{n}\theta\delta}{2\theta}\norm{\vx_*-\vw_0}^2 \\
        &\stackrel{\rom{1}}{\leq}& f(\vw_0)-f(\vx_*)+\frac{1}{2}\left(\frac{5\delta}{2\sqrt{n}}+\frac{\mu}{4}\right) \norm{\vw_0-\vx_*}^2 \stackrel{\rom{2}}{\leq} \left[1+\frac{1}{2}\left(\frac{5\delta}{\mu\sqrt{n}}+\frac{1}{2}\right) \right][f(\vw_0) - f(\vx_*)] \\
        &\leq& 3\left(1+\frac{\delta}{\mu\sqrt{n}}\right)[f(\vw_0) - f(\vx_*)],
    \end{eqnarray*}
    where $\rom{1}$ uses $\theta = \sqrt{p}/(4\delta)$ and $\rom{2}$ uses $\frac{\mu}{2}\norm{\vw_0-\vx_*}^2\stackrel{\eqref{eq:sc}}{\leq} f(\vw_0) - f(\vx_*)$.
    Then we finally get
    \[ \E f(\vw_{k})-f(\vx_*) \stackrel{\eqref{eq:poten-1}}{=} \E \Phi_{k} \leq \left(\max\left\{1-\frac{1}{2}, 1-\frac{\mu\sqrt{n}}{5\delta} \right\}\right)^k \cdot 3\left(1+\frac{\delta}{\mu\sqrt{n}}\right)[f(\vw_0)-f(\vx_*)]. \]
    In order to make $\E \Phi_{k} \leq \varepsilon$, we need 
    \[ \exp\left\{- \frac{k}{\max \left\{2, \frac{5\delta}{\mu \sqrt{n}} \right\}} \right\} \cdot 3\left(1+\frac{\delta}{\mu\sqrt{n}}\right)[f(\vw_0)-f(\vx_*)] \le \varepsilon, \] 
    which leads to $k \geq K_1 := \max\left\{2, \frac{5 \delta}{\mu\sqrt{n}}\right\}\log\frac{3\left(1+\frac{\delta}{\mu\sqrt{n}}\right)[f(\vw_0)-f(\vx_*)]}{\varepsilon}$.
	
    Noting that one-epoch communication complexity in \svrs \ is $\Theta(n)$ in expectation when $p=1/n$ (shown in Section \ref{sec:com-com}), we get total communication complexity is $\tilde{\gO}(n+\sqrt{n} \delta/\mu)$.
\end{proof}

\subsection{Proof of Lemma \ref{lemma:one-loop}}\label{app:one-loop}
\begin{proof}
    Based on Lemma \ref{lemma:one-step-loop} and noting that $\vy_{k+1} = $ \svrs$(f,\vx_{k+1}, \theta, p)$, we get
    \begin{equation}\label{eq:svrs-one-1}
    \begin{aligned}
        \E_k [f(\vy_{k+1})-f(\vx)] \stackrel{\eqref{eq:svrs-one}}{\leq} & \E_k p\dotprod{\vx - \vx_{k+1}, \nabla h(\vy_{k+1})-\nabla h(\vx_{k+1})} \\
        &-\frac{7p}{9} D_h(\vx_{k+1}, \vy_{k+1})-\frac{2\mu\theta}{5} D_{h}(\vx, \vy_{k+1}).
    \end{aligned}
    \end{equation}
    Here $\vx \in \sR^d$ should be independent to random indices $i_1^{(k)},i_2^{(k)},\dots,i_T^{(k)}$ in \svrs$(f,\vx_{k+1},\theta,p)$.
    Then we can apply interpolation $\vz_k$ to derive
    \begin{eqnarray}
        \E_k [f(\vy_{k+1})-f(\vx)] &\stackrel{\eqref{eq:svrs-one-1}}{\leq}& \E_k ~ p \dotprod{\vz_k-\vx_{k+1}, \nabla h(\vy_{k+1})-\nabla h(\vx_{k+1})} - \frac{7p}{9} D_h(\vx_{k+1}, \vy_{k+1}) \nonumber \\
        &&+p \dotprod{\vx-\vz_k, \nabla h(\vy_{k+1}) - \nabla h(\vx_{k+1})} - \frac{2\mu\theta}{5} D_{h}(\vx, \vy_{k+1}) \nonumber \\
        &=& \E_k ~\frac{1-\tau}{\tau} \cdot p \dotprod{\vx_{k+1}-\vy_k, \nabla h(\vy_{k+1})-\nabla h(\vx_{k+1})} - \frac{7p}{9} D_h(\vx_{k+1}, \vy_{k+1}) \nonumber \\
        &&+p \dotprod{\vx-\vz_k, \nabla h(\vy_{k+1})-\nabla h(\vx_{k+1})} - \frac{2\mu\theta}{5} D_{h}(\vx, \vy_{k+1}) \nonumber \\
        &\stackrel{\rom{1}}{\leq}& \E_k~\frac{1-\tau}{\tau} \left[f(\vy_k)-f(\vy_{k+1}) - \frac{7p}{9} D_h(\vx_{k+1}, \vy_{k+1})\right] - \frac{7p}{9} D_h(\vx_{k+1}, \vy_{k+1}) \nonumber \\
        &&+p \dotprod{\vz_k - \vx, \nabla h(\vx_{k+1})-\nabla h(\vy_{k+1})} - \frac{2\mu\theta}{5} D_{h}(\vx, \vy_{k+1}) \label{eq:tmp-666}
    \end{eqnarray}
     where $\rom{1}$ uses Eq.~\eqref{eq:svrs-one-1} with $\vx=\vy_k$, which is independent to indices in \svrs$(f,\vx_{k+1},\theta,p)$. We continue obtaining
    \begin{eqnarray*}
        \E_k [f(\vy_{k+1})-f(\vx)] &\stackrel{\eqref{eq:tmp-666}\eqref{eq:div-con}}{\leq}& \E_k ~ \frac{1-\tau}{\tau} \left[f(\vy_k)-f(\vy_{k+1})\right] - \frac{7p}{9\tau} D_h(\vx_{k+1}, \vy_{k+1}) \\
        &&+p \dotprod{\vz_k - \vx, \nabla h(\vx_{k+1})-\nabla h(\vy_{k+1})} - \frac{\mu(1-\sqrt{n}\theta\delta)}{5} \norm{x-y_{k+1}}^2 \\
        &\stackrel{\rom{1}}{\leq}& \E_k ~ \frac{1-\tau}{\tau} \left[f(\vy_k)-f(\vy_{k+1})\right] - \frac{7p}{9\tau} D_h(\vx_{k+1}, \vy_{k+1}) \\
        &&+\E_k \left[\E_{j_k}\dotprod{\vz_k - \vx, \bm{\gG}_{k+1}} - \frac{3\mu}{20} \norm{\vx - \vy_{k+1}}^2 \right] \\
        &\stackrel{\eqref{eq:aux}}{\leq}& \E_k ~ \frac{1-\tau}{\tau} \left[f(\vy_k)-f(\vy_{k+1})\right] - \frac{7p}{9\tau} D_h(\vx_{k+1}, \vy_{k+1}) \\
        &&+\E_k \left[\E_{j_k}\frac{\alpha}{2}\norm{\bm{\gG}_{k+1}}^2 +\frac{\norm{\vx-\vz_k}^2}{2\alpha}-\frac{1+0.3\mu\alpha}{2\alpha}\norm{\vx-\vz_{k+1}}^2\right],
    \end{eqnarray*}
    where $\rom{1}$ uses $\E_{j_k}\bm{\gG}_{k+1} = p \left[\nabla h(\vx_{k+1})-\nabla h(\vy_{k+1})\right]$ and $\sqrt{n}\theta\delta = 1/4$.
	
    Furthermore, we can estimate 
    \begin{eqnarray*}
        \E_{j_k}\norm{\bm{\gG}_{k+1}}^2 &=& p^2 \E_{j_k} \norm{\nabla h(\vx_{k+1})-\nabla h(\vy_{k+1})+\nabla [f-f_{j_k}](\vx_{k+1})-\nabla [f-f_{j_k}](\vy_{k+1})}^2 \\
        &\stackrel{\rom{1}}{=}& p^2 \E_{j_k} \norm{\nabla h(\vx_{k+1}) - \nabla h(\vy_{k+1})}^2+\norm{\nabla [f-f_{j_k}](\vx_{k+1})-\nabla [f-f_{j_k}](\vy_{k+1})}^2 \\
        &\stackrel{\eqref{eq:avess}}{\leq}& p^2 \E_{j_k} \norm{\nabla h(\vx_{k+1}) - \nabla h(\vy_{k+1})}^2 + p^2\delta^2 \norm{\vx_{k+1}-\vy_{k+1}}^2 \\
        &\stackrel{\rom{2}}{\leq}& \frac{2(1+\sqrt{n}\theta\delta)p^2}{\theta}D_h(\vx_{k+1}, \vy_{k+1}) + \frac{2\theta p^2\delta^2}{1-\sqrt{n}\theta\delta} D_h(\vx_{k+1}, \vy_{k+1}) \\
        &= & \frac{5p^2}{2\theta}D_h(\vx_{k+1}, \vy_{k+1}) + \frac{p^2}{6n\theta} D_h(\vx_{k+1}, \vy_{k+1}) \leq \frac{8p^2}{3\theta}D_h(\vx_{k+1}, \vy_{k+1})
    \end{eqnarray*}
    where $\rom{1}$ uses $\E_{j_k}\nabla [f-f_{j_k}](\vx_{k+1})-\nabla [f-f_{j_k}](\vy_{k+1})=\bm{0}$, $\rom{2}$
    uses the convexity and smoothness of $h$ (e.g., \cite[Theorem A.1 \rom{3}]{d2021acceleration}) and Eq.~\eqref{eq:div-con}.
    After rearrangement, we get
    \begin{align*}
        \E_k~ \frac{\alpha}{\tau} [f(\vy_{k+1})-f(\vx)] \leq& \E_k~(1-\tau) \cdot \frac{\alpha}{\tau} [f(\vy_{k})-f(\vx)]+ \frac{\norm{\vx-\vz_k}^2}{2}-\frac{1+0.3\mu\alpha}{2}\norm{\vx-\vz_{k+1}}^2 \\
        & + \alpha\left(\frac{4\alpha p^2}{3\theta}-\frac{7p}{9\tau}\right)D_{h}(\vx_{k+1}, \vy_{k+1}).
    \end{align*}
    Hence, we see that once $2\tau \alpha p\leq \theta$, Eq.~\eqref{eq:acc-1} holds.
\end{proof}

\subsection{Proof of Theorem \ref{thm:accsvrs-rate}}\label{app:accsvrs-rate}
\begin{proof}
    Taking $\vx=\vx_*$ in Eq.~\eqref{eq:acc-1}, which is independent of any index during the process, we get
    \[ \E~ \frac{\alpha}{\tau} [f(\vy_{k+1})-f(\vx_*)] +\frac{(1+0.3\mu\alpha)\norm{\vx_*-\vz_{k+1}}^2}{2} \leq \E~(1-\tau)\cdot \frac{\alpha}{\tau} [f(\vy_{k})-f(\vx_*)] + \frac{\norm{\vx_*-\vz_k}^2}{2}. \]
    Denote the potential function as
    \[ \Phi_k = [f(\vy_{k})-f(\vx_*)] + \frac{\tau(1+0.3\mu\alpha)}{2\alpha}\norm{\vx_*-\vz_k}^2. \]
    We obtain
    \[ \E\Phi_{k+1} \leq \max\left\{1-\tau, \left(1+\frac{\mu\sqrt{n}}{\delta}\cdot \frac{3}{80\tau} \right)^{-1}\right\}\E\Phi_k. \]
    When $\tau= \frac{1}{4} \leq \frac{1}{8}n^{1/4}\sqrt{\frac{\mu}{\delta}}$, then we have that
    \[ \left(1-\tau\right) \left(1+\frac{\mu\sqrt{n}}{\delta}\cdot \frac{3}{80\tau} \right) \geq \left(1-\tau\right) \left(1+ \frac{3}{20\tau} \right) \geq 1 \Rightarrow \E\Phi_{k+1} \leq \left(1-\frac{1}{4}\right)\E\Phi_k. \]
    When $\tau=\frac{n^{1/4}}{8}\sqrt{\frac{\mu}{\delta}} \leq \frac{1}{4}$, we get
    \[ t:= \frac{\mu\sqrt{n}}{\delta}\cdot \frac{3}{80\tau} = \frac{3n^{1/4}}{10} \sqrt{\frac{\mu}{\delta}} \leq \frac{3}{5} \Rightarrow \frac{1}{1+t} \leq 1-\frac{5t}{8} \Rightarrow \E\Phi_{k+1} \leq \left(1-\frac{n^{1/4}}{8} \sqrt{\frac{\mu}{\delta}} \right)\E\Phi_k. \]
    Therefore, we finally obtain
    \[ \E\Phi_{k+1} \leq \max\left\{1-\frac{1}{4}, 1-\frac{n^{1/4}}{8} \sqrt{\frac{\mu}{\delta}} \right\}\E\Phi_k. \]
    By the strong convexity of $f$ in Assumption \ref{ass:1} and the choice of $\tau$ and $\alpha$, the initial term 
    \begin{eqnarray*}
        \Phi_0 &=& [f(\vy_0)-f(\vx_*)] + \frac{\tau(1+0.3\mu\alpha)}{2\alpha}\norm{\vx_*-\vy_0}^2 \\
        &=& [f(\vy_0)-f(\vx_*)] + \left(\frac{8\delta\tau^2}{\sqrt{n}\mu}+0.3\tau\right)\frac{\mu}{2}\norm{\vx_*-\vy_0}^2 \\
        &\leq& \left(1+\frac{1}{8}+\frac{0.3}{4}\right)[f(\vy_0)-f(\vx_*)] \leq 2[f(\vy_0)-f(\vx_*)].
    \end{eqnarray*}
    To obtain $\varepsilon$-error solution, we need 
    \[ k \geq K_2 = \max\left\{4,8 n^{-1/4} \sqrt{\frac{\delta}{\mu}} \right\}\log\frac{2[f(\vy_0)-f(\vx_*)]}{\varepsilon}. \]
    Note that every call of Algorithm \svrs \ requires $4n$ communication in expectation (shown in Section \ref{sec:com-com}).
    The remaining communication in one iteration of AccSVRS need $4$ communication (the master sends $\vx_{k+1}$ and $\vy_{k+1}$ to the client $j_k$, and then receives $\nabla f_{j_k}(\vx_{k+1})$ and $\nabla f_{j_k}(\vy_{k+1})$). Thus one iteration of AccSVRS is $\Theta(n)$ in expectation, leading to the total communication complexity for $\varepsilon$-error solution is $\tilde{\gO}\left(n+n^{3/4}\sqrt{\frac{\delta}{\mu}}\right)$.
\end{proof}

\subsection{Loopless SVRS}\label{app:loopless-svrs}
In this section, we describe the loopless SVRS (Algorithm \ref{algo:ll-svrs}).
By simple facts shown in Proposition~\ref{app:ll-l}, \svrs$(f,\vw_k,\theta,p)$ can be viewed as the inter iteration until $\vw_k$ in loopless SVRS is updated. Thus, the one-step variation in Lemma \ref{lemma:one-step} still holds.
Hence, we can derive a similar convergence rate and communication complexity for loopless SVRS.

\begin{algorithm}[t]
    \caption{Loopless Stochastic Variance-Reduced Sliding (SVRS)}
    \begin{algorithmic}[1]\label{algo:ll-svrs}
        \STATE \textbf{Input:} $\vw_0 \in\sR^d, p \in (0, 1), \theta>0, K \in \{1,2,\dots\}$
        \STATE Initialize $\vx_0 = \vw_0$ and compute $\nabla f(\vw_0)$
        \FOR{$ k = 0,1,2,\dots, K-1$}
        \STATE Sample $i_k \sim \text{Unif}([n])$ and compute $\vg_k = \nabla f_{i_k}(\vw_k)-\nabla f(\vw_k)$
        \STATE Approximately solve the local proximal point problem:
        \[ \vx_{k+1} \approx \argmin_{\vx \in \sR^d} A_{\theta}^k(\vx):=\dotprod{\nabla f_{i_k}(\vx_k)-\vg_k-\nabla f_{1}(\vx_k), \vx-\vx_k} + \frac{1}{2\theta}\norm{\vx-\vx_k}^2+f_1(\vx) \]
        \STATE $\vw_{k+1} = \begin{cases}
        \vx_{k+1} & \textrm{with probability} \quad p \\
        \vw_k & \textrm{with probability} \quad 1-p 
        \end{cases}$
        \ENDFOR
        \STATE \textbf{Output:} $\vw_K$
    \end{algorithmic}
\end{algorithm}

\begin{thm}\label{thm:ll-svrs-rate}
    Suppose Assumption \ref{ass:1} holds. If in loopless SVRS (Algorithm \ref{algo:ll-svrs}), the hyperparameters are set as $\theta = 1/(4 \sqrt{n} \delta), p=1/n$, and the approximate solution in each proximal step satisfies Eq.~\eqref{eq:cond-app1}, then for any error $\varepsilon>0$, when 
    \[ k \geq K_1 := \max\left\{2n, \frac{11\sqrt{n}\delta}{\mu}\right\}\log\frac{3\left(1+\frac{\delta}{\mu\sqrt{n}}\right)[f(\vx_0)-f(\vx_*)]}{\varepsilon}, \]
    i.e., after $\tilde{\gO}(n+\sqrt{n} \delta / \mu)$ communications, we can guarantee that $\E f(\vw_k)-f(\vx_*) \leq \varepsilon$.
\end{thm}
\begin{proof}
    Noting that in each step of loopless SVRS, the anchor point is $\vw_k$ instead of $\vw_0$, thus Eq.~\eqref{eq:lemma-2} holds after replacing $\vw_0$ to $\vw_k$.
    Now choosing $\vx=\vx_*$ and $\vx=\vw_k$ in Eq.~\eqref{eq:lemma-2}, which are all independent to index $i_k$, we get
    \[ \E_k [f(\vx_{k+1})-f(\vx_*)] \leq \E_k D_{h}(\vx_*, \vx_k)-\left(1+\frac{\mu\theta/2}{1+\sqrt{n}\theta\delta}\right) D_{h}(\vx_*, \vx_{k+1}) + \frac{2\theta^2\delta^2}{(1-\sqrt{n}\theta\delta)^2} D_h(\vw_k, \vx_k). \]
    \[ \E_k [f(\vx_{k+1})-f(\vw_k)] \leq \E_k D_{h}(\vw_k,\vx_k)- \left(1+\frac{\mu\theta/2}{1+\sqrt{n}\theta\delta}\right) D_{h}(\vw_k,\vx_{k+1}) + \frac{2\theta^2\delta^2}{(1-\sqrt{n}\theta\delta)^2} D_h(\vw_k, \vx_k). \]
    Adding both inequalities together and noting that 
    \[ \E_k D_h(\vw_{k+1}, \vx_{k+1}) = \E_k (1-p) D_h(\vw_{k}, \vx_{k+1}) + p D_h(\vx_{k+1}, \vx_{k+1}) = (1-p) \E_k D_h(\vw_{k}, \vx_{k+1}), \]
    as well as 
    \[ \E_k f(\vw_{k+1}) = \E_k (1-p)f(\vw_k)+pf(\vx_{k+1}), \]
    we could obtain 
    \begin{align*}
        &\E~\frac{2}{p}\left[f(\vw_{k+1})-(1-p)f(\vw_k)\right] - f(\vw_k)-f(\vx_*) \\
        \leq & \E D_{h}(\vx_*, \vx_k)-\left(1+\frac{\mu\theta/2}{1+\sqrt{n}\theta\delta}\right)D_{h}(\vx_*,\vx_{k+1}) + \left(1+\frac{4\theta^2\delta^2}{(1-\sqrt{n}\theta\delta)^2}\right) D_{h}(\vw_k, \vx_k)-D_{h}(\vw_k,\vx_{k+1}) \\
        = &  \E D_{h}(\vx_*, \vx_k)-\left(1+\frac{\mu\theta/2}{1+\sqrt{n}\theta\delta}\right) D_{h}(\vx_*,\vx_{k+1})+ \left(1+\frac{4\theta^2\delta^2}{(1-\sqrt{n}\theta\delta)^2}\right) D_{h}(\vw_k, \vx_k)-\frac{D_{h}(\vw_{k+1},\vx_{k+1})}{1-p}.
    \end{align*}
    Rearranging the terms, we get
    \begin{eqnarray*}
        &&\E [f(\vw_{k+1})-f(\vx_*)] + \frac{p}{2} \left(1+\frac{\mu\theta/2}{1+\sqrt{n}\theta\delta}\right)D_{h}(\vx_*,\vx_{k+1}) + \frac{p}{2(1-p)}D_{h}(\vw_{k+1}, \vx_{k+1}) \\
        &\leq& \E (1-\frac{p}{2})[f(\vw_k)-f(\vx_*)]+ \frac{p}{2}D_{h}(\vx_*, \vx_k) + \frac{p}{2}\left(1+\frac{4\theta^2\delta^2}{(1-\sqrt{n}\theta\delta)^2}\right)D_{h}(\vw_k, \vx_k).
    \end{eqnarray*}
    Now we denote the potential function as
    \[ \Phi_k = \E [f(\vw_{k})-f(\vx_*)] + \frac{p}{2} \left(1+\frac{\mu\theta/2}{1+\sqrt{n}\theta\delta}\right)D_{h}(\vx_*,\vx_{k}) + \frac{p}{2(1-p)}D_{h}(\vw_{k}, \vx_{k}). \]
    Then we obtain
    \[ \E\Phi_{k+1} \leq \max\left\{1-\frac{p}{2}, \left(1+\frac{\mu\theta/2}{1+\sqrt{n}\theta\delta}\right)^{-1}, \left(1+\frac{4\theta^2\delta^2}{(1-\sqrt{n}\theta\delta)^2}\right)(1-p) \right\}\E\Phi_k. \]
    Since we choose $\theta = 1/(4\sqrt{n}\delta)$, we get $\theta\mu \leq \sqrt{n}\theta\delta \leq 1/4$ by Assumption \ref{ass:1}, which shows that
    \[ \left(1+\frac{\mu\theta/2}{1+\sqrt{n}\theta\delta}\right)^{-1} = 1-\frac{\mu\theta/2}{1+\sqrt{n}\theta\delta+\mu\theta/2} = 1-\frac{4\mu\theta}{11} = 1-\frac{\mu}{11\delta\sqrt{n}}. \]
    Additionally, by $p=1/n$ and $\theta = 1/(4\delta\sqrt{n})$, we also have that
    \[ \left(1+\frac{4\theta^2\delta^2}{(1-\sqrt{n}\theta\delta)^2}\right)(1-p) = \left(1+\left(\frac{\frac{1}{2\sqrt{n}}}{1-\frac{1}{4}}\right)^2\right)(1-p) = \left(1+\frac{4p}{9}\right)\left(1-p\right)\leq 1-\frac{5p}{9}. \]
    Therefore, we obtain the ratio between $\E \Phi_{k+1}$ and $\E \Phi_k$:
    \[ \E \Phi_{k+1} \leq \max\left\{1-\frac{p}{2}, 1-\frac{\mu}{11\delta\sqrt{n}}, 1-\frac{5p}{9} \right\}\E\Phi_k \leq \max\left\{1-\frac{p}{2}, 1-\frac{\mu}{11\delta\sqrt{n}} \right\} \E \Phi_k. \]
    Moreover, the initial term
    \begin{align*}
        \Phi_0 &= f(\vw_0)-f(\vx_*) + \frac{p}{2} \left(1+\frac{\mu\theta/2}{1+\sqrt{n}\theta\delta}\right)D_{h}(\vx_*,\vx_0) \\
        &\stackrel{\eqref{eq:div-con}}{\leq} f(\vx_0) - f(\vx_*) + \frac{p}{2} \left(1+\frac{\mu\theta/2}{1+\sqrt{n}\theta\delta}\right)\frac{1+\sqrt{n}\theta\delta}{2\theta}\norm{\vx_*-\vx_0}^2 \\
        &\stackrel{\rom{1}}{\leq} f(\vx_0)-f(\vx_*)+\frac{p}{2}\left(\frac{2.5\delta}{\sqrt{p}}+\frac{\mu}{4}\right) \norm{\vx_0-\vx_*}^2 \stackrel{\rom{2}}{\leq} \left[1+\frac{p}{2}\left(\frac{5\delta}{\mu\sqrt{p}}+\frac{1}{2}\right) \right][f(\vx_0) - f(\vx_*)] \\
        &\leq 3\left(1+\frac{\delta}{\mu\sqrt{n}}\right)[f(\vx_0) - f(\vx_*)],
    \end{align*}
    where $\rom{1}$ uses $\theta = \sqrt{p}/(4\delta)$ and $\rom{2}$ uses $f(\vx_0) - f(\vx_*) \stackrel{\eqref{eq:sc}}{\geq} \frac{\mu}{2}\norm{\vx_0-\vx_*}^2$.
    Then we finally get
    \[ \E f(\vw_{k})-f(\vx_*) \leq \E \Phi_{k} \leq \left(\max\left\{1-\frac{1}{2n}, 1-\frac{\mu}{11\delta\sqrt{n}} \right\}\right)^k \cdot 3\left(1+\frac{\delta}{\mu\sqrt{n}}\right)[f(\vx_0)-f(\vx_*)]. \]
    In order to make $\E \Phi_{k} \leq \varepsilon$, we need 
    \[ \exp\left\{-\frac{k}{\max \left\{2n, \frac{11\sqrt{n} \delta}{\mu} \right\}} \right\} \cdot 3\left(1+\frac{\delta}{\mu\sqrt{n}}\right)[f(\vx_0)-f(\vx_*)] \le \varepsilon, \]
    which leads to
    \[ k \geq K_1 := \max\left\{2n, \frac{11\sqrt{n} \delta}{\mu}\right\}\log\frac{3\left(1+\frac{\delta}{\mu\sqrt{n}}\right)[f(\vx_0)-f(\vx_*)]}{\varepsilon}, \]
    Noting that communication complexity in each iteration is $2p(n-1)+2$ in expectation (by similar analysis in Section \ref{sec:com-com}), we get communication complexity is $4$ in each iteration in expectation. 
    Therefore, the total communication complexity is $\tilde{\gO}(n+\frac{\sqrt{n} \delta}{\mu})$ in expectation.
\end{proof}

\section{Computation of Gradient Complexity}\label{app:grad-com}
We show the detail omitted in Section \ref{sec:grad-com}.
Let $\vx_{t, *} = \argmin_{\vx \in\sR^d} A_{\theta}^t(\vx)$. 
Noting that by \cite[Theorem 2.2.2]{nesterov2018lectures}, we could obtain
\begin{equation*}
    \norm{\nabla A_{\theta}^t(\vx_{t, s})}^2 \leq 2L' \left(A_{\theta}^t(\vx_{t, s}) - A_{\theta}^t(\vx_{t, *})\right) \leq 
    \frac{20 \mu' L' \norm{\vx_t-\vx_{t,*}}^2}{3} \left[e^{(s+1)/\sqrt{\kappa'}}-1\right]^{-1}
\end{equation*}
if we start from $\vx_{t}$ in the proximal step for optimizing $A_{\theta}^t(\vx)$, where $\kappa' = L'/\mu', L'=L+1/\theta, \mu'=-\sqrt{n}\delta+1/\theta$ based on assumptions.
Then Eq.~\eqref{eq:cond-app1} could be satisfied after $T_{\mathrm{app}}$ iterations when 
\[ \frac{20 \mu' L' \norm{\vx_t-\vx_{t,*}}^2}{3} \left[e^{(T_{\mathrm{app}}+1)/\sqrt{\kappa'}}-1\right]^{-1} \leq \frac{\mu}{20\theta} \norm{\vx_t-\vx_{t,*}}^2. \]
Note that $\theta=1/(4\sqrt{n}\delta)$, which leads to  
\[  T_{\mathrm{app}} = \gO\left(\sqrt{\frac{1+\theta L}{1-\sqrt{n}\theta\delta}} \cdot \log \left(\frac{(1+\theta L)(1-\sqrt{n}\theta\delta)}{\mu\theta}\right) \right) = \gO\left(\left(1+n^{-1/4}\sqrt{\delta/\mu}\right)\log\frac{\sqrt{n}\delta+L}{\mu}\right). \]
Hence, the total number of gradient calls in expectation is
\[ \gO(n T_{\mathrm{app}} \cdot K_{2}) = \tilde{\gO}\left[ \left(n {+} n^{3/4}\sqrt{\frac{\delta}{\mu}}\right)\left(1 {+} \frac{1}{n^{1/4}}\sqrt{\frac{L}{\delta}}\right)\right] = \tilde{\gO} \left(n {+} n^{3/4}\left(\sqrt{\frac{\delta}{\mu}} {+} \sqrt{\frac{L}{\delta}}\right) {+} \sqrt{\frac{nL}{\mu}}\right). \]
Since $\delta \in [\mu, L]$, we obtain 
$\sqrt{\frac{\delta}{\mu}} + \sqrt{\frac{L}{\delta}} \leq \sqrt{\frac{L}{\mu}}+1$, leading to
\[ n+n^{3/4}\left(\sqrt{\frac{\delta}{\mu}} + \sqrt{\frac{L}{\delta}}\right) + \sqrt{\frac{nL}{\mu}} \leq 2\left(n+ n^{3/4} \sqrt{\frac{L}{\mu}}\right).\]
Thus, the gradient complexity is $\tilde{\gO}\left(n+n^{3/4}\sqrt{L/\mu}\right)$.
Moreover, when $\delta=\Theta(\sqrt{\mu L})$, we obtain
\[ n+n^{3/4}\left(\sqrt{\frac{\delta}{\mu}} + \sqrt{\frac{L}{\delta}}\right) =  n + \Theta\left(n^{3/4}\left(\frac{L}{\mu}\right)^{1/4}\right) + \sqrt{\frac{nL}{\mu}} = \Theta\left(n+\sqrt{\frac{nL}{\mu}}\right). \]
Thus, the gradient complexity is $\tilde{\gO}\left(n + \sqrt{nL/\mu}\right)$ in this time.

Note that Assumption \ref{ass:1} and smoothness of $f_1$ could only guarantee
\begin{align*}
    &\frac{1}{n}\sum_{i=1}^n \norm{\nabla f_i(\vx) -\nabla f_i(\vy)}^2 \stackrel{\eqref{eq:avess}}{\leq} \delta^2+\norm{\nabla f(\vx) -\nabla f(\vy)}^2 \\
    \leq& \delta^2 {+} 2\norm{\nabla [f {-} f_1](\vx) {-} \nabla [f {-} f_1](\vy)}^2 {+} 2 \norm{\nabla f_1(\vx) {-} \nabla f_1(\vy)}^2 {\leq} \left[(2n {+} 1)\delta^2 {+} 2L^2\right]\norm{\vx {-} \vy}^2,
\end{align*}
that is, $f_i$'s are $(2L+2\sqrt{n}\delta)$-average smooth. Hence, the tightness of our gradient complexity holds for the average smooth setting only when $\sqrt{n}\delta=\gO(L)$.
Moreover, we can also compute
\begin{align*}
    &\norm{\nabla f_i(\vx) -\nabla f_i(\vy)}^2 \leq 2\norm{\nabla [f-f_i](\vx) -\nabla [f-f_i](\vy)}^2+2\norm{\nabla f(\vx) -\nabla f(\vy)}^2 \\
    \stackrel{\eqref{eq:avess}}{\leq}& 2n\delta^2 {+} 4\norm{\nabla [f {-} f_1](\vx) {-} \nabla [f {-} f_1](\vy)}^2 {+} 4 \norm{\nabla f_1(\vx) {-} \nabla f_1(\vy)}^2 \leq \left[6n\delta^2 {+} 4L^2\right] \norm{\vx {-} \vy}^2,
\end{align*}
that is, $f_i$'s are $(2L+3\sqrt{n}\delta)$-smooth. Hence, the tightness of our gradient complexity holds for the component smooth setting only when $\delta=\Theta(\sqrt{\mu L})$ and $n\mu =\gO(L)$.

\input{append_lower}

\section{Experiment Details}\label{app:exp}
We show some detail of our numerical experiments in this section.
The computation of problem-dependent parameters is defined as follows.
Since the objective is 
\begin{equation*}
    f(\vx) = \frac{1}{n}\sum_{i=1}^n \left[f_i(\vx):= \frac{1}{m} \sum_{j=1}^{m}\left(\vz_{i,j}^\top\vx-y_{i,j}\right)^2+\frac{\mu}{2}\norm{\vx}^2\right],
\end{equation*}
Let $\mZ_i = \left(\vz_{i,1}, \cdots, \vz_{i,m} \right)/\sqrt{m/2}, \vy_i = (y_{i,1}, \dots, y_{i,m})^\top/\sqrt{m/2}$. We reformulate $f_i, i \in [n]$ into
\[ f_i(\vx) = \frac{1}{2}\norm{\mZ_i^\top\vx-\vy_{i}}^2+\frac{\mu}{2}\norm{\vx}^2, \nabla^2 f_i(\vx) = \mZ_i\mZ_i^\top+\mu\mI_d. \]
Thus, we obtain the smoothness of each $f_i$ is
$L_i = \norm{\mZ_i}^2 + \mu$, and $ L := \max_{i \in [n]} L_i$.
Obviously, $f$ is $\mu$-strongly convex.

For the synthetic data, we first generate a random symmetric matrix $\mZ_0 \in \sR^{d\times d}$ with $d=100$ and $\norm{\mZ_0}=3000$, then we add a perturbed symmetric matrix $\mN_i, \forall i \in [n]$ with $n=400, \norm{\mN_i} \approx 30$ to obtain $\mZ_i=\mZ_0+\mN_i$.
We also add a correction $\lambda_{\min}(\mZ_i)\mI_d$ to $\mZ_i$ to further make $\mZ_i \succeq 0$.
Finally, we recompute the center matrix $\mZ = \sum_{i=1}^n \mZ_i/n$ and $\delta$-average similarity coefficient following AveHS in Eq.~\eqref{eq:hes-ss} as 
\[ \delta = \sqrt{\frac{1}{n}\sum_{i=1}^n\norm{\mZ_i-\mZ}^2}. \]

We use the analytic solution obtained by the proximal step since
\[ \prox_{f_1}^{\theta}(\vx_0) := \argmin_{\vx \in\sR^d} f_1(\vx) + \frac{1}{2\theta} \norm{\vx - \vx_0}^2 = \left[\mZ_1\mZ_1^\top+\left(\mu+\frac{1}{\theta}\right)\mI_d\right]^{-1}\left(\mZ_1\vy +\frac{\vx_0}{\theta}\right).
\]
For Katyusha X \cite[Fact 4.2]{allen2018katyusha}, and AccSVRS (Thm \ref{thm:accsvrs-rate}), we scale the interpolation coefficient $\tau = s\tau_0$ with $s \in \{0.5, 1,2,5,10\}$ and $\tau_0$ is the theoretical value. The finally used scaling $s$ is shown in Table~\ref{table:s}.
The initial points of all methods are the same, which are sampled from $\mathrm{Unif}(\gS^{d-1})$.

We also run the real data `a9a' from LIBSVM library \cite{chang2011libsvm}, where we split it into $n=50$ datasets with the data size $m=600$. 
The results are shown in Figure \ref{fig:exp2}, and we can observe similar behavior of our methods. 

\begin{table}[t]
\centering
\caption{The choices of interpolation $\tau=s\tau_0$ in experiments.}
    \begin{tabular}{|c|c|c|c|c|c|c|c|}
        \hline
        & & \multicolumn{3}{c|}{Katyusha X} & \multicolumn{3}{c|}{AccSVRS} \\
        \hline
        \multirow{2}{*}{Synthetic data} & $\mu$ & 1 & 0.1 & 0.01 & 1 & 0.1 & 0.01 \\
        \cline{2-8}
        & s & 1 & 2 & 5 & 2 & 5 & 10 \\
        \hline 
        \multirow{2}{*}{Real data} & $\mu$ & 0.1 & 0.01 & 0.001 & 0.1 & 0.01 & 0.001 \\
        \cline{2-8}
        & s & 1 & 1 & 2 & 2 & 0.5 & 0.5 \\
        \hline
    \end{tabular}   
    \label{table:s}
\end{table}

\begin{figure}[t]
    \centering
    \begin{subfigure}[b]{0.33\textwidth}
        \includegraphics[width=\linewidth]{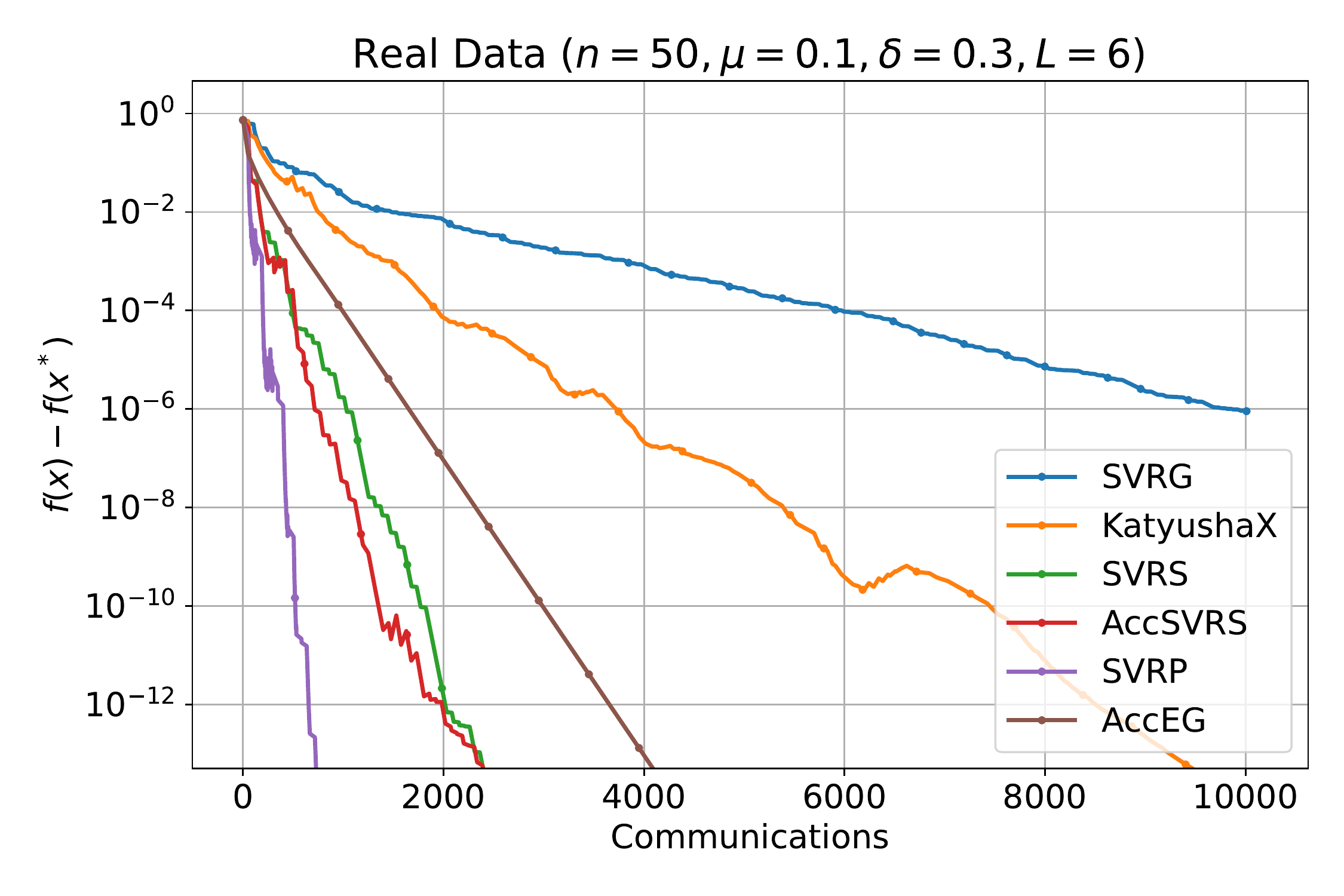}
    \end{subfigure}
    \hspace{-5pt}
    \begin{subfigure}[b]{0.33\textwidth}
        \includegraphics[width=\linewidth]{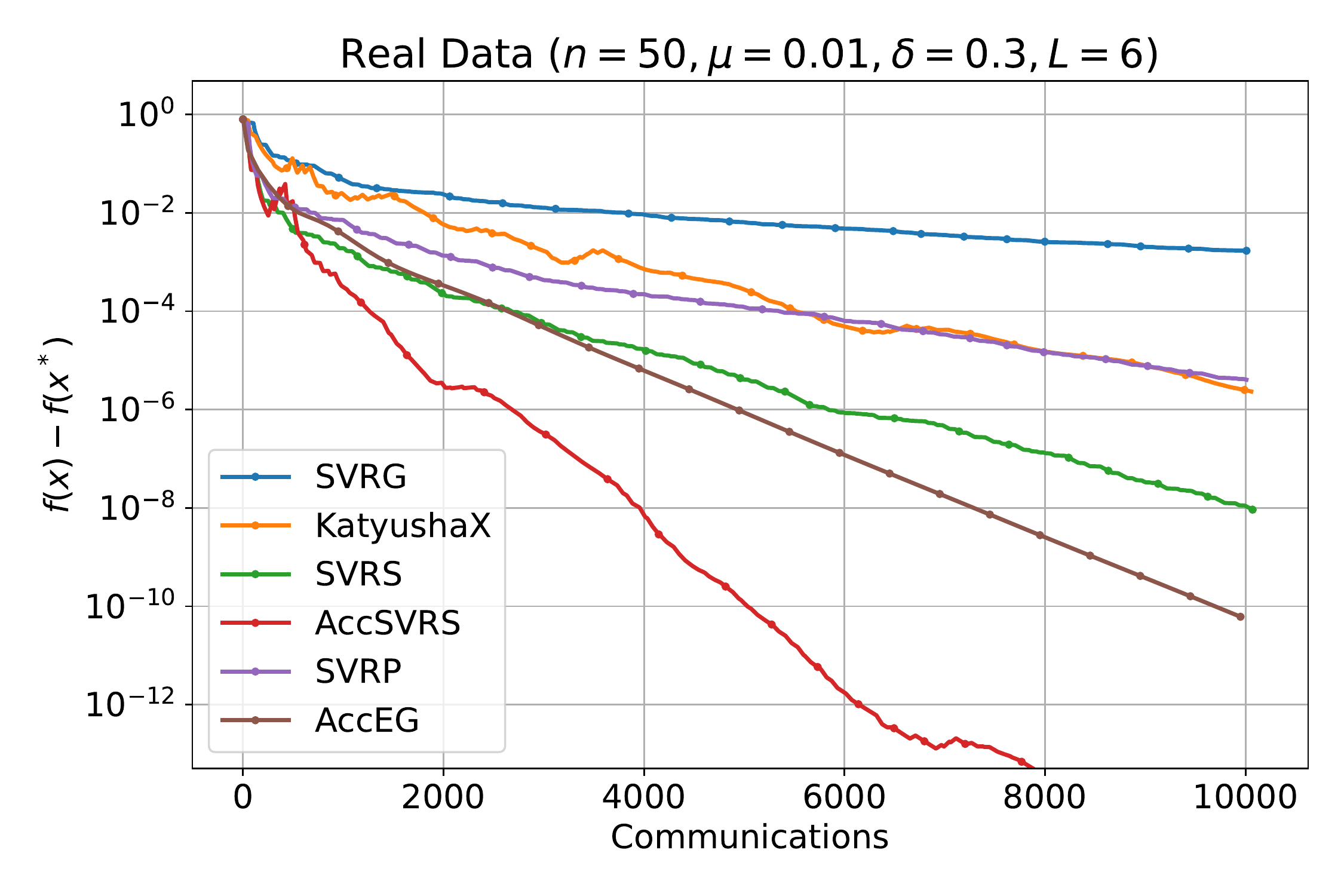}
    \end{subfigure}	
    \hspace{-5pt}
    \begin{subfigure}[b]{0.33\textwidth}
        \includegraphics[width=\linewidth]{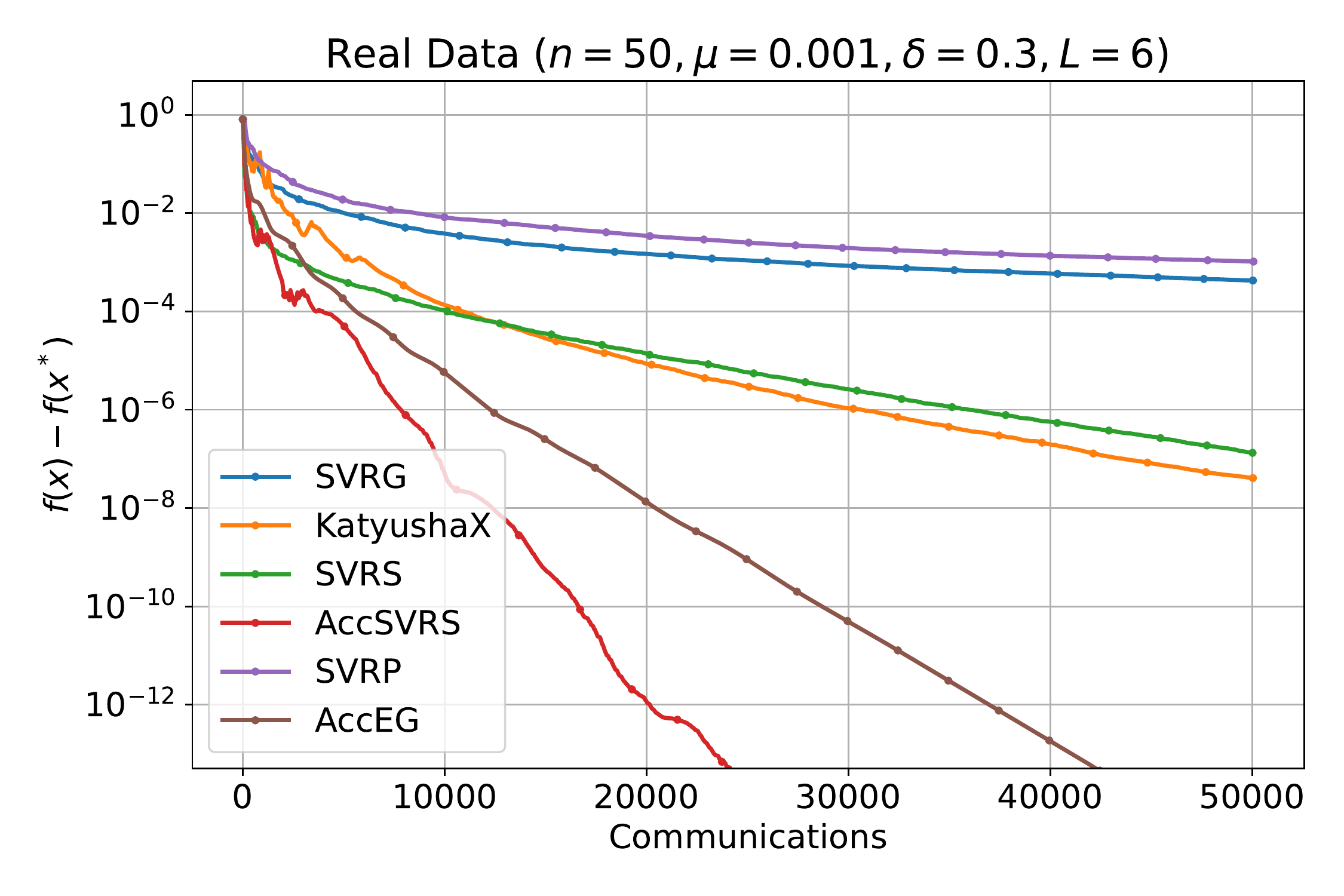}
    \end{subfigure}
    \caption{Numerical experiments on real data. The corresponding coefficients are shown in the title of each graph. We plot the function gap on a log scale versus the number of communication steps, where one exchange of vectors counts as a communication step.}
    \label{fig:exp2}
\end{figure}

%% file: append_lower.tex
\section{Omitted Details of Section~\ref{sec:lower}}
In this section, we give the omitted details of Section~\ref{sec:lower} as well as their proofs.

\subsection{
Formal Statement of Definition~\ref{defn:lower_alg_informal}
and Discussion
}\label{app:lower:defn}
In this subsection, we give the formal statement of Definition~\ref{defn:lower_alg_informal} and show that Algorithm~\ref{algo:acc-svrs} satisfies our definition.

We first introduce the two oracles: the incremental first-order oracle (IFO) \citep{agarwal2015lower, zhou2019lower} and the Proximal Incremental First-order
Oracle (PIFO)\footnote{Although we have defined PIFO in Section~\ref{sec:lower:defn}, we restate it here for completement.} \citep{woodworth2016tight, han2021lower}, which are defined as 
    $\ifo_{f_i} (\vx)
    = [ f_i(\vx), \nabla f_i(\vx) ]
    $ and $
    \pifo_{f_i} (\vx, \gamma)
    = [ f_i(\vx), \nabla f_i(\vx), \prox_{f_i}^\gamma (\vx) ]$ with $\gamma > 0$ 
respectively.
Here the proximal operator is 
\begin{align*}
    \prox_{f_i}^\gamma (\vx) := \argmin_{\vu} \left\{ f_i (\vu) + \frac{1}{2\gamma} \norm{\vx - \vu}^2 \right\}
    = \argmin_{\vu} \left\{ \gamma f_i (\vu) + \frac{1}{2} \norm{\vx - \vu}^2 \right\}.
\end{align*}
The IFO 
$\ifo_{f_i} (\vx) $
takes a point $\vx$ and a component $f_i$ as input and returns the zero-order and first-order information of the component at $\vx$.
The PIFO 
$\pifo_{f_i} (\vx, \gamma)$
has
an additional input $\gamma > 0$, which can be viewed as the step size of the proximal operator.
Besides the local zero-order and first-order information returned by 
$\ifo_{f_i} (\vx) $,
$\pifo_{f_i} (\vx, \gamma)$
also provides some global information of $f_i$ by means of the proximal operator.
To see this, if we let $\gamma \rightarrow +\infty,$ $\prox_{f_i}^\gamma (\vx)$ converges to the exact minimizer of $f_i$, irrelevant to the choice of $\vx$.
In practice, it could be hard to compute $\prox_{f_i}^\gamma (\vx)$ precisely. Nevertheless, since we only focus on communication complexity, it makes no difference to distinguish between the IFO and the PIFO\footnote{See Lemma~\ref{lem:hard_expand}}.
Thus we assume the algorithm has access to the PIFO and the definition is as follows.


\begin{defn}[Formal version of Definition~\ref{defn:lower_alg_informal}]\label{defn:lower_alg}
    Consider a randomized algorithm $\gA$ to solve problem~\eqref{eq:obj}.
    Suppose the number of communication rounds is $T$. Define information sets $\gI_{t+1}$, $\gI_{t+1}^0$ and $\gI_{t+1}^1$.
    Here $\gI_{t+1}$ denotes all the information $\gA$ obtains after round $t$,
    while $\gI_{t+1}^0$ and $\gI_{t+1}^1$ denote the information 
    before and after (possible) anchor point updating during round $t$, respectively.
    The algorithm updates the information set
    by the following procedure.
    \begin{enumerate}
    
    \item  
    Choose a distribution $\gD$ over $[n]$ with $q_i = \sP_{Z \sim \gD}(Z=i) > 0$,
    a positive number $p \le c_0 / n$\footnote{To include catalyst accelerated algorithms, we also need $p \ge c_1 / n$ for some $c_1>0$ (see footnote~\ref{foot:catalyst}). To analyze Algorithm~\ref{algo:acc-svrs}, $p \le c_0 / n$ is enough.} and the initial points $\vx_0$. 
    Specify a master note~$1$ and assume $\max_{2 \le i \le n} q_i \le q_0 / n $. Node~$1$ sends 
    $\vx_0 $ to all the other nodes and other nodes send 
    $ \pifo_{f_i}(\vx_0, \gamma_0) $ back to node~$1$.
    Initialize the information set as $ \gI_0
    := \spn \big\{ \vx_0, \nabla f_i (\vx_0), \prox_{f_i}^{\gamma_0} (\vx_0) ~\big|~
    1 \le i \le n \big\}$ and set $t=0$ 
    and $\tilde{\vx}_0 = \vx_0$.
    \label{item:lower_alg_init}
    
    \item \label{item:lower_alg_sample_one}
    Sample $i_t \sim \gD$. Node~$1$ sends 
    $\tilde{\vx}_t$ to node~$i_t$ and node~$i_t$ sends 
    $ \pifo_{f_{i_t}} (\tilde{\vx}_t, \gamma_t) $ back to node~$1$. 
    Update the information set 
    \begin{align}\label{eq:lower_alg_infor0}
        \gI_{t+1}^0 := \spn \big\{ \vy, \nabla {f_{i_t}} (\tilde{\vx}_t), \prox_{f_{i_t}}^{\gamma_t} (\tilde{\vx}_t) ~\big|~ \vy \in \gI_t \big \}
    \end{align}
    
    \item \label{item:lower_alg_span}
    Update the information set $\gI_{t+1}^1$ and
    choose $ {\vx}_{t+1} \in \gI_{t+1}^1 $ following the linear-span protocol
    \begin{align}\label{eq:lower_alg_subprob}
        \vx_{t+1} \in \gI_{t+1}^1 := \spn
        \big\{\vy, \nabla f_1(\vz), \prox_{f_1}^{\gamma'_t} (\vw) ~\big|~ \vy, \vz, \vw \in \gI_{t+1}^{0} \big\}.
    \end{align}    

    \item Sample a Bernoulli random variable $a_t$ with expectation equal to $p$. If $a_t=1$, go to step~\ref{item:lower_alg_sample_all} (update the anchor point); otherwise, set $\tilde{\vx}_{t+1} = \vx_{t+1}$, $\gI_{t+1} = \gI_{t+1}^{1}$ and go to step~\ref{item:lower_alg_final} (do not update the anchor point). \label{item:lower_alg_bernoulli}
    
    \item \label{item:lower_alg_sample_all} 
    Sample $j_t \sim \gD$. Node~$1$ sends 
    some $\vy_{t+1} \in \gI_{t+1}^1$ to node~$j_t$ and node~$j_t$ sends 
    $ \pifo_{f_{j_t}} (\vy_{t+1}, \gamma''_{t+1}) $ back to node~$1$. 
    Obtain the anchor point $\tilde{\vy}_{t+1}$ by
    \begin{align}\label{eq:lower_alg_before_fullgrad}
        \tilde{\vy}_{t+1}  \in  \spn
        \big\{ \vy, \nabla f_1(\vz), \prox_{f_1}^{\gamma'_t} (\vw),
        \nabla f_{j_t} (\vy_{t+1}), \prox_{f_{j_t}}^{\gamma''_{t+1}} (\vy_{t+1}) ~\big|~ \vy, \vz, \vw  \in  \gI_{t+1}^{1}  \big\}.
    \end{align}   
    Then node~$1$ sends the anchor point 
    $ \tilde{\vy}_{t+1}$ to all the other nodes and other nodes send 
    $ \pifo_{f_{i}}(\tilde{\vy}_{t+1}, {\gamma}_{t+1}) $ 
    back to node~$1$.
    Update the information set and obtain $\tilde{\vx}_{t+1}$ by
    \begin{align}\label{eq:lower_alg_infor2}
        \tilde{\vx}_{t+1} \in \gI_{t+1} := \spn \big\{ \vy, \tilde{\vy}_{t+1},
        \nabla f_i (\tilde{\vy}_{t+1}), \prox_{f_i}^{\gamma_{t+1}} (\tilde{\vy}_{t+1}) ~\big|~ \vy \in \gI_{t+1}^1, 1 \le i \le n \big\}.
    \end{align}

    \item \label{item:lower_alg_final}
    If $t = T-1$, output some point in $ \gI_T$;
    otherwise, set $t \leftarrow t+1$
    and go back to step~\ref{item:lower_alg_sample_one}.
    
    \end{enumerate}
Here all the random variables $i_t, j_t$ and $a_t$ with $0 \le t \le T-1$ are mutually independent, and the step sizes of the proximal operator $\gamma_t, \gamma'_t$ and $\gamma''_t$ are positive numbers. 
\end{defn}
Now we explain this definition and show that Algorithm~\ref{algo:acc-svrs} (with Algorithm~\ref{algo:l-svrs} as a part) satisfies our definition. 

\textbf{Initialization.}
In our definition, 
step~\ref{item:lower_alg_init} is the initialization step. Without loss of generality, we can assume $\vx_0 = \vzero$ and node~$1$ is the master node. Otherwise, it suffices to consider $\{ \tilde{f}_i (\vx) = f_i(\vx + \vx_0) \}_{i=1}^n$ and exchange the indices between node~$1$ and the master node.
In Algorithm~\ref{algo:acc-svrs}, the distribution $\gD$ is $\mathrm{Unif}([n])$\footnote{When analyzing computational complexity, this distribution can also depend on the smoothness of each component function \citep{xiao2014proximal,allen2018katyusha}}, and $p = 1 / n$.
In the initialization stage, the algorithm needs to calculate the full gradient of the initial point $\vx_0$, whose communication cost is $2(n-1)$.

We note that Definition~\ref{defn:lower_alg} enjoys a loopless structure while Algorithm~\ref{algo:acc-svrs} has two loops. In fact, when $p$ is fixed, a loopless algorithm is equivalent to a two-loop one with the inner loop size obeying $\geo(p)$\footnote{See Proposition~\ref{app:ll-l}.}. 

\textbf{Analysis of one communication round.}
In each communication round, whether to calculate the full gradient depends on a coin toss with success probability $p$, as shown in step~\ref{item:lower_alg_bernoulli}.

\textbf{The case $a_t=0$.}
We first focus on the case where the full gradient need not be calculated.
Such a scenario corresponds to an iteration of Algorithm~\ref{algo:l-svrs}.
Each communication round start with step~\ref{item:lower_alg_sample_one}.
In this step,
the algorithm samples a local node, with which the master node communicates. And the communication cost is $2$.
$\tilde{\vx}_t$ in this step corresponds to $\vx_t$ in Algorithm~\ref{algo:l-svrs}.
In step~\ref{item:lower_alg_span}, the master node calculates the next point based on the current information set $\gI_{t+1}^0$ as well as the PIFO $\pifo_{f_1}$.
This corresponds to line~7 in Algorithm~\ref{algo:l-svrs}.
Indeed, the subproblem~\eqref{eq:alg_prox} can be rewritten as finding
\begin{align*}
    \argmin_{\vx \in \sR^d} A_\theta^t (\vx)
    & = \argmin_{\vx \in \sR^d} \left\{ \frac{1}{2\theta} \norm{ \vx - \vx_t + \theta [\nabla f_{i_t} (\vx_t) - \vg_t - \nabla f_1 (\vx_t)] }^2 + f_1(\vx) \right\} \\
    & = \prox_{f_1}^\theta \left( \vx_t - \theta [\nabla f_{i_t} (\vx_t) - \vg_t - \nabla f_1 (\vx_t)] \right).
\end{align*}
If the algorithm has access to the PIFO $\pifo_{f_1}$, then the subproblem~\eqref{eq:alg_prox}  can be exactly solved by one step of \eqref{eq:lower_alg_subprob}.
Otherwise, one can apply \eqref{eq:lower_alg_subprob} recursively without the proximal information (i.e., only using the IFO $\ifo_{f_1}$), e.g., (accelerated) gradient methods, to find an approximate solution of \eqref{eq:alg_prox}\footnote{Such a modification makes no difference to subsequent analysis. See Remark~\ref{rema:recursive_apply}.} 

\textbf{The case $a_t=1$.}
When $a_t = 1$ in step~\ref{item:lower_alg_bernoulli}, the algorithm needs to perform step~\ref{item:lower_alg_sample_all}, which corresponds to an outer iteration of Algorithm~\ref{algo:acc-svrs}.
Before calculating the full gradient, 
the algorithm first samples a local node~$j_t$ again and the master node communicates the information about $\vy_{t+1}$ with this node.
Here $\vy_{t+1}$ corresponds to $\vy_{k+1}$ in Algorithm~\ref{algo:acc-svrs},
and the communication cost is $2$.
Then the master node calculates $\tilde{\vy}_{t+1}$ by \eqref{eq:lower_alg_before_fullgrad}, which corresponds to $\vx_{k+1}$ (of the next iteration) in Algorithm~\ref{algo:acc-svrs}.
That is to say, lines~7, 8 and 4 (of the next iteration) in Algorithm~\ref{algo:acc-svrs} can be summarized as \eqref{eq:lower_alg_before_fullgrad}.
Then the master node communicates with all the other nodes the information about $\tilde{\vx}_{t+1}$, and the communication cost is $2(n-1)$.
In \eqref{eq:lower_alg_infor2}, the algorithm picks up $\tilde{\vx}_{t+1}$ as the starting point of the next round.
In Algorithm~\ref{algo:acc-svrs}, $\tilde{\vx}_{t+1}$ is the same to $\tilde{\vy}_{t+1}$.

\textbf{Communication cost.}
From the above analysis, the  communication cost in step~\ref{item:lower_alg_sample_all} is $2(n-1)$.
Since we assume $q \le c_0/n$,
step~\ref{item:lower_alg_sample_all} is performed infrequently and the expected communication cost is (at most) $2(n-1) \cdot p \approx 2 c_0 $ for a sufficiently large $n$.
As a result,  the total communication cost of a round is roughly $2 + 2c_0$ in expectation.
After $T$ rounds, the expected communication cost is roughly $ 2 (n-1) + 2 (1+c_0) T$.
As a result, we can use the number of rounds to measure communication complexity.

\textbf{The linear-span protocol and information set.}
In Definition~\ref{defn:lower_alg}, we focus on loopless algorithms based on the linear-span protocol.
One can check that many methods, e.g., KatyushaX~\citep{allen2018katyusha}, L-SVRG and L-Katyusha~\citep{kovalev2020don}, Loopless SARAH~\citep{li2020convergence} and SVRP\footnote{\label{foot:catalyst}For Catalyzed SVRP in their paper, we can slightly modify it without affecting 
the gradient or communication
complexity. Specifically, we remove the full gradient step at the beginning of the inner loop and do not update the current point until the full gradient is calculated. The number of additional communication rounds is $1/p = \Theta(n)$ in expectation, as long as $p = \Theta(1/n)$.
Since in each inner loop,
the algorithm must calculate the full gradient, whose gradient or communication complexity is also $\Theta(n)$, such a modification would not affect the total complexity.}~\citep{khaled2022faster}, satisfy our definition. 
And this class of algorithms is sufficiently large in that the upper and lower bounds have matched for most cases \citep{han2021lower}.
Built on the linear-span protocol, the information set $\gI_{t+1}$, a linear subspace of the whole space, gathers all the gradient and proximal information obtained by $t$ rounds of communication and includes all the possible points generated by the algorithm after round $t$.
Clearly, the sequence $\{ \gI_t \}_{t=0}^T$
is nondecreasing in the sense that $\gI_t \subseteq \gI_{t'}$ for any $t' > t$.


\subsection{Details of Section~\ref{sec:lower:construct}}\label{app:lower:construct}
Recall that in Section~\ref{sec:lower:construct}, we consider the following class of matrices
\begin{align*}
    \mB(m, \hardzeta ) = 
    \begin{bmatrix}
    1 & -1 & & & \\
     & 1 & -1 & & \\
     & & \ddots & \ddots & \\
     & & & 1 & -1 \\
     & & & & \hardzeta \\
    \end{bmatrix}
    \in \sR^{m \times m}.
\end{align*}
And one can check the matrix $\mA(\harddim, \hardzeta)$ is a tridiagonal matrix, i.e.,
\begin{align*}
    \mA(m, \hardzeta) := \mB(m, \hardzeta)^\top \mB(m, \hardzeta)
    := \begin{bmatrix}
    1 & -1 & & & \\
    -1 & 2 & -1 & & \\
     & \ddots & \ddots & \ddots & \\
     & & -1 & 2 & -1 \\
     & & & -1 & \hardzeta^2 + 1 
    \end{bmatrix} 
    \in \sR^{m \times m}.
\end{align*}


With the hard instance constructed in \eqref{eq:obj_hard_before_scale}, we have the following lemma, which is a modification of Lemma~6.1 in \citet{han2021lower} in that the partitions of the index sets are slightly different. 

\begin{lemma}
\label{lem:hard_expand}
    Suppose that $n \ge 3$, $m \ge 3$, $\gamma$ is an arbitrarily positive number and $\vx \in \subspace_k$ for some $0 \le k < \harddim$.
    If $k = 0$, we have
    \begin{align*}
        \nabla \hardr_i(\vx),\, \prox_{\hardr_{i}}^{\gamma} (\vx)\in 
        \begin{cases}
        \subspace_{1}, & \text{ if } 
        i=1, \\
        \subspace_{0}, & \text{ otherwise}.
        \end{cases}
    \end{align*}
    If $k > 0$, we have
    \begin{align*}
        \nabla \hardr_i(\vx),\, \prox_{\hardr_{i}}^{\gamma} (\vx)\in 
        \begin{cases}
        \subspace_{k+1}, & \text{ if } 
        k \in \gL_i, \\
        \subspace_{k}, & \text{ otherwise}.
        \end{cases}
    \end{align*}
Here $\{ \subspace_k \}_{k=0}^\harddim$ are defined as $\subspace_0 = \{ \vzero \}$ and $\subspace_k = \spn\{ \ve_1, \ve_{2}, \dots, \ve_{k}\} $ for $1 \le k \le \harddim$, and we omit the parameters of ${r}_i$ to simplify the notation.
\end{lemma}

Lemma~\ref{lem:hard_expand} 
tells us that if the current point $\vx$ lies in some subspace of $\sR^\harddim$, only one component can provide the information of the next dimension by gradient or proximal information.
In this sense, PIFO cannot provide more information than IFO.
Thus, when we focus on the communication complexity of an algorithm, it makes no difference to distinguish between IFO and PIFO.
And we can assume the algorithm has access to PIFO without loss of generality.
Moreover, when $k > 0$, the oracle of $r_1$ can never provide any information on the next dimension.
The proof of Lemma~\ref{lem:hard_expand} is deferred to Appendix~\ref{app:lower:pf_expand}.

With Lemma~\ref{lem:hard_expand}, Lemma~\ref{lem:hard_expand_infor} is a natural corollary and the proof is deferred to Appendix~\ref{app:lower:pf_expand_infor}.

\begin{remark}\label{rema:recursive_apply}
 Recall that in steps~\ref{item:lower_alg_sample_one} and \ref{item:lower_alg_span},
 the difference between $\gI_{t+1}^1$ and $\gI_{t+1}^0$ only resides in $\pifo_{f_1}(\vx, \gamma)$ (or $\pifo_{\hardr_1}(\vx, \gamma)$ when we consider problem~\eqref{eq:obj_hard_before_scale}) for $\vx \in \gI_{t}^0$, while
Lemma~\ref{lem:hard_expand} implies that $\pifo_{\hardr_1}(\vx, \gamma)$ would not expand the information set as long as $\vx \neq \vzero$\footnote{In the proof of Lemma~\ref{lem:hard_expand_infor} in Appendix~\ref{app:lower:pf_expand_infor}, we show that $\gI_{t+1}^1 = \gI_{t+1}^0$}.
This demonstrates that applying \eqref{eq:lower_alg_subprob} recursively would not affect the analysis of communication complexity.
\end{remark}

The next result is a corollary of Lemma~\ref{lem:hard_expand_infor} and the proof is deferred to Appendix~\ref{app:lower:pf_stopping_time}.
\begin{corollary}\label{coro:lower_stopping_time}
Define the random variables $ T_0 = -1$,
\begin{align}
\label{eq:lower_stopping_time}
    T_k := \min_t \{t: t > T_{k-1}, 
    3k{-}2 \in \gL_{i_t}
    \mbox{ or } a_t = 1 \}~\text{ for } 1 \le k \le (\harddim - 1) / 3,
\end{align}
and $Y_k :=
T_k - T_{k-1}$.
Then  
we have
(i) $\gI_{t+1} \subseteq \subspace_{3k-2}$ for any $t < T_k$;
(ii) the $Y_k$ are mutually independent; 
(iii) $Y_k \sim \geo( q_{k'} + p - p q_{k'} )$ 
with $k' \equiv 3k{-}1~(\bmod~(n{-}1))$, $2 \le k' \le n$.
\end{corollary}
Corollary~\ref{coro:lower_stopping_time} claims that $T_k$ is the smallest index of the communication round after which the information set can be expanded to $\subspace_{3k+2}$.
Moreover, $T_k$ can be decomposed into the sum of independent geometric random variables.
With Lemma~\ref{lem:geo}, which gives a concentration result for the sum of geometric random variables, we have the following proposition, whose proof is deferred to Appendix~\ref{app:lower:pf_complexity}.

\begin{prop}\label{prop:complexity}
Let $0 \le M \le (\harddim - 2 )/3 $ and $N = \frac{n(M+1)}{4 
(q_0 + c_0) } + 1$ with $q_0$ and $c_0$ defined in Definition~\ref{defn:lower_alg_informal}.
Suppose we use an algorithm $\gA$ satisfying Definition~\ref{defn:lower_alg} to solve problem~\eqref{eq:obj_hard_before_scale}.
After $N$ round of communication, the algorithm obtains the information set $\gI_N$. 
Then we have $\gI_N \subseteq \subspace_{3 M + 1} \subset \subspace_{m - 1}$.
Moreoever,
if
\begin{align}
\label{eq:complexity_condition}
\min_{\vx \in 
\subspace_{3M+1}
}
\hardr (\vx) - \min_{\vx \in \sR^\harddim} \hardr(\vx) \ge 9\eps,
\end{align}
we have
\begin{align*}
    \E \min_{\vx \in \gI_N}  \hardr(\vx)
    - \min_{\vx \in \sR^\harddim} \hardr(\vx) \ge \eps.
\end{align*}
\end{prop}

Proposition~\ref{prop:complexity} specifies the number of communication rounds needed to find an $\eps$-suboptimal solution under the condition \eqref{eq:complexity_condition}.
Roughly speaking, the condition requires that the exact solution of problem \eqref{eq:obj_hard_before_scale}
does not lie in some subspace of $\subspace_{\harddim-1}$

Now we come back to the hard instance \eqref{eq:obj_hard_before_scale}.
Recall that $\hardr$ is $\hardc$-strongly convex and $\hardr_i$'s satisfy
$\sqrt{8n+4}$-aveSS.
We need to
properly scale the function class $\{ r_i \}_{i=1}^n$ such that it satisfies Assumption~\ref{ass:1}.
Note that rescaling does not influence Lemma~\ref{lem:hard_expand}. Thus Proposition~\ref{prop:complexity} still holds for any rescaled version of problem~\eqref{eq:obj_hard_before_scale}.
Specially, we consider the following problem
\begin{equation}\label{eq:hard_scale}
\begin{aligned}
    & \qquad \min_{\vx \in \sR^\harddim } \hardf (\vx) :=  \frac{1}{n} \sum_{i=1}^n \hardf_i(\vx) \quad  \text{where} \quad \hardf_i(\vx):= \lambda \, \hardr( \vx / \beta; \harddim, \hardzeta, \hardc), \\
    & \lambda = \frac{4 \Delta}{\rho {-} 1},\, 
    \beta = \frac{4}{\rho {-} 1 } \sqrt{ \frac{\Delta}{\mu (\rho {+} 1) } }, \,
    \hardzeta = \sqrt{ \frac{2}{1{+}\rho} } \,
    \text{ and }\, \hardc = \frac{4}{ \rho^2 {-} 1}
    \ \text{with} \ \rho = \sqrt{ \frac{2 \delta / \mu }{ \sqrt{2n{+}1} } {+} 1 }.
\end{aligned}
\end{equation}
Here $n, \delta, \mu$ and $\Delta$ are given parameters.
As shown in the next Proposition, $\delta$ is the AveSS parameter, $\mu$ is the strong convexity parameter and $\Delta$ is the function value gap between the initial point and the solution.

\begin{prop}\label{prop:hard_scale_property}
The problem defined in \eqref{eq:hard_scale} with $n \ge 3$ and $m \ge 3$ has the following properties.
\begin{enumerate}
    \item $\hardf$ is $\mu$-strongly convex and $\hardf_i$'s satisfy
    $\delta$-AveSS. \label{prop:hard_scale_property1}
    \item Let $q {=} \frac{\rho - 1}{\rho + 1}$. The minimizer of $\hardf$ is $\vx_* {=} \frac{\beta(\rho+1)}{2} (q, q^{2}, \dots, q^m)^{\top} $ and $\hardf(\vzero) {-} \hardf(\vx_*) {=} \Delta$. \label{prop:hard_scale_property2}
    \item For $0 \le k \le m-1$, we have \label{prop:hard_scale_property3}
    \begin{align}\label{eq:hard_scale_property3}
        \min_{\vx \in \subspace_{k}} \hardf(\vx) -
        \hardf(\vx_*)
        \ge \Delta q^{2k} \text{ and }
        \min_{\vx \in \subspace_k} \norm{ \vx - \vx_* }^2 \ge \frac{ 4 \Delta }{ \mu (\rho + 1) } q^{2k}.
    \end{align}
\end{enumerate}
\end{prop}
Property~\ref{prop:hard_scale_property2} shows that the minimizer of problem~\eqref{eq:hard_scale} has all elements nonzero.
Thus, it does not lie in any subspace $\subspace_k$ for $k < m$.
As a result, we cannot obtain an approximate solution up to an arbitrarily small accuracy, unless we get an iterate with the last element nonzero, as claimed by Property~\ref{prop:hard_scale_property3}.
This implies problem~\eqref{eq:hard_scale} satisfies the condition~\ref{eq:complexity_condition}.

Combining Propositions~\ref{prop:complexity} and \ref{prop:hard_scale_property}, we can establish the lower bound of the communication complexity.
\begin{thm}[Formal version of Theorem~\ref{thm:lower}]\label{thm:lower_formal}
    Suppose we use any algorithm $\gA$ satisfying Definition~\ref{defn:lower_alg} to solve the minimization problem~\eqref{eq:hard_scale} and the following conditions hold
    \begin{align*}
        n \ge 3\; \text{ and } \;
        \eps \le \frac{\Delta}{9} \cdot q^3, \text{ with } q = \frac{\rho - 1}{\rho + 1}, \rho = \sqrt{\frac{2 \delta/\mu}{ \sqrt{2n+1} } + 1}.
    \end{align*}
    Set $\harddim = \left\lfloor \frac{\log (\Delta/(9 \eps)) }{2 \log (1/q)} + 2 \right\rfloor$.
    In order to find $\hat{\vx} $ such that $ \E \hardf (\hat{\vx}) - \min_{\vx \in \sR^\harddim} \hardf (\vx) < \eps$, the communication complexity in expectation is 
    \begin{align*}
        \begin{cases}
        \Omega \left( n + n^{3/4} \sqrt{\delta / \mu} \log (1 / \eps) \right), & \text{ for } \frac{\delta}{\mu} = \Omega(\sqrt{n}), \\
        \Omega \left( n + \frac{n \log (1 / \eps)  }{ 1 + ( \log (\mu \sqrt{n} / \delta) )_+ } \right), & \text{ for } \frac{\delta}{\mu} = \gO (\sqrt{n}).
        \end{cases}
    \end{align*}
\end{thm}
Recall that in \eqref{eq:hard_scale_property3}, $\min_{\vx \in \subspace_{k}} \hardf(\vx) -  \hardf(\vx_*)$ and
$\min_{\vx \in \subspace_k} \norm{ \vx - \vx_* }^2$ are both lower bounded by $q^{2k}$ multiplied with some constants.
If we want to find $\hat{\vx}$ such that $\E \norm{ \hat{\vx} - \vx_* }^2 < \eps$, the communication complexity is the same.
The proof of Theorem~\ref{thm:lower_formal} is deferred to Appendix~\ref{app:lower:pf_thm_lower}.

\subsection{Proof of Proposition~\ref{prop:hard_aveSS_except1}}\label{app:lower:pf_avess}

\begin{proof}
For convenience of notation, we omit the dependence of $\hardr_i$, $\hardr$, $\mB$ and $\vb_l$ on the parameters $\harddim$, $\hardzeta$ and $\hardc$.
With the definition of $\{ \hardr_i \}_{i=1}^n$ and $\hardr$, we have
\begin{equation}\label{eq:grad-r}
    \nabla (\hardr_i - \hardr) (\vx) = \begin{cases}
    - \sum_{l=1}^\harddim \vb_l \vb_l^\top \vx - (n-1) \ve_1, & i=1, \\
    n \sum_{l \in \gL_i} \vb_l \vb_l^\top \vx - \sum_{l=1}^\harddim \vb_l \vb_l^\top \vx + \ve_1, & i \neq 1.
    \end{cases}
\end{equation}
From the definition of $\vb_l$, we have 
\begin{equation}\label{eq:pf_avess_inner_b}
    \vb_l^\top \vb_l = \begin{cases}
        2,  & 1 \le l \le \harddim - 1, \\ \hardzeta^2, & l=m,
    \end{cases}
    \qquad
    \vb_l^\top \vb_{l+1} = 
    \begin{cases}
        -1, & 1 \le l \le \harddim-2, \\
        -\hardzeta, & l=\harddim-1,
    \end{cases}
\end{equation}
and $\vb_l^\top \vb_{l'} = 0$ for any $ |l - l'| \ge 2 $.
Since $n \ge 3$, this implies $ \vb_l^\top \vb_{l'} = 0$ for any $l ,l' \in \gL_i$ and $l \neq l'$.
Define $\vb_0 = \vb_{\harddim+1} = \vzero$ for ease of notation.
Let $\mA_i = \sum_{l \in \gL_i}\vb_l \vb_l^\top, \mA = \sum_{i=2}^n\mA_i = \sum_{l=1}^m \vb_l \vb_l^\top$.
Then for any $\vx, \vy \in \sR^\harddim$ and $\vu = \vx - \vy$, we have
\begin{eqnarray*}
    && \sum_{i=1}^n \norm{\nabla (\hardr_i - \hardr) (\vx) - \nabla ( \hardr_i - \hardr ) (\vy) }^2 \stackrel{\eqref{eq:grad-r}}{=}
    \norm{\mA\vu}^2 + \sum_{i=2}^n \norm{n\mA_i\vu-\mA\vu}^2 \\
    &=& \norm{\mA\vu}^2 + \sum_{i=2}^n n^2\norm{\mA_i\vu}^2-2n \sum_{i=2}^n(\mA_i\vu)^\top\mA\vu +(n-1)\norm{\mA\vu}^2 = n^2\sum_{i=2}^n \norm{\mA_i\vu}^2 -n\norm{\mA\vu}^2. 
\end{eqnarray*}
Note that by Eq.~\eqref{eq:pf_avess_inner_b}, $\norm{\mA_i\vu}^2 = \norm{\sum_{l \in \gL_i} \vb_l \vb_l^\top\vu}^2 = \sum_{l \in \gL_i} 
(\vb_l^\top \vu)^2\vb_l^\top\vb_l$ and 
\begin{align*}
    \norm{\mA\vu}^2 &= \norm{\sum_{l=1}^m \vb_l \vb_l^\top\vu}^2 = \sum_{l=1}^m
    (\vb_l^\top \vu)^2\vb_l^\top\vb_l + (\vb_l^\top \vu)(\vb_{l+1}^\top \vu)\vb_{l+1}^\top\vb_l + (\vb_l^\top \vu)(\vb_{l-1}^\top \vu)\vb_{l-1}^\top\vb_l \\
    &= \sum_{l=1}^m (\vb_l^\top \vu)^2\vb_l^\top\vb_l + 2(\vb_l^\top \vu)(\vb_{l+1}^\top \vu)\vb_{l+1}^\top\vb_l,
\end{align*}
where the final equality uses $\vb_0 = \vb_{\harddim+1} = \vzero$.
Hence, we get
\begin{eqnarray}
    && \frac{1}{n}\sum_{i=1}^n \norm{\nabla (\hardr_i - \hardr) (\vx) - \nabla ( \hardr_i - \hardr) (\vy)}^2 \nonumber \\
    &=& n\sum_{l =1}^m (\vb_l^\top \vu)^2\vb_l^\top\vb_l -
    \left[\sum_{l=1}^m (\vb_l^\top \vu)^2\vb_l^\top\vb_l + 2(\vb_l^\top \vu)(\vb_{l+1}^\top \vu)\vb_{l+1}^\top\vb_l \right] \nonumber \\
    &=& (n-1
    ) \sum_{l =1}^m (\vb_l^\top \vu)^2\vb_l^\top\vb_l - 2
    \sum_{l=1}^m (\vb_l^\top \vu)(\vb_{l+1}^\top \vu)\vb_{l+1}^\top\vb_l. \label{eq:pf_avess_tmp1}
\end{eqnarray}

Recall that $0 < \hardzeta \le \sqrt{2}$. Then \eqref{eq:pf_avess_inner_b} implies $ \vb_l^\top \vb_l \le 2$ and $|\vb_l^\top \vb_{l+1} | \le \sqrt{2}$ for any $l$.
Substituting these into \eqref{eq:pf_avess_tmp1} and
using Cauchy's inequality, we have
\begin{eqnarray*}
    \frac{1}{n} \sum_{i=1}^n \norm{ \nabla (\hardr_i {-} \hardr) (\vx) {-} \nabla ( \hardr_i {-} \hardr) (\vy) }^2 &\le& 2 (n-1) \sum_{l=1}^\harddim  (\vb_l^\top \vu)^2  
    + \sqrt{2} \sum_{l=1}^\harddim  \left[ (\vb_l^\top \vu)^2 + (\vb_{l+1}^\top \vu)^2 \right] \\
    &\le& \left(2n + 2 \sqrt{2} - 2\right) \sum_{l=1}^\harddim (\vb_l^\top \vu)^2 \leq (2n+1)\sum_{l=1}^\harddim (\vb_l^\top \vu)^2.
\end{eqnarray*}
Notice that 
$(\vb_l^\top \vu)^2 = (u_l - u_{l+1})^2 \le 2 (u_l^2 + u_{l+1}^2) $ for $1 \le l \le \harddim - 1$ and $(\vb_\harddim^\top \vu)^2 = (\hardzeta u_\harddim)^2 \le 2 u_\harddim^2$.
This implies
\begin{align*}
    \frac{1}{n} \sum_{i=1}^n \norm{ \nabla (\hardr_i - \hardr) (\vx) - \nabla ( \hardr_i - \hardr ) (\vy) }^2 \le 4(2n+1) \norm{\vu}^2 = (8 n + 4) \norm{\vx - \vy }^2.
\end{align*}
As a result, 
$\hardr_i$'s satisfy
$\sqrt{8n + 4}$-AveSS.
\end{proof}

\subsection{Proof of Lemma~\ref{lem:hard_expand}}
\label{app:lower:pf_expand}

\begin{proof}
For convenience of notation, we omit the dependence of $\hardr_i$, $\hardr$, $\mB$ and $\vb_l$ on the parameters $\harddim$, $\hardzeta$ and $\hardc$.

1) First, we focus on the gradient of the $r_i$.
Recall that
\begin{align*}
     & \hardr_i(\vx) {=}  
    \begin{cases}
        \frac{\hardc}{2} \norm{\vx}^2
        - n \dotprod{\ve_1, \vx},
        & \text{for } i = 1, \\
        \frac{n}{2} \sum\limits_{l \in \gL_i}  \vx^\top \vb_l \vb_{l} ^{\top} \vx 
        + \frac{\hardc}{2} \norm{\vx}^2,
        & \text{for } i \neq 1.
    \end{cases}
\Rightarrow
     \nabla \hardr_i(\vx) {=}  
    \begin{cases}
        \hardc \vx
        - n \ve_1,
        & \text{for } i = 1, \\
        n \sum\limits_{l \in \gL_i}  \vb_l \vb_{l} ^{\top} \vx 
        + \hardc \vx,
        & \text{for } i \neq 1.
    \end{cases}
\end{align*}
\rom{1}: If $\vx \in \subspace_0$, i.e., $\vx = \vzero$, 
we have $\nabla \hardr_1 (\vzero) = -n \ve_1 \in \subspace_1$ and $ \nabla \hardr_i (\vzero) = \vzero$ for $i \neq 1$.
\rom{2}: If $\vx = (x_1,\dots,x_{m}) \in \subspace_k$ for $1 \le k < \harddim$, we have $\nabla \hardr_1 (\vx) = \hardc \vx - n \ve_1 \in \subspace_k$.
As for $i \neq 1$, we need to examine $\vb_l \vb_l^\top \vx$.
One can check
\begin{align*}
    \vb_l \vb_l^{\top} \vx = 
    \begin{cases}
    (x_{l} - x_{l+1}) (\ve_{l} - \ve_{l+1}), & 1 \le l \le \harddim-1, \\
    \hardzeta^2 x_\harddim \ve_\harddim, &l = \harddim.
    \end{cases}
\end{align*}
For $\vx \in \subspace_k$, we have 
\begin{align}\label{eq:proof_expand_bbl}
    \vb_l \vb_l^{\top} \vx \in
    \begin{cases}
    \subspace_{k}, &l \neq k, \\
    \subspace_{k+1}, &l = k.
    \end{cases}
\end{align}
As a result, if $k \in \gL_i$, then $\nabla \hardr_i (\vx) \in \subspace_{k+1}$; otherwise, $\nabla \hardr_i (\vx) \in \subspace_k$.

2) Now we turn to the proximal operator.
\rom{1} For $i = 1$, it is easy to verify $\prox_{\hardr_1}^\gamma (\vx) = (1 / \gamma + \hardc)^{-1} (\vx / \gamma + n \ve_1)$.
Thus, if $\vx \in \subspace_0$, $\prox_{\hardr_1}^\gamma (\vx) \in \subspace_1$; if $\vx \in \subspace_k$ for $k \ge 1$, $\prox_{\hardr_1}^\gamma (\vx) \in \subspace_k$.
\rom{2} For $i \neq 1$, we define $\vu_i := \prox_{\hardr_i}^\gamma (\vx)$ for simplicity. 
Then $\vu_i$ satisfies the following equation
\begin{align*}
    \left[n\gamma \mB_i^{\top} \mB_i + \left( \hardc \gamma + 1\right)\mI  \right] \vu_i = \vx, \ \mB_i := \sum_{l \in \gL_i} \ve_l \vb_l^\top.
\end{align*}
Note that $\mB_i^\top \mB_i = \sum_{l \in \gL_i} \vb_l \vb_l^\top$. 
By the Sherman-Morrison-Woodbury formula, we get
\[ \left(\mI + \tilde{c} \mB_i^{\top} \mB_i \right)^{-1} = \mI - \mB_i^{\top} \left( \frac{1}{\tilde{c}}\mI + \mB_i \mB_i^{\top} \right)^{-1} \mB_i, \forall \tilde{c} \neq 0. \]
In the proof of Proposition~\ref{prop:hard_aveSS_except1}, we have shown that $\vb_{l}^{\top} \vb_{l'} = 0$ for any $|l - l'| \ge 2$ and consequently $\vb_l^\top \vb_{l'} = 0$ for any $l , l' \in \gL_i$ and $l \neq l'$.
Thus, $\mB_i \mB_i^{\top} = \sum_{l \in \gL_i} \vb_l^\top \vb_l \ve_l \ve_l^\top$ is a diagonal matrix. Then we can denote ${\mD}_i = \left( \frac{\hardc \gamma + 1}{n\gamma}\mI + \mB_i \mB_i^\top \right)^{-1} = \sum_{l=1}^\harddim d_{i,l} \ve_l \ve_l^\top$ and obtain
\begin{align*}
    \vu_i & = \left[n\gamma \mB_i^{\top} \mB_i + \left( \hardc \gamma + 1\right)\mI \right]^{-1} \vx    
    = \frac{1}{ \hardc \gamma + 1} \left( \mI + \frac{n\gamma}{\hardc \gamma + 1} \mB_i^{\top} \mB_i \right)^{-1} \vx = \frac{\vx-\mB_i^\top {\mD}_i \mB_i \vx}{\hardc \gamma + 1}.
\end{align*}
Then we have $\mB_i^\top \mD_i \mB_i \vx = \sum_{l \in \gL_i} d_{i,l} \vb_l \vb_l^\top \vx$.
For $\vx \in \subspace_k$, by \eqref{eq:proof_expand_bbl}, if $k \in \gL_i$, then $\vu_i \in \subspace_{k+1}$; otherwise, $\vu_i \in \subspace_k$. This completes the proof.
\end{proof}

\subsection{Proof of Lemma~\ref{lem:hard_expand_infor}}\label{app:lower:pf_expand_infor}

\begin{proof}
Since we can assume $\vx_0 = \vzero$,
Lemma~\ref{lem:hard_expand} implies $\nabla r_1 \in \subspace_1$.
Then from step~\ref{item:lower_alg_init}, we have $\gI_0 = \subspace_1$\footnote{In the definition of the information set, each $f_i$ is replaced by $r_i$ here.}.

Then we focus on the second claim and
examine how many dimensions of the information set we can increase after a round of communication.

Since we choose node~$1$ as the master node, we have access to $\nabla r_1$ and $\prox_{r_1}^\gamma$ in each communication round.
Recall that we set $\gL_1$ as the empty set.
Lemma~\ref{lem:hard_expand} guarantees that the information provided by $r_1$ can never expand the information set unless the information set only contains $\vzero$.
Thus, \eqref{eq:lower_alg_subprob} does not affect the information set, i.e., $\gI_{t+1}^1 = \gI_{t+1}^0$ for any $t \ge 0$.

When $a_t = 0$, only \eqref{eq:lower_alg_infor0} can expand the information set. By Lemma~\ref{lem:hard_expand}, if $i_t$ satisfies $k \in \gL_{i_t}$, we have $ \gI_{t+1}^1 = \gI_{t+1}^0 \subseteq \subspace_{k+1}$. Otherwise, we still have $\gI_{t+1}^1 = \gI_{t+1}^0  \subseteq \subspace_k$.
Then step~\ref{item:lower_alg_bernoulli} in Definition~\ref{defn:lower_alg} implies $\gI_{t+1} = \gI_{t+1}^1$.

When $a_t = 1$, from the above analysis, we always have $\gI_{t+1} \subseteq \subspace_{k+1}$.
By Lemma~\ref{lem:hard_expand}, \eqref{eq:lower_alg_before_fullgrad} could expand the information set by at most one dimension. It follows that $\tilde{\vy}_{t+1} \in \subspace_{k+2}$.
Using Lemma~\ref{lem:hard_expand} again yields $\gI_{t+1} \subseteq \subspace_{k+3}$.
\end{proof}

\subsection{Proofs of Corollary~\ref{coro:lower_stopping_time}}\label{app:lower:pf_stopping_time}

\begin{proof}
We prove the first claim by induction on $t$.
That is to say, we prove that for any integer $t \ge -1$, $\gI_{t+1} \subseteq \subspace_{3k - 2}$ for any $k$ satisfying $t < T_k$.
Define $k(t)$ as the positive integer such that $T_{k(t)-1} \le t < T_{k(t)}$.
From the monotonicity of $\subspace_{\cdot}$, it suffices to prove $\gI_{t+1} \subseteq \subspace_{3k(t)-2} $.

By Lemma~\ref{lem:hard_expand}, we have $\gI_0 = \subspace_1$ and $k(-1)=1$. The claim holds for $t=-1$.
Suppose that $\gI_{t+1} \subseteq \subspace_{3k(t)-2} $.
If $3k(t) - 2 \in \gL_{i_{t+1}}$ or $a_{t+1} = 1$, Lemma~\ref{lem:hard_expand_infor}  together with \eqref{eq:lower_stopping_time} implies
$k(t+1) = k(t) + 1$ and
$\gI_{t+2} \subseteq \subspace_{3 k(t) + 1} = \subspace_{ 3 k(t+1) - 1 }$.
Otherwise, we still have $k(t+1) = k(t)$ and $\gI_{t+2} \subseteq \subspace_{3 k(t) - 1} = \subspace_{ 3 k(t+1) - 1 }$.

For the second claim, the independence of $\{Y_k\}_{k \ge 1}$ is natural consequence of the independence of $\{(i_t, a_t)\}_{t\ge 1}$. 

For the last one, note that $3k - 2 \in \gL_{i_t}$ is equivalent to $i_t \equiv 3k {-} 1 (\bmod~(n{-}1) ) $ for $2 \le i_t \le n$. Then we have for $k' \equiv 3k{-}1 (\bmod ~(n{-}1)), 2 \le k' \le n$,
\begin{eqnarray*}
    \sP ({T_{k} - T_{k-1} = s} ) &=& \sP( i_{T_{k-1} + 1} \neq k', \dots, i_{T_{k-1} + s - 1} \neq k', a_{T_{k-1} + 1} = \cdots = a_{T_{k-1} + s - 1} = 0, \\
    && \quad i_{T_{k-1} + s} = k' \mbox{ or } a_{T_{k-1} + s} = 1) \\
    &\stackrel{\rom{1}}{=}& [(1 - q_{k'})(1-p)]^{s-1} \left[1-(1-q_{k'})(1-p)\right],
\end{eqnarray*}
where $\rom{1}$ is due to the independence of $\{ (i_t, a_t) \}_{t \ge 1}$. So $Y_k = T_k - T_{k-1}$ is a geometric random variable with success probability $1-(1-q_{k'})(1-p) = q_{k'} + p - q_{k'}p$. 
\end{proof}

\subsection{Proof of Proposition~\ref{prop:complexity}}\label{app:lower:pf_complexity}

\begin{proof}
By Corollary~\ref{coro:lower_stopping_time},
if $N-1 < T_{M+1}$, then $\gI_N \subseteq \subspace_{3M+1} \subseteq \subspace_{\harddim - 1}$.
Thus we have
\begin{align*}
    &\E \min_{\vx \in \gI_N} \hardr(\vx) - \min_{\vx \in \sR^\harddim} \hardr(\vx) \ge \E \left[ \min_{\vx \in \gI_N} \hardr(\vx) - \min_{\vx \in \sR^\harddim} \hardr(\vx) \bigg\vert N-1 < T_{M+1}\right] \sP( N-1 < T_{M+1}) \\
    \geq& \E \left[ \min_{\vx \in \subspace_{3M+1}} \hardr(\vx) - \min_{\vx \in \sR^\harddim} \hardr(\vx) \bigg\vert N-1 < T_{M+1}\right] \sP( N-1 < T_{M+1}) \geq 9 \eps\, \sP(N-1 < T_{M+1}).
\end{align*}
By Corollary~\ref{coro:lower_stopping_time} again, $T_{M+1}$ can be written as $T_{M+1} = \sum_{l=1}^{M+1} Y_l$, where $\{Y_l\}_{1 \le l \le M+1}$ are independent random variables, and $Y_l \sim \geo (\tilde{q}_l )$ with $\tilde{q}_l = q_{l'} + p - p q_{l'}$, $l' \equiv 3l{-}1 (\bmod ~(n{-}1))$ and $2 \le l' \le n$.
Moreover, Definition~\ref{defn:lower_alg} guarantees
$\max_{2 \le l' \le n} q_l' \le q_0 / n $ and $p \le c_0 / n$.
Then we have $\sum_{l=1}^{M+1} \tilde{q}_l \le (q_0 + c_0) (M+1) / n$.
Therefore, by Lemma~\ref{lem:geo}, we have
$$\sP ( T_{M+1} > N-1 ) = \sP \left(\sum_{l=1}^{M+1} Y_l > \frac{n(M+1)}{4 (q_0 + c_0) } \right)
\ge \sP \left( \sum_{l=1}^{M+1} Y_l > \frac{ (M+1)^2 }{ 4 \sum_{l=1}^{M+1} \tilde{q}_l } \right)
\ge \frac{1}{9},$$
which implies our desired result.
\end{proof}

\subsection{Proof of Proposition~\ref{prop:hard_scale_property}}\label{app:lower:pf_hard_property}

\begin{proof}
\textbf{Property~\ref{prop:hard_scale_property1}.}
By Proposition~\ref{prop:hard_aveSS_except1}, we have $\hardf$ is $ \lambda \hardc / \beta^2$-strongly convex and $\hardf_i$'s satisfy
$\lambda \sqrt{8n + 4} / \beta^2$-AveSS. One can check $\lambda \hardc / \beta^2 = \mu$ and $\lambda \sqrt{8n + 4} / \beta^2 = \delta $.

\textbf{Property~\ref{prop:hard_scale_property2}.}
Let $\xi := \lambda / \beta^2 = \mu (\rho^2 - 1) / 4$. We have
\begin{align*}
    \hardf(\vx) 
    = \frac{\xi}{2} \vx^\top \mA(\harddim, \hardzeta)\, \vx + \frac{ \mu}{2} \norm{\vx}^2 - \xi \beta \dotprod{\ve_1,\vx}.
\end{align*}
Letting $ \nabla \hardf (\vx) = \vzero $ yields $\left( \xi \mA \left(\harddim, \hardzeta \right) + \mu \mI \right)\vx = \xi \beta \ve_1$, or equivalently.
\begin{align}\label{proof:strongly:minimizer0}
    \begin{bmatrix}
        1 + \frac{\mu}{\xi} & -1 & & & \\
        -1 & 2 + \frac{\mu}{\xi} & -1 & & \\
        & \ddots & \ddots & \ddots & \\
        & & -1 & 2 + \frac{\mu}{\xi} & -1 \\
        & & & -1 & \hardzeta^2 + 1 + \frac{\mu}{\xi}
    \end{bmatrix}
    \vx = 
    \begin{bmatrix}
        \beta \\
        0 \\
        \vdots \\
        0 \\
        0
    \end{bmatrix}.
\end{align}
Since $q = \frac{\rho - 1}{\rho + 1}$, we get $2 + \frac{\mu}{\xi} = \frac{2\rho^2+2}{\rho^2 - 1}=q+\frac{1}{q}$ and $\hardzeta^2 + 1 + \frac{\mu}{\xi} = \frac{\rho + 1}{\rho - 1} = \frac{1}{q}$. 
We solve \eqref{proof:strongly:minimizer0} by
\[  x_{m-1} - \frac{x_m}{q} = 0, \ x_{k} - \left( q+\frac{1}{q} \right) x_{k+1} + x_{k+2} = 0, k \in [m-2], \ \left( q+\frac{1}{q} -1\right) x_1-x_2=\beta. \]
Thus, $\vx_{*} = \frac{\beta}{1-q} (q, q^{2}, \dots, q^m)^{\top} $ and 
$ \hardf(\vx_*) = -\frac{\xi \beta \dotprod{\ve_1, \vx_*}}{2} 
= - \frac{ \xi \beta^2 q }{2(1-q)}
= - \frac{\lambda (\rho - 1)}{4} \stackrel{\eqref{eq:hard_scale}}{=} - \Delta$.

\textbf{Property~\ref{prop:hard_scale_property3}.}
If $\vx \in \subspace_k$, $1 \le k < m$, then $x_{k+1} = x_{k+2} = \cdots = x_{m} = 0$. 
Let $\vy$ denote the first $k$ coordinates of $\vx$ and $\mA_k$ denote the first $k$ rows and columns of $\mA(m,\hardzeta)$.
Then for any $\vx \in \subspace_k $, we can rewrite $\hardf (\vx)$ as
\begin{align*}
    \hat{\hardf} (\vy) := \hardf (\vx) = \frac{\xi}{2} \vy^{\top} \mA_k \vy + \frac{\mu}{2} \norm{\vx}^2 - 
    \xi \beta \dotprod{\hat{\ve}_1, \vy},
\end{align*}
where $\hat{\ve}_1$ is the first $k$ coordinates of $\ve_1$.
Let $\nabla f_k (\vy) = \vzero$,
that is
\begin{align*}
    \begin{bmatrix}
        1 + \frac{\mu}{\xi} & -1 & & & \\
        -1 & 2 + \frac{\mu}{\xi} & -1 & & \\
        & \ddots & \ddots & \ddots & \\
        & & -1 & 2 + \frac{\mu}{\xi} & -1 \\
        & & & -1 & 2 + \frac{\mu}{\xi}
    \end{bmatrix}
    \vy = \begin{bmatrix}
        \beta \\ 0 \\ \vdots \\ 0 \\ 0
    \end{bmatrix}.
\end{align*}
Similarly, we need to solve 
\[ x_{m-1} = \left(q + \frac{1}{q}\right)x_m, \ x_{k} - \left( q+\frac{1}{q} \right) x_{k+1} + x_{k+2} = 0, k \in [m-2], \ \left( q+\frac{1}{q} -1\right) x_1-x_2=\beta. \]
By some computation, one can check the solution is 
\begin{align*}
    {\vy}_* = \frac{\beta q^{k+1}}{2(q-1)(1+q^{2k+1}) } \left( q^{-k} - q^k, q^{-(k-1)} - q^{k-1}, \dots, q^{-1} - q^{1} \right)^{\top}.
\end{align*}
Thus, we have
\[ \hardf (\vy_*) = -\frac{\xi \beta \dotprod{\hat{\ve}_1, \vy_*}}{2} =
\frac{\xi \beta^2 q}{2(1-q)} \cdot \frac{1 - q^{2k}}{1 + q^{2k+1}} 
=\frac{ \lambda ( \rho - 1) }{4} \cdot \frac{1 - q^{2k}}{1 + q^{2k+1}} 
= \frac{(1 - q^{2k})\Delta }{1 + q^{2k+1}}, \]
and by $q<1$, we further have that
\begin{align*}
    \min_{\vx \in  \subspace_k} \hardf (\vx) {-} \min_{\vx \in \sR^\harddim} \hardf (\vx) 
    = \hardf (\vy_*) {-} \hardf (\vx_*) = \Delta \left(1 - \frac{1-q^{2k}}{1+q^{2k+1}} \right) 
    = \frac{(1+q)q^{2k}\Delta}{1+q^{2k+1}} \ge \Delta q^{2k}.
\end{align*}
Moreover, recall that $\vx_* = \frac{\beta (\rho + 1) }{2} (q, q^{2}, \dots, q^m)^{\top}$. Then we have
\begin{align*}
    \min_{\vx \in \subspace_k} \norm{\vx - \vx^\star}^2
    = \frac{ \beta^2 (\rho + 1)^2 }{4} \sum_{i = k+1}^m q^{2i}
    \ge \frac{\beta^2 (\rho + 1)^2 q^2 }{4} q^{2k}
    = \frac{4 \Delta}{ \mu (\rho + 1) } q^{2k}.
\end{align*}
This completes the proof.
\end{proof}

\subsection{Proof of Theorem \ref{thm:lower_formal}}\label{app:lower:pf_thm_lower}

\begin{proof}
Let $q = \frac{\rho - 1}{ \rho + 1}$ and $M = \left\lfloor  \frac{\log ( \Delta / 9\eps )}{6\log 1/ q} - \frac{1}{3} \right\rfloor$.
From the condition on $\eps$ and the definition of $\harddim$, one can check $0 \le M \le (m-2) / 3$ and $m \ge 3$.
Moreover, we have $3M + 1 \le \frac{\log (9\eps/\Delta)}{2 \log q}$.
Then by Proposition~\ref{prop:complexity}, after $N = \frac{n (M+1)}{4 (q_0 + c_0) } + 1$ rounds of communication, the information set satisfies $\gI_N \subseteq \subspace_{3M+1} \subseteq \subspace_{\harddim - 1}$.
The third property of Propostion~\ref{prop:hard_scale_property} implies
\begin{align*}
    \min_{\vx \in \subspace_{3M+1}} \hardf (\vx) - \min_{\vx \in \sR^\harddim} f(\vx) \ge \Delta q^{6M+2} \ge 9 \eps.
\end{align*}
Then by Proposition~\ref{prop:complexity} again, in order to find $\hat\vx$ such that $\E \hardf(\hat\vx) - \min_{\vx \in \sR^\harddim} \hardf(\vx) < \eps$, the algorithm $\gA$ needs at least $N$ communication rounds. 

Now we give a lower bound $N$.
According to whether $ 2\delta / \mu $ is larger than $\sqrt{2n+1}$, we divide the analysis into two cases.

\textbf{Case~1: $ 2 \delta / \mu \ge \sqrt{2n+1}$.}
Then $\rho \geq \sqrt{2}$. By inequality $0 < \log(1+x) \leq x, \forall x>0$, we get
\begin{align*}
    \frac{1}{\log \frac{1}{q}} =  \frac{1}{ \log \left(1+\frac{2}{\rho-1}\right) } \ge \frac{1}{\frac{2}{\rho-1}} = \frac{\rho-1}{2} = \frac{1}{2}\left(\sqrt{\frac{2 \delta/\mu}{\sqrt{2n+1}}+1}-1\right) \geq \frac{\sqrt{2\delta/\mu}}{6\sqrt[4]{2n+1}},
\end{align*}
where the final inequality uses $\sqrt{t+1} -1 \geq \sqrt{t}/3, \forall t \geq 1$.
Moreover, the condition on $\eps$ implies 
$\log \frac{\Delta }{9 \eps} \ge 3 \log \frac{1}{q}$.
Then we have 
\begin{align*}
    M+1 \ge \frac{ \log \frac{\Delta}{9 \eps} }{6 \log \frac{1}{q} } - \frac{1}{3}
    \ge \frac{ \log \frac{\Delta}{9 \eps} }{18 \log \frac{1}{q} } 
    \ge \frac{ \sqrt{2\delta / \mu} }{ 108\sqrt[4]{2n+1} } \log \frac{\Delta}{9 \eps} = \Omega \left( \frac{\sqrt{\delta / \mu} }{ n^{1/4} } \log \frac{1}{\eps} \right).
\end{align*}
It follows that $N = \frac{n (M+1)}{4 (q_0 + c_0)} + 1 = \Omega \left( n^{3/4} \sqrt{\delta / \mu} \log(1 / \eps) \right)$.
The total communication cost in expectation is of the order $\Theta(n + N) = \Omega \left(n + n^{3/4} \sqrt{\delta / \mu} \log(1 / \eps) \right)$.

\textbf{Case~2: $2 \delta / \mu < \sqrt{2n+1}$.}
In this case, we have $1 \leq \rho < \sqrt{2}$ and consequently
\[ \log \frac{1}{q} = \log \left(1+\frac{2}{\rho-1}\right) \stackrel{\rom{1}}{\leq} \log\left(1+\frac{3\sqrt{2n+1}}{\delta/\mu}\right) \leq \log \left(\frac{7\mu\sqrt{2n+1}}{2\delta}\right) \leq 2 + \log \left(\frac{\mu\sqrt{2n+1}}{2\delta}\right), \]
where $\rom{1}$ uses $\sqrt{t+1} -1 \geq t/3, \forall 0 < t \leq 1$ and $\rho = \sqrt{\frac{2\delta/\mu}{\sqrt{2n+1}}+1}$,
Moreover, the condition on $\eps$ implies 
$\log \frac{\Delta }{9 \eps} \ge 3 \log \frac{1}{q}$.
Then we have
\begin{align*}
    M+1 \ge \frac{ \log \frac{\Delta}{9 \eps} }{6 \log \frac{1}{q} } - \frac{1}{3}
    \ge \frac{ \log \frac{\Delta}{9 \eps} }{18 \log \frac{1}{q} }
    = \Omega \left( \frac{\log (1 / \eps) }{ 1 + ( \log(\mu \sqrt{n} / \delta) )_+ } \right).
\end{align*}\
where $(a)_+$ denote $\max \{ a, 0 \}$.
It follows that $N = \frac{n (M+1)}{4 (q_0 + c_0)} + 1 = \Omega \left(  \frac{n \log(1 / \eps) }{1 + ( \log(\mu \sqrt{n} / \delta) )_+} \right)$.
The total communication cost in expectation is of the order $\Theta(n + N) = \Omega \left( n +  \frac{n \log(1 / \eps) }{1 + ( \log(\mu \sqrt{n} / \delta) )_+} \right)$.
\end{proof}